\documentclass[twoside]{article}

\usepackage[accepted]{aistats2026}

\usepackage[utf8]{inputenc} 
\usepackage[T1]{fontenc}    
\usepackage{url}            
\usepackage{booktabs}       
\usepackage{nicefrac}       
\usepackage{microtype}      
\usepackage{xcolor}         

\usepackage{times}
\usepackage{soul}
\usepackage{parskip}
\usepackage{url}
\usepackage[small]{caption}
\usepackage{graphicx}
\usepackage{amsmath}
\usepackage{amsthm}

\usepackage[toc,page,header]{appendix}
\usepackage{minitoc}

\usepackage{natbib}
\usepackage{color}
\usepackage{amsfonts,amssymb,bbm, bm}
\usepackage{empheq}
\usepackage{xcolor}
\usepackage{tcolorbox}
\definecolor{mydarkblue}{rgb}{0,0.08,0.45}
\usepackage[colorlinks=true, linkcolor=mydarkblue, citecolor=mydarkblue]{hyperref}
\usepackage{cleveref}
\usepackage{tikz}
\usetikzlibrary{shadows,arrows.meta,positioning,backgrounds,fit,chains,scopes}
\usepackage{caption}
\usepackage{subcaption}

\usepackage{etoolbox}

\usepackage{savesym,multirow,url,xspace,graphicx,hyperref,enumitem,color}
\usepackage{dsfont}
\usepackage{multicol}
\usepackage{algorithmic}
\usepackage{algorithm}
\usepackage{dsfont}
\allowdisplaybreaks
\usepackage{tikz}
\usetikzlibrary{automata, positioning, arrows}
\tikzset{
	->, 
	every state/.style={thick, fill=gray!10}, 
	initial text=$ $, 
}
\usetikzlibrary{arrows.meta}
\usepackage{pgfplots}
\pgfplotsset{width=10cm,compat=1.9}
\usepackage{diagbox}
\usepackage{caption}

\makeatletter
\newif\if@restonecol
\makeatother
\usepackage{amsmath,amssymb,xfrac,amsthm}
\usepackage{mathtools}
\usepackage{wrapfig}
\setitemize{noitemsep,topsep=1pt,parsep=1pt,partopsep=1pt, leftmargin=12pt}
\newtheorem{lemma}{Lemma}[section]
\newtheorem{theorem}{Theorem}[section]
\newtheorem*{theorem*}{Theorem}

\newtheorem{corollary}{Corollary}
\newtheorem{proposition}{Proposition}[section]

\newtheorem{assumption}{Assumption}

\newtheorem{remark}{Remark}
\makeatletter

\newcommand{\Rmnum}[1]{\expandafter\@slowromancap\romannumeral #1@}
\makeatother
\usepackage{caption}

\usepackage{enumitem}
\usepackage{wasysym}

\usepackage{xcolor}
\newcommand{\LINECOMMENT}[1]{\STATE {\color{mydarkblue}\ttfamily\small \(\triangleright\) #1}}
\usepackage{bm}
\usepackage{bbm}

\newcommand{\norm}[1]{\left\lVert#1\right\rVert}

\newcommand{\expct}[1]{\mathbb{E}\left[#1\right]}
\newcommand{\expctu}[2]{\mathbb{E}_{#1}\left[#2\right]}

\newcommand{\bs}[1]{\boldsymbol{#1}}

\renewcommand{\tilde}{\widetilde}

\begin{document}

%
\runningtitle{Corruption-robust MARLHF}

%
\runningauthor{Nika, Mandal, Kamalaruban, Singla, and Radanović}

\twocolumn[

\aistatstitle{Corruption-robust Offline Multi-agent \\ Reinforcement Learning From Human Feedback}

\aistatsauthor{ Andi Nika\\MPI-SWS \And  Debmalya Mandal\\University of Warwick \And Parameswaran Kamalaruban\\Visa \AND Adish Singla\\MPI-SWS \And Goran Radanović\\MPI-SWS}\vspace{0.5cm}]


\begin{abstract}
\looseness-1We consider robustness against data corruption in offline multi-agent reinforcement learning from human feedback (MARLHF) under a strong‐contamination model: given a dataset $D$ of trajectory–preference tuples (each preference being an $n$-dimensional binary label vector representing each of the $n$ agents’ preferences), an $\epsilon$-fraction of the samples may be arbitrarily corrupted. We model the problem using the framework of linear Markov games. First, under a \emph{uniform coverage} assumption—where every policy of interest is sufficiently represented in the clean (prior to corruption) data—we introduce a robust estimator that guarantees an $O(\epsilon^{1-o(1)})$ bound on the Nash‐equilibrium gap. Next, we move to the more challenging \emph{unilateral coverage} setting, in which only a Nash equilibrium and its single‐player deviations are covered: here our proposed algorithm achieves an $O(\sqrt{\epsilon})$ Nash‐gap bound. Both of these procedures, however, suffer from intractable computation. To address this, we relax our solution concept to \emph{coarse correlated equilibria} (CCE). Under the same unilateral‐coverage regime, we then derive a quasi-polynomial‐time algorithm whose CCE gap scales as $O(\sqrt{\epsilon})$.  To the best of our knowledge, this is the first systematic treatment of adversarial data corruption in offline MARLHF. 
\end{abstract}

\doparttoc
\faketableofcontents

\section{INTRODUCTION}

Reinforcement Learning from Human Feedback (RLHF) \citep{christiano2017deep} has surged in popularity as a straightforward, efficient means of fine-tuning large language models (LLMs) \citep{ouyang2022training}. While most existing work focuses on single-agent settings, many real-world applications involve multiple interacting decision-makers, such as autonomous vehicles, distributed systems, or strategic marketplaces. Extending RLHF to such settings leads to multi-agent RLHF (MARLHF), where the goal is to learn aligned joint policies from preference data. Yet, despite its importance, research on MARLHF remains remarkably limited \citep{zhang2024multi}.

\looseness-1At the same time, a critical shortcoming of all learning algorithms is their susceptibility to data-poisoning attacks—and RLHF is no exception \citep{wang2023exploitability, shi2023badgpt, rando2023universal, nika2025policy, baumgartner2024best}. By injecting malicious feedback or subtly corrupting preference labels, adversaries can steer LLMs toward biased or harmful outputs—an especially serious risk as these models are increasingly deployed in safety-critical settings.  \cite{mandal2024corruption} have proposed robust methods against data corruption for single-agent RLHF. However, no prior work has tackled the robustness of MARLHF systems against data‐poisoning attacks. The added strategic complexity of MARLHF can amplify the impact of such attacks. It is therefore unclear whether single-agent methods extend directly to this setting.

Motivated by the above, we initiate the study of corruption-robust offline MARLHF. Our setting assumes access to preference data $D=\{(\tau,\tau',o)\}$ of size $m$, where $\tau,\tau'$ are two sample trajectories and $o$ is an $n$-dimensional vector of binary entries, each denoting a preference over the trajectory pair, corresponding to one of $n$ agents. We assume that an $\epsilon$-fraction of $D$ is arbitrarily corrupted by an attacker. Using the standard Bradley-Terry (BT) \citep{bradley1952rank} preference model, we cast our problem as an instance of a linear Markov game. Previous work has already established that the optimal theoretical guarantees in corruption‐robust offline RL and two‐player zero‐sum Markov games, under linear function approximation, exhibit a linear dependence on $\epsilon$. However, such dependence is achieved only when the data covers all possible directions (i.e.  \textit{uniform coverage}). When the data covers only a Nash policy and its unilateral deviations (i.e. \textit{unilateral coverage}), prior work in RLHF (also RL and two-player zero-sum MGs) achieves $\sqrt{\epsilon}$ bounds. An immediate question is whether we can attain the same $\epsilon$-dependent robustness rates in MARLHF as those of RLHF. Optimal dependence can be expected under uniform coverage. However, weaker coverage assumptions (e.g. unilateral coverage) imply a more challenging setting. 

\paragraph{Challenges.} The main challenge stems from uncertainty over multiple reward models: deriving worst-case guarantees means selecting a policy that performs well under \textit{every} reward model in a confidence set obtained from a robust reward estimation, but computing that policy’s worst‐case performance is challenging---it requires optimizing over all candidate reward models, even though our data‐coverage assumption only holds with respect to the true (unknown) reward. In the single-agent case, \cite{mandal2024corruption} resort to subgradient methods (yielding an $O(\epsilon^{1/4})$ rate) and then a primal-dual approach to recover $O(\sqrt{\epsilon})$. No analogous primal-dual theory exists for Markov games. Thus, we take a different approach.

\paragraph{Our approach.} In multi-agent environments, the ultimate objective is to identify equilibrium policies, and minimizing the \textit{Nash gap} provides a natural surrogate for that goal. Our key insight is that, for any reward model in the confidence set obtained by a robust reward estimation, the gap at a true Nash equilibrium policy $\bs{\pi}^*$ must lie within a small margin of the minimal gap over all policies, up to the reward-estimation error induced by the confidence set. Thus, if $\widetilde{\bs{\pi}}$ nearly minimizes the gap for a candidate reward model, then the gradient of $\bs{\pi}^*$ with respect to rewards  can serve as a biased—but usable—proxy for the gradient at $\widetilde{\bs{\pi}}$. To estimate that proxy, we make use of our unilateral coverage (i.e. $D$ sufficiently covers $\bs{\pi}^*$ and its unilateral deviations). In the linear Markov-game setting, the empirical feature differences between data-generating policies $\bs{\mu}$ and $\bs{\mu}_\textnormal{ref}$ then furnish a tractable approximation to the desired gradient. Plugging this into a first-order optimizer allows us to obtain the desired $O(\sqrt\epsilon)$ guarantee on the Nash gap. In particular, we make the following contributions. A summary of our results is given in Table \ref{tab:results}.
\begin{itemize}
    \item First, assuming that $D$ has uniform coverage, we design an algorithm that approximately computes a NE solely from corrupted preference data. Our algorithm first robustly estimates each reward function. Then, it runs a value-based backward-induction procedure to compute a policy that minimizes an estimate of the gap. We prove that our algorithm incurs $O(n\epsilon^{1-o(1)}+n/\sqrt{m})$ bounds on the Nash gap.
    \item Next, we relax our coverage assumption on our data and assume only unilateral coverage. We run projected gradient ascent (PGA) with a biased estimate of the gradient of the true gap for $T_1$ steps to compute the worst-case reward parameter. Using that parameter, we run the same procedure as in our previous method. We finally prove that our algorithm incurs $O(n\sqrt{\epsilon}+n/\sqrt{m} + n/\sqrt{T_1})$ bounds on the Nash gap.
    \item Our final contribution is on computational tractability. It is well-known that the NE computation is intractable for general-sum Markov games. We thus relax the NE notion into that of coarse correlated equilibrium (CCE). This allows us to frame each stage game as a saddle-point problem with convex-concave objective. We then utilize Optimistic Hedge to learn the CCE of each stage game. This yields and $O(n\sqrt{\epsilon}+n/\sqrt{m} + n/\sqrt{T_1} + n/T_2)$ bound on the CCE gap, where $T_2$ is the number of steps for which we run Optimistic Hedge. 
\end{itemize}

\section{PRELIMINARIES}
This section contains the background technical material to be used throughout the paper.
\subsection{Markov Games}
A Markov game of finite horizon $H$ between $n$ agents is defined by the tuple $G = \bigl(S,\{A_i\}_{i=1}^n, \{P_h\}^{H-1}_{h=0}, \{\mathcal{R}_{i,h}\}^{H-1}_{h=0},s_0\bigr),$
where $S$ is the state space, $A_i$ is the action set of agent $i$, $P_h: S \times A_1 \times \cdots \times A_n \to \Delta(S)$ is the state transition kernel at time-step $h$; the map $\mathcal{R}_{i,h}:S\times A_1\times\cdots\times A_n\to\Delta(\mathbb{R})$ denotes the random reward of agent $i$ at time-step $h$, with $R_{i,h}(s,\bs{a}):= \mathbb{E}[\mathcal{R}_{i,h}(s,\bs{a})|s,a]$; finally, $s_0\in S$ denotes the initial state.

\begin{table*}[t]
  \renewcommand{\arraystretch}{1.8}
  \centering
  \footnotesize
  \scalebox{1.07}{%
    \begin{tabular}{cc}
      \toprule
      & \textbf{Bounds on the NE (CCE) gap}\\
      \midrule
      NE \& Unif. Cov.
        & $\widetilde{O}\left(\left(\frac{1}{\xi_R}+\frac{1}{\xi_P}\right)Hn\epsilon^{1-o(1)} + \frac{H^2n\sqrt{\textnormal{poly}(d)}}{\xi_P\sqrt m}\right)$\\
      NE \&  Unil. Cov.
        & $\widetilde{O}\left(\left(\frac{1}{\sqrt{C_R}}+\frac{1}{\sqrt{C_P}}+\frac{1}{\sqrt{T_1}}\right)\left(H^{5/2}nd^{3/4}\sqrt\epsilon + \frac{H^2n\sqrt{\textnormal{poly}(d)}}{\sqrt m}\right)\right)$\\
      CCE \& Unil. Cov.
        & $\widetilde{O}\left(\left(\frac{1}{\sqrt{C_R}}+\frac{1}{\sqrt{C_P}}+\frac{1}{\sqrt{T_1}}\right)\left(H^{5/2}nd^{3/4}\sqrt\epsilon + \frac{H^2n\sqrt{\textnormal{poly}(d)}}{\sqrt m}\right)+ \frac{Hn^2}{T_2}\right)$\\
      \bottomrule
    \end{tabular}%
  }\vspace{0.2cm}
  \caption{Summary of our bounds for: (i) NE gap minimization under uniform coverage, (ii) NE gap minimization under unilateral coverage, and (iii) CCE gap minimization under unilateral coverage. Here, $\widetilde{O}$ hides any poly-logarithmic factors, $n$ denotes the number of agents, $H$ denotes the horizon and $d$ denotes the dimension. Moreover, $\xi_R$ and $\xi_P$ denote the uniform coverage coefficients, while $C_R$ and $C_P$ denote the unilateral coverage coefficients; $\epsilon$ denotes the corruption parameter, $m$ is the data size, $T_1$ is the number of gradient steps while $T_2$ is the number of steps for which \texttt{OptimisticHedge} is run. It is worth mentioning that our bounds have optimal dependence on $\epsilon$ in the uniform coverage setting, while maintaining the same dependence as the single-agent (and two-player zero-sum Markov game) settings under non-uniform coverage. Moreover, note that the algorithm for CCE approximation is also computationally efficient.}
  \label{tab:results}
\end{table*}

\paragraph{Policies and value functions.} Given agent $i$, a Markov policy $\pi_i=(\pi_{i,0},\ldots,\pi_{i,H-1})$ denotes the tuple containing the decision-making strategies of agent $i$, where, for each $h\in [H-1]:=\{0,1,\ldots,H-1\}$, $\pi_{i,h} :S \rightarrow \Delta(A_i)$ maps states to probability simplices over actions. A joint product policy is defined as the tuple $\bs{\pi}=(\pi_1,\ldots,\pi_n)$ over all agents. We denote by $\Pi^\textnormal{PP}=\Pi^\textnormal{PP}_1\times\ldots\times\Pi^\textnormal{PP}_n$ the overall product policy class and write $\bs{\pi}=(\pi_i,\bs{\pi}_{-i})$, where $\bs{\pi}_{-i}=(\pi_1,\ldots,\pi_{i-1},\pi_{i+1},\ldots,\pi_n)$. Given joint policy $\bs{\pi}$ and state $s$ at time-step $h$, the value function with respect to $\bs{\pi}$, $s$, and $h$ is defined as $$V^{\bs{\pi}}_{i,h}(s) = \expct{\sum_{t=h}^{H-1} R_{i,t}(s_t,\bs{a}_t) | s_h=s,\bs{\pi}_t,P_t}~,$$
where $\bs{a}=(a_1,\ldots,a_n)$. Moreover, for given $\bs{a}$ the action value function is defined as $$Q^{\bs{\pi}}_{i,h}(s,\bs{a}) = \expct{\sum_{t=h}^{H-1} R_{i,t}(s_t,\bs{a}_t) | s_h=s,\bs{a}_h=\bs{a},\bs{\pi}_t,P_t}$$
\paragraph{Nash equilibria.} A product policy $\bs{\pi}^*$ is said to be an $\alpha$-\emph{Nash equilibrium} if there exists $\alpha\geq 0$, such that, for every agent $i$ and state $s$, we have $V^{\bs{\pi}^*}_{i,0}(s)\geq V^{\pi'_i,\bs{\pi}^*_{-i}}_{i,0}(s)-\alpha$, for every ${\pi}'_i\in\Pi^\textnormal{PP}_i$.  If $\alpha=0$, then $\bs{\pi}^*$ is said to be a Nash equilibrium. We also define the notion of optimality of a given policy $\bs{\pi}$ profile.  The \emph{Nash gap} \citep{cui2022provably} of $\bs{\pi}$ is defined as  $
    \textnormal{Gap}(\bs{\pi}) = \sum_{i\in[n]} V^{\dagger,\bs{\pi}_{-i}}_{i,0}(s_0) - V^{\bs{\pi}}_{i,0}(s_0)~,$ where $V^{\dagger,\bs{\pi}_{-i}}_{i,0}(s_0) = \max_{\pi'_i} V^{\pi'_i,\bs{\pi}_{-i}}_{i,0}(s_0)~.$
Note that, by definition, any Nash equilibrium has $0$ Nash gap, and any $\alpha$-Nash equilibrium has at most $\alpha$ Nash gap. 
\paragraph{Linear Markov games.} In this paper, we consider linear Markov games \citep{zhong2022pessimistic}.
    Formally, $G$ is said to be a linear Markov game with feature map $\phi:S\times A\rightarrow \mathbb{R}^d$, for some $d\in \mathbb{N}$, if we have $
        P_h(s_{h+1}|s_h,\bs{a}_h) =\langle \phi(s_h,\bs{a}_h), \xi_{h}(s_{h+1})\rangle$, and $
        \mathcal{R}_{i,h}(s_h,\bs{a}_h)= \langle \phi(s_h,\bs{a}_h), \theta^*_{i,h}\rangle + \zeta_{i,h}, \forall (s_h,\bs{a}_h, i, h) \in S\times A\times [n] \times [H-1]~,$
    where $\xi_h$ and $\theta^*_{i}$ are unknown parameters and $\zeta_{i,h}$ zero-mean $\gamma^2$-subGaussian noise. Here, $\norm{\phi(s,\bs{a})}_2\leq 1$ for all state-action tuples $(s,\bs{a})\in S\times A$, $\max\{\norm{\theta^*_{i,h}}_2,\norm{\xi_{h}(s)}_2\}\leq \sqrt{d}$, for all $i\in [n]$ and $h\in [H-1]$. Let $\Theta$ denote the set of all feasible $\theta$ as defined here. 

\begin{remark}
    There are two main reasons why we consider linear Markov games to model our problem. First, the corruption-robust offline RL literature in the general function approximation \citep{ye2023corruption} consider a corruption model which is defined in terms of Bellman residuals. Since we only assume access to preference data, this type of corruption model is not well-defined for our setting. Second, relaxing the linearity of rewards would then require corruption-robust maximum likelihood estimation procedures beyond generalized linear models, which, to the best of our knowledge, are not present in the current literature.
\end{remark}

\subsection{Preference Data}

Following the formulation in \citep{zhang2024multi}, we denote by $\widetilde{D}= \left\{ (\widetilde{\tau}_i,\widetilde{\tau}'_i,\widetilde{o}_i)\right\}^{m}_{i=1}$ the clean preference dataset, where $\widetilde{\tau}=(\widetilde{s}_0,\widetilde{a}_{1,0},\widetilde{a}_{2,0},\ldots,\widetilde{a}_{n,0},\widetilde{s}_1,\ldots,\widetilde{s}_{H-1})$ and $\widetilde{\tau}'=(\widetilde{s}'_0,\widetilde{a}'_{1,0},\widetilde{a}'_{2,0},\ldots,\widetilde{a}'_{n,0},\widetilde{s}'_1,\ldots,\widetilde{s}'_{H-1})$ denote sampled trajectories from behavior policies $\bs{\mu}$ and $\bs{\mu}_\textnormal{ref}$, respectively, and $\widetilde{o}_i=(\widetilde{o}_{i,1},\ldots,\widetilde{o}_{i,n})$, with $\widetilde{o}_{i,j}\in \{-1,+1\}$, for all $j\in[n]$, gives information about individual preferences of agents for each pair of trajectories: $\widetilde{o}_{i,j}=1$ implies that $\widetilde{\tau}_i$ is preferred to $\widetilde{\tau}'_i$ for agent $j$. We assume that the preferences are generated according to the Bradley-Terry (BT) model \citep{bradley1952rank}: for each agent $j$, we assume that we have 
\begin{align*}
    & \mathbb{P}(\widetilde{o}_{i,j}=1|\widetilde{\tau}_i,\widetilde{\tau}'_i) \\ & = \sigma\left( \sum^{H-1}_{h=0}R_{i,h}(\widetilde{s}_h,\widetilde{\bs{a}}_h) - \sum^{H-1}_{h=0} R_{i,h}(\widetilde{s}'_h,\widetilde{\bs{a}}'_h)\right)~,
\end{align*}
with $\sigma(x) = 1/(1+\exp (-x))$ being the sigmoid function.

\subsection{Corruption Model}

Following the $\epsilon$-corruption model in offline RL(HF) \citep{zhang2022corruption, mandal2024corruption}, we assume that there exists an attacker that has full access to the dataset $\widetilde{D}$ and arbitrarily perturbs an $\epsilon$-fraction of it. That is, given $\epsilon \in [0,1/2)$, we assume that the attacker inspects $\widetilde{D}$ and modifies up to $\epsilon \cdot m$ samples in $\widetilde{D}$. We denote by $D$ the poisoned dataset. In other words, there are at most $\epsilon\cdot m$ data samples in $D$ such that $(\tau,\tau',o)\neq(\widetilde{\tau},\widetilde{\tau}',\widetilde{o})$.\footnote{Following prior literature on corruption-robust RL \citep{zhang2022corruption}, we assume that, for each $h\in[H-1]$, the subset of $D$ containing only the $h$ steps of the samples, is also consistent with the $\epsilon$-corruption model.}

\subsection{Data Coverage}

Offline learning problems necessitate access to a dataset that contains, at least to some extent, "good" samples, in the sense that they are traversed by policies we are trying to approximate. This condition is usually described by the notion of \emph{data coverage}. In linear Markov games (MG), data coverage is captured by the feature covariance matrix. Formally, for every $h\in[H-1]$, we define
$$\Sigma_{\bs{\mu}}(h) = \expctu{\bs{\mu}_h}{\phi(s_h,\bs{a}_h)\phi(s_h,\bs{a}_h)^\top}$$ and $$\Sigma^-_{\bs{\mu},\bs{\mu}_\textnormal{ref}} = \expctu{{\bs{\mu}}, {\bs{\mu}_\textnormal{ref}}}{\left(\phi(\tau)-\phi(\tau')\right)\left(\phi(\tau)-\phi(\tau')\right)^\top}$$ as the feature covariance matrices that will determine coverage of our given data. Here $\phi(\tau)=\sum^{H-1}_{h=0}\phi(s_h,\bs{a}_h)$. Note that $\Sigma_{\bs{\mu}}(h)$ is the standard covariance matrix used in the literature of offline RL, while $\Sigma^-_{\bs{\mu},\bs{\mu}_\textnormal{ref}}$ is the difference covariance matrix which has been previously used in RLHF literature \citep{zhan2023provable}. 

\section{ROBUST NE LEARNING UNDER UNIFORM COVERAGE}\label{sec:uniform_setting}

We start by considering the case when the preference dataset has uniform coverage, that is, all basis directions of the feature space are sufficiently covered. We state this assumption below. 
\begin{assumption}[Uniform Coverage]\label{asmp:uniform_coverage}
    Let $\bs{\mu}$ and $\bs{\mu}_\textnormal{ref}$ be the behavior policies that were used to generate trajectories present in the data $\widetilde{D}$. We assume that we have $\Sigma^-_{\bs{\mu},\bs{\mu}_\textnormal{ref}} \succeq \xi_R H\cdot I$ and $\Sigma_{\bs{\mu}}(h) \succeq \xi_P \cdot I$, for all $h\in[H-1]$, where $\xi_R$ and $\xi_P$ are strictly positive constants and $A \succeq B$ is equivalent to $x^\top(A-B)x\geq 0$, for all vectors $x\neq 0$.
\end{assumption}
\begin{remark}
    Note that we require coverage with respect to two covariance matrices. The first one, $\Sigma^-_{\bs{\mu},\bs{\mu}_\textnormal{ref}}$, captures coverage of the rewards, since rewards are estimated in terms of feature differences. The second one, $\Sigma_{\bs{\mu}}(h)$, captures coverage of transitions for each step $h$. In standard RL, the second condition is enough to provide guarantees. However, in preference-based RL, the first condition is necessary due to the lack of the reward signal in the given data \citep{zhan2023provable}.
\end{remark}

On a high level, all our algorithms are based on the following pipeline. They first use the preference data to robustly estimate each agent's reward parameters. Then, using those rewards, they proceed to compute robust pessimistic and optimistic estimates of the individual $Q$-functions for policies of interest. Finally, they estimate the gap and output a joint policy that minimizes it. We will instantiate different versions of this pipeline under different coverage assumptions and notions of equilibria. 
\begin{align*}
    & \textnormal{Robust Reward Estimation}\\ & \quad \quad \Rightarrow \textnormal{Robust Q-function Estimation} \\ & \quad\quad\quad \quad \Rightarrow \textnormal{Estimated Gap Minimization}
\end{align*}

\subsection{Algorithm} 

\looseness-1The main idea of our proposed algorithm is as follows. First, note that our overall objective is to find a joint policy that minimizes the Nash gap with respect to the ground-truth reward functions. However, we have access neither to these functions, nor to the real environment. We are only given a finite preference dataset $D$, an $\epsilon$-fraction of which is arbitrarily corrupted. Therefore, our first step is to compute robust estimates of the reward parameters of each agent. Note that, for linear rewards, maximum likelihood estimation becomes standard logistic regression. And for such an objective, it is known \citep{awasthi2022trimmed} that we can recover the true parameter of interest from $\epsilon$-corrupted data with $O(\epsilon^{1-o(1)})$ accuracy via a robust method called \texttt{TrimmedMLE} (for a pseudocode, see Algorithm \ref{alg:regularized_mle} in Appendix \ref{sec:additional_algorithms}). Thus, for each agent $i$, we let $\widetilde{\theta}_i = \texttt{TrimmedMLE}\left(D,\epsilon,\nu\right)$,
where we denote by $\theta_i$ the $Hd$-dimensional result of the concatenation of $\theta_{i,h}$, for $h\in[H-1]$, and $\nu$ denotes a granularity hyperparameter.
\begin{algorithm}[!ht]
    \caption{Corruption-robust Equilibrium Learning from Human Feedback with Uniform Coverage}\label{alg:pmael}
    \begin{algorithmic}[1]
        \REQUIRE Preference dataset $D$; confidence parameter $\delta$.
        \STATE Split $D$ into equal $D_1$ and $D_2$.
        \LINECOMMENT{Reward Estimation via Trimmed MLE using $D_1$.}
        \FOR{$i \in [n]$}
        \STATE Compute $\widetilde{\theta}_i = \texttt{TrimmedMLE}(D_1,\epsilon,\nu)$.
        \STATE Set the optimistic and pessimistic rewards $\overline{R}_{i,h}(\cdot,\cdot)$ and $\underline{R}_{i,h}(\cdot,\cdot) $ as in Equations \eqref{eq:opt_reward_est} and \eqref{eq:pess_reward_est}, respectively.
        \LINECOMMENT{Robust Value Function Estimation Phase using $D_2$.}
        \FOR{$\bs{\pi}\in\Pi^\textnormal{PP}$}
        \STATE Apply Algorithm \ref{alg:robust_value_estimation} on $D_2$ using only the preferred trajectories generated by $\bs{\mu}$, with input  $\bs{\pi}$, $i$, $\underline{R}_i$, and $\overline{R}_i$, and bonus function $\Gamma(\cdot,\cdot)=0$, to obtain $\overline{V}^{\dagger,\bs{\pi}_{-i}}_{i,h}(\cdot)$ and $\underline{V}^{\bs{\pi}}_{i,h}(\cdot)$, for all $h\in[H-1]$.
        \ENDFOR
        \ENDFOR
        \LINECOMMENT{Nash Gap Estimation Phase}
        \FOR{every policy $\bs{\pi}\in\Pi^\textnormal{PP}$:}
        \STATE Compute the estimated gap $\widetilde{\textnormal{Gap}}(\bs{\pi})$.
        \ENDFOR
        \RETURN $\widetilde{\bs{\pi}} \in \arg\min_{\bs{\pi}\in\Pi^\textnormal{PP}}  \widetilde{\textnormal{Gap}}(\bs{\pi})$.
    \end{algorithmic}
\end{algorithm}

Next, since we are in the offline setting, the most reasonable approach is to apply pessimism with respect to the recovered parameters. To that end, we form confidence sets, for each agent $i$, based on the \texttt{TrimmedMLE} guarantees:
\begin{align}\label{eq:uniform_confidence_set}
    & \Theta_\textnormal{Unif}(\widetilde{\theta}_i) = \left\{ \theta\in \Theta : \norm{\widetilde{\theta}_i -\theta}_2 \leq O\left(\frac{\epsilon^{1-o(1)}}{\xi_R}\right) \right\}~,
\end{align}
where $\delta > 0$ is a randomness parameter. Once we have access to the confidence set, we compute the boundary parameters
\vspace{-0.25cm}
\begin{align}
    \overline{R}_{i,h}(s,\bs{a}) & =\max_{\theta_i\in\Theta_\textnormal{Unif}(\widetilde{\theta}_i)} \theta_{i,h}^\top\phi(s,\bs{a}),\label{eq:opt_reward_est} \\ 
    \underline{R}_{i,h}(s,\bs{a})  & =\min_{\theta_i\in\Theta_\textnormal{Unif}(\widetilde{\theta}_i)} \theta_{i,h}^\top\phi(s,\bs{a})\label{eq:pess_reward_est}~.
\end{align}
Note that the inner problems are convex programs that have closed-form solutions. Now that we have access to our estimate reward functions, we need to minimize the Nash gap with respect to these rewards. In order to do that, we apply backward induction. First, we initialize the value function estimates $\underline{V}^{\bs{\pi}}_H(\cdot)=\overline{V}^{\dagger,\bs{\pi}_{-i}}_H(\cdot)=0$, for every joint policy $\bs{\pi}$. Then, for every step $h$ down to $0$, we apply a robust estimation algorithm \texttt{Rob-Q} (see Algorithm \ref{alg:robust_q-value_estimation}) for the Q-values. Essentially, the procedure first robustly estimates the parameters of the Bellman operator using a \texttt{RobEst} oracle which is guaranteed to return an $O(\epsilon)$-close value parameter under uniform coverage \citep{zhang2022corruption}. Then, it computes the estimated Q-values by properly clipping the bonus-inflated (deflated) estimates so that they remain in $[-H\sqrt{d},H\sqrt{d}]$.
\begin{algorithm}
    \caption{Robust Estimation of Q-Functions (\texttt{Rob-Q})}\label{alg:robust_q-value_estimation}
    \begin{algorithmic}[1]
        \REQUIRE Dataset $D$; corruption level $\epsilon$; policy $\bs{\pi}$; agent $i$; reward functions $\overline{R}_{i,h}$ and $\underline{R}_{i,h}$; step $h$; next-step value estimates $\underline{V}^{\bs{\pi}}_{i,h+1}(\cdot),\overline{V}^{\dagger,\bs{\pi}_{-i}}_{i,h+1}(\cdot)$; bonus $\Gamma(\cdot,\cdot)$.
        \STATE Set $\underline{\omega}^{\bs{\pi}}_{i,h} $ as \\$\texttt{RobEst}\left(\phi(s_h,\bs{a}_h),\underline{R}_{i,h}(s_h,\bs{a}_h) + \underline{V}^{\bs{\pi}}_{i,h+1}(s_{h+1})\right)$.
        \STATE Set $\overline{\omega}^{\dagger,\bs{\pi}_{-i}}_{i,h}$ as \\ $ \texttt{RobEst}\left(\phi(s_h,\bs{a}_h),\overline{R}_{i,h}(s_h,\bs{a}_h) + \overline{V}^{\dagger,\bs{\pi}_{-i}}_{i,h+1}(s_{h+1})\right)$.
        \STATE Set $\underline{Q}_{i,h}^{\bs{\pi}}(\cdot,\cdot) $ as \\ $ \textnormal{Clip}_{[-(H-h)\sqrt{d},(H-h)\sqrt{d}]}\left( \phi(\cdot,\cdot)^\top\underline{\omega}^{\bs{\pi}}_{i,h}-\Gamma(\cdot,\cdot)\right)$.
        \STATE Set $\overline{Q}_{i,h}^{\dagger,\bs{\pi}_{-i}}(\cdot,\cdot) $ as \\ $ \textnormal{Clip}_{[-(H-h)\sqrt{d},(H-h)\sqrt{d}]}\left( \phi(\cdot,\cdot)^\top\overline{\omega}^{\dagger,\bs{\pi}_{-i}}_{i,h} + \Gamma(\cdot,\cdot)\right)$.
        \RETURN Q-functions $\underline{Q}^{\bs{\pi}}_{i,h}(\cdot,\cdot)$ and $\overline{Q}^{\dagger,\bs{\pi}_{-i}}_{i,h}(\cdot,\cdot)$.
    \end{algorithmic}
\end{algorithm}
Once we have the $Q$-functions, the value function estimates  $\underline{V}^{\bs{\pi}}_{i,h}\left(\cdot\right)$ (and $\overline{V}^{\dagger,\bs{\pi}_{-i}}_{i,h}\left(\cdot\right)$) for all steps, defined with respect to the estimated reward parameters, are then computed by taking expectations over (and taking $\max$ over actions for player $i$) the given policies.\footnote{Note that Algorithm \ref{alg:robust_value_estimation} runs Algorithm \ref{alg:robust_q-value_estimation} for $H$ steps and finally returns the estimates of the value functions. We have used Algorithm \ref{alg:robust_value_estimation} to present Algorithm \ref{alg:pmael} for ease of presentation. However, Algorithm \ref{alg:robust_q-value_estimation} will be necessary in the following sections.} Once we do this for every policy, we then return the policy $\widetilde{\bs{\pi}}$ that minimizes the estimated gap with respect to $(\overline{R},\underline{R}) =(\overline{R}_1,\underline{R}_1,\ldots,\overline{R}_n,\underline{R}_n)$:
\begin{align}\label{eq:uniform_estimated_gap}
    \arg\min_{\bs{\pi}} \widetilde{\textnormal{Gap}}\left(\bs{\pi},\overline{R},\underline{R}\right) : = \sum_{i\in[n]}\overline{V}^{\dagger,\bs{\pi}_{-i}}_{i,0}\left(s_0\right) - \underline{V}^{\bs{\pi}}_{i,0}\left(s_0\right)~.
\end{align}
As shown in Appendix (see Lemma \ref{lem:initial_value_bounds}), our optimistic and pessimistic value estimates are high-probability approximates of the true value function, which implies that the estimated gap is a high-probability upper bound on the actual Nash gap. Minimizing this surrogate gap therefore serves as a proxy for minimizing the true Nash gap—and, as the gap approaches zero, the resulting joint policy correspondingly approaches a Nash equilibrium.  Algorithm \ref{alg:pmael} provides the pseudocode for the full procedure. 

\subsection{Theoretical guarantees} 

In this section, we state the theoretical guarantees on the convergence of the Algorithm \ref{alg:pmael}. Proofs can be found in Appendix \ref{appendix:1}.
\begin{theorem}\label{thm:uniform_coverage_bounds}
    Let $\epsilon\in [0,1/2)$,  $\delta >0$ and $\Gamma(\cdot,\cdot)=0$. Furthermore, assume that $ m \geq \Omega( (H^{3/2}/\epsilon^2)(d+\log(n/\delta)))$. Then, under Assumption \ref{asmp:uniform_coverage} with $\xi_R \geq 5\epsilon$, for some positive constant $c$, there exist robust algorithms \texttt{TrimmedMLE} and \texttt{RobEst} such that, with probability at least $1-\delta$, the output $\widetilde{\bs{\pi}}$ of Algorithm \ref{alg:pmael} satisfies
    \begin{align*}
        \textnormal{Gap}(\widetilde{\bs{\pi}}) & \leq O\left(Hn\left(\frac{\exp\left(H+\sqrt{\log(n/2\delta\epsilon)}\right)}{\xi_R} \right. \right. \\ & \left. \left. +\frac{H\sqrt{d}+\gamma}{\xi_P}\right)\cdot \epsilon + Hn\sqrt{\frac{(H\sqrt{d}+\gamma)^2\textnormal{poly}(d)}{\xi^2_P m}}\right)
    \end{align*}
\end{theorem}
\begin{remark}
    Note that the bounds of Theorem \ref{thm:uniform_coverage_bounds} have a quasi-linear dependence on $\epsilon$, which is known to be optimal in the single agent setting \citep{zhang2022corruption} and the two-agent zero-sum setting \citep{nika2024corruption}. This is due to the strong coverage assumption on the data. In practice, the data may cover only some directions of interest, in which case, a different approach is needed.
\end{remark}

\section{ROBUST NE LEARNING UNDER UNILATERAL COVERAGE}\label{sec:unilateral_setting}

In the previous section, we proposed an algorithm that returns an $O(n\epsilon+n/\sqrt{m})$-approximate Nash equilibrium under uniform coverage. However, such coverage is rarely possible in practice. The purpose of this section is to solve the Nash gap minimization problem under a more relaxed notion of coverage, namely \emph{unilateral coverage}, which simply requires coverage of a Nash policy and all its unilateral deviations for each agent. We extend the notion of low relative uncertainty of \citep{zhong2022pessimistic} to the MARLHF setting.
\begin{assumption}[Unilateral Coverage]\label{asmp:low_relative_uncertainty}
    Given Nash equilibrium $\bs{\pi}^*$, we assume that there exist positive constants $C_R$ and $C_P$, such that, for all $h\in[H-1]$ and $i\in \{1,\ldots,n\}$,
    \begin{align*}
        \Sigma^-_{\bs{\rho},\bs{\rho}'} & \succeq C_R\cdot \Sigma^-_{(\pi_i,\bs{\pi}^*_{-i}),\bs{\rho}'}~,\;\;\textnormal{for}\;\; \bs{\rho},\bs{\rho}'\in\{\bs{\mu},\bs{\mu}_\textnormal{ref}\},\\ & \textnormal{and}\;\;  \Sigma_{\bs{\mu}}(h) \succeq C_P\cdot \Sigma_{\pi_i,\bs{\pi}^*_{-i}}(h)~.
    \end{align*}
\end{assumption}
The first condition simply says that behavior policies $\bs{\mu}$ and $\bs{\mu}_\textnormal{ref}$ sufficiently cover a Nash equilibrium and its unilateral deviations in the feature space. Different from single-agent RL, where single policy concentrability is enough to provide theoretical guarantees, its extension to Markov games, where only Nash policies are covered does not allow for any guarantees. Unilateral coverage is in fact necessary and sufficient to provide any meaningful guarantees in zero-sum \citep{zhong2022pessimistic} and general-sum Markov games \citep{zhang2023offline}. 

\subsection{Algorithm} 

\looseness-1For the uniform coverage setting, we applied \texttt{TrimmedMLE} to obtain estimates for the ground-truth reward parameters. The benefit of such an approach is that, under such coverage, it comes with bounds on the $\ell^2$-norm of the error, which then allows for defining our confidence set in terms of such error bounds.  This, in turn, allows us to directly upper bound the difference between value functions and their estimates in terms of $\ell^2$-difference of their respective reward parameters. When we do not have uniform coverage, the final estimate is not guaranteed to remain close to the true parameter in the $\ell^2$ sense. In this case, as shown in Appendix (see Lemma \ref{lem:log_prob_bound_unilateral}), the parameters are close in the log-sigmoid sense. Given the output $\widetilde{\theta}_i$ of \texttt{TrimmedMLE}, we define the confidence set for the unilateral coverage setting as 
\begin{align}\label{eq:confidence_set_unilateral}
    & \Theta_\textnormal{Unil}(\widetilde{\theta}_i) =\Bigg\{ \theta\in\Theta:  \nonumber\\ &  \frac{2}{m}\sum_{(\tau,\tau',o)\in D}\log\frac{\sigma\left(o\cdot\widetilde{\theta}_i^\top\left(\phi(\tau)-\phi(\tau')\right)\right)}{ \sigma\left(o\cdot\theta^\top\left(\phi(\tau)-\phi(\tau')\right)\right)} \leq \kappa \Bigg\}~,
\end{align}
where $\kappa = 6\epsilon H\sqrt{d} + (2d/m)\cdot\log(Hm/\delta)$ controls the `radius' of the confidence set. We can provide theoretical guarantees that $\Theta_\textnormal{Unil}(\widetilde{\theta}_i)$ contains the ground-truth parameter $\theta^*_i$ with high probability.
\begin{algorithm}
    \caption{Reward Parameter Estimation (\texttt{RewardEst})}\label{alg:reward_estimation}
    \begin{algorithmic}[1]
        \REQUIRE Dataset $D$; corruption level $\epsilon$; confidence parameter $\delta$; learning rate $\eta$; slackness parameter $\nu$; number of steps $T$.
                \FOR{$i \in [n]$}
        \STATE Let $\widetilde{\theta}_i = \texttt{TrimmedMLE}(D,\epsilon,\nu)$.
        \STATE Initialize $\widehat{\theta}^{(0)}_i$ uniformly at random in $\Theta_\textnormal{Unil}(\widetilde{\theta}_i)$ (defined in Equation \eqref{eq:confidence_set_unilateral}).
        \FOR{$t=0,1,\ldots,T-1$}
        \STATE Take gradient step
        \begin{align*}
            \widehat{\theta}^{(t+1)}_i = \mathcal{P}_{\Theta_\textnormal{Unil}(\widetilde{\theta}_i)}\left(\widehat{\theta}^{(t)}_i + \eta \widetilde{\nabla}_{\theta_i} \textnormal{Gap}\left(\bs{\pi}^*,\widehat{\bs{\theta}}^{(t)}\right)\right)~.
        \end{align*}
        \ENDFOR
        \STATE Set $\widehat{\theta}_i=(1/T)\sum^T_{t=1}\widehat{\theta}^{(t)}_i$.
        \ENDFOR
        \RETURN $\widehat{\bs{\theta}}=(\widehat{\theta}_1,\ldots,\widehat{\theta}_n)$.
    \end{algorithmic}
\end{algorithm}
\begin{algorithm}[!ht]
    \caption{Corruption-robust Nash Equilibrium Learning from Human Feedback}\label{alg:pmael_unilateral}
    \begin{algorithmic}[1]
        \REQUIRE Dataset $D$; corruption level $\epsilon$; regularization parameter $\lambda$; confidence parameter $\delta$; learning rate $\eta_1$; bonus functions $\Gamma(\cdot,\cdot)$; slackness parameter $\nu$; number of gradient steps $T_1$.
        \STATE Split $D$ into equal $D_1$ and $D_2$.
        \STATE Compute $\bs{\widehat{\theta}} = \texttt{RewardEst}\left(D_1,\epsilon,\delta,\eta_1,\nu,T_1\right)$.
        \FOR{$i\in[n]$}
        \FOR{$\bs{\pi}\in\Pi^\textnormal{PP}$}
        \STATE Initialize $\underline{V}^{\bs{\pi}}_{i,H}(\cdot)=\overline{V}^{\dagger,\bs{\pi}_{-i}}_{i,H}(\cdot)=0$.
        \FOR{$h=H-1,\ldots,0$}
        \STATE Compute $\widehat{R}_{i,h}(\cdot,\cdot) = (\widehat{\theta}_{i,h})^\top\phi(\cdot,\cdot)$ to be the estimated reward.
        \STATE Compute $\left( \underline{Q}^{\bs{\pi}}_{i,h}(\cdot,\cdot),\overline{Q}^{\dagger,\bs{\pi}_{-i}}_{i,h}(\cdot,\cdot)\right) = \texttt{Rob-Q}\left(D_{2,h},\bs{\pi},\epsilon,\widehat{R}_{i,h},\underline{V}^{\bs{\pi}}_{i,h+1},\overline{V}^{\dagger,\bs{\pi}_{-i}}_{i,h+1},\Gamma\right)$.
        \STATE Set $\underline{V}^{\bs{\pi}}_{i,h}(\cdot)=\expctu{\bs{a}\sim\bs{\pi}_h}{\underline{Q}^{\bs{\pi}}_{i,h}(\cdot,\bs{a})}$ and $\overline{V}^{\dagger,\bs{\pi}_{-i}}_{i,h}(\cdot)=\max_{a_i}\expctu{\bs{a}_{-i}\sim\bs{\pi}_{-i,h}}{\overline{Q}^{\dagger,\bs{\pi}_{-i}}_{i,h}(\cdot,\bs{a})}$.
        \ENDFOR
        \ENDFOR
        \ENDFOR
        \RETURN $\widetilde{\bs{\pi}} \in \arg\min_{\bs{\pi}\in\Pi^\textnormal{PP}}  \widetilde{\textnormal{Gap}}(\bs{\pi},\widehat{\bs{\theta}})$.
    \end{algorithmic}
\end{algorithm}
Unfortunately though, the same analysis does not go through just by choosing reward estimates that maximize (minimize) over the confidence set, due to this more complicated notion of closeness. Hence, we follow a different approach in this section. Intuitively, if we can find $\theta$ in our confidence set that maximizes \textit{the gap of our output policy}, and, similarly, find a policy $\bs{\pi}$ that minimizes \textit{the gap with respect to this choice} of $\theta$, we can finally bound the true gap on the ground-truth reward. This intuition is based on the observation that any Nash gaps of policies that used $\theta$ parameters in the confidence set should be close to each-other. 

Based on the above discussion, we will utilize projected gradient ascent (PGA) to update our estimates of $\theta$. However, we do not have access to the true gap given parameter $\theta$, neither of its gradient with respect to $\theta$. We thus resort to using biased estimates of it. 
\looseness-1As we show in Appendix (see Lemma \ref{lem:main_gap_bounds_unilaterl}), for any $\bs{\theta}:=(\theta_1,\ldots,\theta_n)\in\Theta_\textnormal{Unil}(\widetilde{\theta}_i)\times\ldots,\times\Theta_\textnormal{Unil}(\widetilde{\theta}_n)$, any policy $\bs{\pi}$ that is the minimizer of the estimated gap computed via \texttt{Rob-Q} on $\bs{\theta}$, satisfies 
$$| \textnormal{Gap}(\bs{\pi},\bs{\theta})-\textnormal{Gap}(\bs{\pi}^*,\bs{\theta})|\leq O(n\sqrt{\epsilon} + n/\sqrt{m})~,$$ for Nash equilibrium policy $\bs{\pi}^*$ which is covered by $D$, where $\textnormal{Gap}(\bs{\pi},\bs{\theta})$ here denotes the true Nash gap of $\bs{\pi}$ under reward function parameterized by $\bs{\theta}$. 
Therefore, we optimize $\textnormal{Gap}(\bs{\pi}^*,\bs{\theta})$ as a surrogate objective, and later transfer guarantees to $\textnormal{Gap}(\bs{\pi},\bs{\theta})$ using Lemma \ref{lem:main_gap_bounds_unilaterl}.

\looseness-1However, we do not have access to $\nabla_{\bs{\theta}}\textnormal{Gap}(\bs{\pi}^*,\bs{\theta})$. Here, we use the following observation: in the linear setting, the (sub)gradient of the gap becomes the average feature differences over convex combinations of occupancy measures. Thus, we can use $\nabla_{\bs{\theta}} \sum_{i\in[n]}(V^{\bs{\mu}}_{i,0}(s_0,\theta_i)-V^{\bs{\mu}_\textnormal{ref}}_{i,0}(s_0,\theta_i))$ as an estimate for $\nabla_{\bs{\theta}}\textnormal{Gap}(\bs{\pi}^*,\bs{\theta})$, since $\bs{\mu}$ and $\bs{\mu}_\textnormal{ref}$ already cover $\bs{\pi}^*$ and its unilateral deviations. Here, $V^{\bs{\mu}}_{i,0}(s_0,\theta_i)$ denotes the value function of $\bs{\mu}$ with respect to reward function parametrized by $\theta_i$. To estimate  $\nabla_{\theta_i} (V^{\bs{\mu}}_{i,0}(s_0,\theta_i)- V^{\bs{\mu}_\textnormal{ref}}_{i,0}(s_0,\theta_i))$, we use a robust mean oracle \texttt{RobMean} that takes as input corrupted feature differences and returns a $O(\sqrt{\epsilon})$-approximate estimate of their true mean, which in our case is just $\nabla_{\theta_i} (V^{\bs{\mu}}_{i,0}(s_0,\theta_i)- V^{\bs{\mu}_\textnormal{ref}}_{i,0}(s_0,\theta_i))$. We thus define our gradient estimate $\widetilde{\nabla}_{\theta_i} \textnormal{Gap}\left({\bs{\pi}}^*,\bs{\theta}\right)$ with respect to $\theta_i$ as
\begin{align*}
     \sum^{H-1}_{h=0}\texttt{RobMean}\left(D^{\bs{\mu}}_{h,\phi}\right) - \sum^{H-1}_{h=0}\texttt{RobMean}\left(D^{\bs{\mu}_\textnormal{ref}}_{h,\phi}\right)~,
\end{align*}
where $D^{\bs{\mu}}_{h,\phi}$ and $D^{\bs{\mu}_\textnormal{ref}}_{h,\phi}$ partition each $h$-level of $D$ and store only the features of trajectories generated by $\bs{\mu}$ and $\bs{\mu}_\textnormal{ref}$, respectively. After running PGA for $T_1$ steps, we compute the empirical average of the iterates and use that to compute our estimated reward function. The reward estimation procedure \texttt{RewardEst} is described in Algorithm \ref{alg:reward_estimation}. Once we have access to this reward, we can run \texttt{Rob-Q} on it and obtain estimated gaps for each policy. 

However, lack of uniform coverage implies weaker guarantees on \texttt{Rob-Q}. Thus, we need to properly define a bonus term that accounts for corruption and lack of coverage. First, let us define a scaled sample covariance matrix with respect to the preferred trajectories in the corrupted data, using regularization parameter $\lambda \geq 0$ (to be specified later) as $$\Lambda_h = \frac{3}{5}\left(\frac{1}{m}\sum^m_{j=1}\phi(s^j_h,\bs{a}^j_h)\phi(s^j_h,\bs{a}^j_h)^\top + (\epsilon+\lambda) I\right).$$
Using $\Lambda_h$, we now define the bonus term to be used in this section as follows. For any $(s,\bs{a})$, let $$ \Gamma(s,\bs{a}) = E(d,m,\delta,\epsilon)\cdot\norm{\phi(s,\bs{a})}_{\Lambda_h^{-1}}~,$$ where $E(d,m,\delta,\epsilon) = O(\sqrt{\epsilon}+1/\sqrt{m})$ (see Appendix \ref{appendix:2} for detailed definition). We run \texttt{Rob-Q} with bonus set as $\Gamma$ and obtain estimated gaps for every joint policy. Finally, we return a joint policy that minimizes estimated gap. The full procedure is described in Algorithm \ref{alg:pmael_unilateral}. 

\subsection{Theoretical guarantees} 

In this section, we state the theoretical guarantee on the convergence of Algorithm \ref{alg:pmael_unilateral}.

\begin{theorem}\label{thm:unilateral_coverage_bounds}
    Let $\epsilon\in[0,1/2), \lambda\geq \Omega(dH\log(m/\delta)/m)$, and $\delta >0$. Set $\Theta_\textnormal{Unil}(\cdot)$ as in Equation \eqref{eq:confidence_set_unilateral} and $\Gamma(s,\bs{a})=E(d,m,\delta,\epsilon)\cdot\norm{\phi(s,\bs{a})}_{\Lambda_h^{-1}}$. Suppose Assumption \ref{asmp:low_relative_uncertainty} is satisfied and PGA is run for $T_1$ steps with learning rate $\eta= O(1/\sqrt{T_1})$. Then, there exist robust subroutines \texttt{RobEst}, \texttt{TrimmedMLE}, and \texttt{RobMean}  such that, with probability at least $1-\delta$, the output $\widetilde{\bs{\pi}}$ of Algorithm \ref{alg:pmael_unilateral} with subroutines \texttt{RobEst}, \texttt{TrimmedMLE}, \texttt{RobMean} and \texttt{RewardEst}, satisfies
    \begin{align*}
        \textnormal{Gap}(\widetilde{\bs{\pi}}) & \leq \widetilde{O}\left( \left(\frac{1}{\sqrt{C_R}}+\frac{1}{\sqrt{C_P}}+\frac{1}{\sqrt{T_1}}\right) \right. \\ & \left. \cdot\left(H^{5/2}nd^{3/4}\sqrt{\epsilon} + H^2n\sqrt{\frac{\textnormal{poly}(d)}{m}}\right)\right)~.
    \end{align*}
\end{theorem}
\begin{remark}
    Note that the order of $\epsilon$ in the above bounds is $1/2$. This deterioration comes from the relaxation of uniform coverage. This dependence is identical to that in single-agent RL \citep{zhang2022corruption}, two-player zero-sum Markov games \citep{nika2024corruption}, and single-agent RLHF \citep{mandal2024corruption} under data corruption. The currently established linear lower bounds hold under uniform coverage. It remains an open question whether weaker coverage implies tighter lower bounds. 
\end{remark}

\section{ROBUST CCE LEARNING UNDER UNILATERAL COVERAGE}

In the previous section, we provided an algorithm that was designed to compute an approximate NE using corrupted preference data under the minimal unilateral coverage assumption. However, a key bottleneck of Algorithm \ref{alg:pmael_unilateral} is the intractability of the gap-minimization step. It is well-known that even normal-form general-sum games suffer from the curse of multi-agents---computational time scales exponentially with the number of agents (actions) \citep{foster2023hardness}. Thus, to address this issue, previous work has considered more relaxed versions of the NE, such as \emph{correlated equilibria} or \emph{coarse correlated equilibria} (CCE) \citep{cui2023breaking, zhang2023offline, ma2023decentralized, song2021can}, the latter of which can be approximated using no-regret learning algorithms. 

\looseness-1A general correlated policy is defined as a set of $H$ maps $\bs{\pi}:=\{\bs{\pi}_h:\Omega\times (S\times A)^{h-1}\times S \rightarrow \Delta(A)\}_{h\in[H-1]}$, where the first argument $w\in\Omega$ is sampled from some underlying distribution. A crucial difference from Markov policies is information about prior states that is given as input. We denote by $\Pi^\textnormal{GCP}$ the space of all general correlated policies. We denote by $\Pi^\textnormal{GCP}_i$ the set of general correlated policies for agent $i$. Then, policy $\bs{\pi}^*$ is said to be an $\alpha$-CCE if there exists $\alpha\geq 0$, such that, for every agent $i$ and state $s$, we have $V^{\bs{\pi}^*}_{i,0}(s)\geq V^{\pi',\bs{\pi}^*_{-i}}_{i,0}(s) - \alpha$, for every $\pi'_i\in\Pi^\textnormal{GCP}_i$.  If $\alpha=0$, then $\bs{\pi}^*$ is said to be a CCE. Note that the only difference between NEs and CCEs is that an NE is restricted to be a product policy, while a CCE can be any arbitrary combination of individual policies in the joint action space simplex. Hereafter, we overload notation and use $\textnormal{Gap}(\boldsymbol{\pi})$ to denote the CCE gap of a joint policy $\boldsymbol{\pi}$.

There has been a lot of interest in efficiently computing approximate CCEs in Markov games using V-learning type algorithms \citep{jin2021v, wang2023breaking, cui2023breaking}. However, all these works consider the online setting, where the learner can explore the environment and increasingly gather more relevant data. We only have at our disposal offline corrupted preference data. 

\subsection{Algorithm} 

In this section, we propose an offline-learning algorithm for computing approximate CCE in linear Markov games. First, we again assume unilateral coverage on our data (Assumption \ref{asmp:low_relative_uncertainty}). Given preference data $D$, we again run \texttt{RewardEst} procedure to obtain the reward estimates $\widehat{\bs{\theta}}$. At this point, different from the previous section, we take another approach at the estimated gap minimization problem. First, for a given joint policy $\bs{\pi}$, agent $i$, state $s$, step $h$, and actions $a$ and $a^\dagger$, we define the loss $\mathcal{L}^s_i(a^\dagger,a')$ for this stage as
\begin{align*}
    \expctu{\bs{a}_{-i}\sim\bs{\pi}_{-i,h}(\cdot|s)}{\overline{Q}^{\dagger,\bs{\pi}_{-i}}_{i,h}(s,a^\dagger,\bs{a}_{-i}) - \underline{Q}^{\bs{\pi}}_{i,h}(s,a',\bs{a}_{-i})}~,
\end{align*}
where the optimistic and pessimistic $Q$-function estimates are computed using $\widehat{\bs{\theta}}$. Using this loss function, we can now express the estimated gap minimization problem at stage $h$ as
\begin{align*}
    \min_{\bs{\pi}'_h}\sum_{i\in[n]}\max_{\pi^\dagger_{i,h}} \expctu{a^\dagger_i\sim\pi^\dagger_{i,h}(\cdot|s),a'_i\sim\pi'_{i,h}(\cdot|s)}{  \mathcal{L}^s_i(a^\dagger,a')}~.
\end{align*}
Note that such an objective can be framed as a normal-form game at stage $h$. To solve the stage game, we utilize \texttt{OptimisticHedge} \citep{daskalakis2021near}, a no-regret learning algorithm which returns a $\widetilde{O}(1/T_2)$-approximate CCE of the game when run for $T_2$ iterations. Each player basically solves a $\max-\min$ problem at stage $h$ and updates its policy using a multiplicative weights update style. We run the algorithm for $T_2$ iterations and return the average joint policy. The pseudo-code for Optimistic Hedge applied to our setting is given in Algorithm \ref{alg:optimistic_hedge} (see Appendix \ref{sec:additional_algorithms}). Once we have computed our joint policy at stage $h$, we then compute the optimistic and pessimistic values by taking expectations of the $Q$-function estimates over the newly computed policies. We use these value estimates to run the next iteration $h-1$ of our algorithm. Finally, we return the joint policy $\widetilde{\bs{\pi}}$ which is a composition of the returned policies from \texttt{OptimisticHedge} at each stage $h$. The full procedure is given in Algorithm \ref{alg:pmael_unilateral_cce}. 

\subsection{Theoretical guarantees} 

\looseness-1In this section, we provide upper bounds on the CCE gap for the output of Algorithm \ref{alg:pmael_unilateral_cce}.

\begin{algorithm}[!ht]
    \caption{Corruption-robust CCE Learning from Human Feedback}\label{alg:pmael_unilateral_cce}
    \begin{algorithmic}[1]
        \REQUIRE Preference dataset $D$; regularization parameter $\lambda$; confidence parameter $\delta$; learning rates $\eta_1$, $\eta_2$; bonus functions $\Gamma(\cdot,\cdot)$; slackness parameter $\nu$; number of gradient steps $T_1$; number of optimization steps $T_2$.
        \STATE Split $D$ into equal $D_1$ and $D_2$.
        \STATE Compute $\bs{\widehat{\theta}} = \texttt{RewardEst}\left(D_1,\epsilon,\delta,\eta_1,\nu,T_1\right)$.
        \STATE Initialize $\widetilde{\bs{\pi}}$ uniformly at random and $\underline{V}^{\widetilde{\bs{\pi}}}_{i,H}(\cdot)=\overline{V}^{\dagger,\widetilde{\bs{\pi}}_{-i}}_{i,H}(\cdot)=0$, for all $i\in \{1,\ldots,n\}$.
        \FOR{$h=H-1,\ldots,0$}
        \FOR{$i=1,\ldots,n$}
        \STATE Compute $\widehat{R}_{i,h}(\cdot,\cdot) = (\widehat{\theta}_{i,h})^\top\phi(\cdot,\cdot)$ to be the estimated reward.
        \STATE Compute $\left( \underline{Q}^{\widetilde{\bs{\pi}}}_{i,h}(\cdot,\cdot),\overline{Q}^{\dagger,\widetilde{\bs{\pi}}_{-i}}_{i,h}(\cdot,\cdot)\right) = \texttt{Rob-Q}\left(D_{2,h},\widetilde{\bs{\pi}},\epsilon,\widehat{R}_{i,h},\underline{V}^{\widetilde{\bs{\pi}}}_{i,h+1},\overline{V}^{\dagger,\widetilde{\bs{\pi}}_{-i}}_{i,h+1},\Gamma\right)$.
        \STATE Compute loss $\mathcal{L}^s_i$, for states $s\in S$.
        \ENDFOR
        \STATE Compute $\widetilde{\bs{\pi}}_h(\cdot|s) = \texttt{OptimisticHedge}\left(\mathcal{L}^s_1,\ldots,\mathcal{L}^s_n,\eta_2,T_2\right)$.
        \STATE Set $\underline{V}^{\widetilde{\bs{\pi}}}_{i,h}(\cdot)=\expctu{\bs{a}\sim\widetilde{\bs{\pi}}_h}{\underline{Q}^{\widetilde{\bs{\pi}}}_{i,h}(\cdot,\bs{a})}$ and $\overline{V}^{\dagger,\widetilde{\bs{\pi}}_{-i}}_{i,h}(\cdot)=\max_{a_i}\expctu{\bs{a}_{-i}\sim\widetilde{\bs{\pi}}_{-i,h}}{\overline{Q}^{\dagger,\widetilde{\bs{\pi}}_{-i}}_{i,h}(\cdot,\bs{a})}$, for $i\in\{1,\ldots,n\}$.
        \ENDFOR
        \RETURN $\widetilde{\bs{\pi}}= (\widetilde{\bs{\pi}}_0,\ldots,\widetilde{\bs{\pi}}_{H-1})$.
    \end{algorithmic}
\end{algorithm}

\begin{theorem}\label{thm:robust_cce_learning}
    Let $\epsilon\in[0,1/2), \lambda\geq \Omega(dH\log(m/\delta)/m)$, and $\delta >0$. Set $\Theta_\textnormal{Unil}(\cdot)$ as in Equation \eqref{eq:confidence_set_unilateral} and $\Gamma(s,\bs{a})=E(d,m,\delta,\epsilon)\cdot\norm{\phi(s,\bs{a})}_{\Lambda_h^{-1}}$. Suppose Assumption \ref{asmp:low_relative_uncertainty} is satisfied, PGA is run for $T_1$ steps with learning rate $\eta_1= O(1/\sqrt{T_1})$, and \texttt{OptimisticHedge} is run for $T_2$ steps with learning rate $\eta_2=O(1/(n\log^4T_2))$. Then, there exist robust subroutines \texttt{RobEst}, \texttt{TrimmedMLE}, and \texttt{RobMean}  such that, with probability at least $1-\delta$, the output $\widetilde{\bs{\pi}}$ of Algorithm \ref{alg:pmael_unilateral_cce} with subroutines \texttt{RobEst}, \texttt{TrimmedMLE}, \texttt{RobMean} and \texttt{OptimisticHedge}, satisfies
    \begin{align*}
        \textnormal{Gap}(\widetilde{\bs{\pi}}) & \leq \widetilde{O}\left(\left(\frac{1}{\sqrt{C_R}}+\frac{1}{\sqrt{C_P}}+\frac{1}{\sqrt{T_1}}\right)\right.\\ & \left.\cdot\left(H^{5/2}nd^{3/4}\sqrt{\epsilon} + H^2n\frac{\sqrt{\textnormal{poly}(d)}}{\sqrt{m}}\right) + \frac{Hn^2}{T_2}\right)~.
    \end{align*}
\end{theorem}
\begin{remark}
    Note that we only incur an additional $O(1/T_2)$ term on the CCE Gap, which comes from applying the no-regret sub-routine \texttt{OptimisticHedge}. The benefit of such procedure is that it can be run in quasi-polynomial time in dataset size and feature dimension. The computational complexity of Algorithm \ref{alg:pmael_unilateral_cce} is
    \begin{align*}
        O \Big( (nd)^{\log(\frac{1}{\epsilon})} & + (T_1+ n H)\cdot (\text{poly}(m,d,\frac{1}{\epsilon}) \\ & + HT_2\max_i |A_i| )\Big).
    \end{align*}
\end{remark}

\section{RELATED WORK}

\looseness-1\paragraph{Reinforcement Learning from Human Feedback (RLHF)} RLHF has substantially grown in popularity in the recent years, largely due to LLMs
\citep{ziegler2019fine, nakano2021webgpt, wu2021recursively, ouyang2022training, stiennon2020learning, glaese2022improving, ramamurthy2023is, menick2022teaching, ganguli2022red, bai2022training, gao2023scaling}. Yet RLHF’s utility extends far beyond LLMs, encompassing diverse applications—from game playing \citep{christiano2017deep, warnell2018deep, knox2008tamer, macglashan2017interactive} to robotic control \citep{shin2023benchmarks, brown2019learning}. Our work is related to recent theoretical studies on (MA)RLHF \citep{zhan2023provable, zhu2023principled, zhang2024multi, li2023reinforcement, xiong2023gibbs, nika2024reward}. In particular, we consider data corruption on MARLHF. In the single player setting, \cite{nika2025policy} propose a general data-poisoning framework in RLHF, while \cite{mandal2024corruption} propose robust algorithms trained on $\epsilon$-corrupted data. The latter is the most closely related work to ours. While we share the preference-based model and the data corruption model, our setting is a generalization of the single-agent RLHF setting considered in \citep{mandal2024corruption}. This introduces a new layer of complexity: instead of maximizing value functions over single policies, our goal is to minimize the Nash gap over joint policies. This involves a different style of analysis. Algorithmically, our methods diverge in two key ways. First, instead of relying on zeroth-order oracle calls to estimate gradients, we directly approximate the gradient of the Nash gap with respect to each agent’s strategy via the biased gradient with respect to a Nash policy. This allows us to also maintain the $O(\sqrt{\epsilon})$ bounds on the gap. Second, we incorporate a quasi-polynomial-time subroutine that computes an approximate coarse correlated equilibrium (CCE) of the induced game.

\vspace{-0.1cm}
\looseness-1\paragraph{Corruption-robust Offline Reinforcement Learning (RL)} There has been a substantial body of research on adversarial attacks in (MA)RL \citep{huang2017adversarial, lin2017tactics, wu2023reward, rakhsha2021policy, rangi2022understanding, nika2024corruption, ma2023decentralized, gleave2019adversarial}. Our research relates to a specific type of adversarial attack, namely, data corruption \citep{mei2015using, xiao2015feature, rakhsha2021policy}. Our focus is on designing robust algorithms trained on corrupted data generated via $\epsilon$-corruption model (a.k.a. strong contamination model \citep{diakonikolas2019robust}). In this line of work, \cite{zhang2022corruption} first consider corruption-robust RL via linear Markov decision processes (MDP), which is later extended to linear zero-sum Markov games (MG) \citep{nika2024corruption}. Using a different contamination model, \cite{ye2023corruption} study the problem of corruption-robustness in RL with general function approximation. Our work diverges from the above in that we study strong data corruption in multi-agent reinforcement learning from human feedback, which, due to its dependence on preference data, introduces additional layers of complexity in providing robustness guarantees. 

\looseness-1\paragraph{Offline Markov Games (MG)} Our work also relates to the literature of learning in MGs \citep{tian2021online, vrancx2008decentralized, littman1994markov, littman2001value}. We model our underlying environment from which the data is generated as a linear MG \citep{zhong2022pessimistic}. We are interested in approximating notions of optimal joint policies from corrupted preference data. The primary notion of optimality in MGs is that of the Nash equilibrium (NE) \citep{nash1950equilibrium}. Due to its computational intractability in general-sum MGs, prior work has considered relaxed versions of it such as CCEs \citep{cui2023breaking, zhang2023offline, ma2023decentralized, song2021can}, and designed computationally efficient methods to compute them in the online setting \citep{jin2021v, wang2023breaking, cui2023breaking}. We depart from this line of work and consider the CCE computation problem in the offline setting, where we compute the CCE of each stage game via no-regret methods \citep{daskalakis2021near}.

\section{DISCUSSION}

\looseness-1In this paper, we studied the problem of data corruption in offline MARLHF. We proposed provable-robust algorithms, both under uniform and unilateral coverage assumptions. Finally, we proposed a computationally efficient algorithm that robustly approximates a coarse correlated equilibrium of the underlying Markov game. 

\looseness-1A key technical contribution of our work is a new way to optimize the Nash gap without access to true reward functions or their gradients. Prior single-agent RLHF approaches rely on primal-dual methods or unbiased gradients, which do not extend to general-sum Markov games due to strategic coupling. We instead introduce \textit{a biased but tractable gradient surrogate}: by leveraging the linear structure of the underlying Markov game, we approximate the gradient at a Nash equilibrium using feature expectations induced by behavior policies. Under unilateral coverage, these policies capture the occupancy measures of the equilibrium and its unilateral deviations, so their feature differences act as a proxy for the true gradient direction. Despite the bias, this estimate is accurate enough to guide projected gradient ascent over the reward confidence set, yielding $O(\sqrt{\epsilon})$ robustness guarantees. This idea—leveraging equilibrium structure to construct usable gradient surrogates from corrupted offline preference data—appears to be new and may be of independent interest.

\looseness-1Several interesting directions are worth pursuing. First, it is not clear how one can formulate the data corruption problem in MARLHF with general function approximation, and then how to design robust algorithms in that setting. Second, it would be interesting to address the open question of whether the $O(\sqrt{\epsilon})$ bound under non-uniform coverage is tight. Finally, implementing the proposed algorithms and experimentally testing them on MARL environments is another exciting future direction. 

\subsection*{Acknowledgements}
The work of Andi Nika and Goran Radanovic was funded by the Deutsche Forschungsgemeinschaft (DFG, German Research Foundation) – project number 467367360.

\doparttoc
\faketableofcontents

\bibliographystyle{plainnat}
\bibliography{bibliography}

@article{zhang2024multi,
  title={Multi-{A}gent {R}einforcement {L}earning from {H}uman {F}eedback: Data {C}overage and {A}lgorithmic {T}echniques},
  author={Zhang, Natalia and Wang, Xinqi and Cui, Qiwen and Zhou, Runlong and Kakade, Sham M and Du, Simon S},
  journal={arXiv preprint arXiv:2409.00717},
  year={2024}
}

@inproceedings{mandal2024corruption,
  title={Corruption {R}obust {O}ffline {R}einforcement {L}earning with {H}uman {F}eedback},
  author={Mandal, Debmalya and Nika, Andi and Kamalaruban, Parameswaran and Singla, Adish and Radanovi{\'c}, Goran},
  booktitle={{AISTATS}},
  year={2025}
}

@inproceedings{zhong2022pessimistic,
  title={Pessimistic {M}inimax {V}alue {I}teration: Provably {E}fficient {E}quilibrium {L}earning from {O}ffline {D}atasets},
  author={Zhong, Han and Xiong, Wei and Tan, Jiyuan and Wang, Liwei and Zhang, Tong and Wang, Zhaoran and Yang, Zhuoran},
  booktitle={{ICML}},
  Xpages={27117--27142},
  year={2022},
  Xorganization={PMLR}
}

@article{wang2023exploitability,
  title={On the {E}xploitability of {R}einforcement {L}earning with {H}uman {F}eedback for {L}arge {L}anguage {M}odels},
  author={Wang, Jiongxiao and Wu, Junlin and Chen, Muhao and Vorobeychik, Yevgeniy and Xiao, Chaowei},
  journal={CoRR},
  volume={abs/2311.09641},
  year={2023}
}

@article{vrancx2008decentralized,
  title={Decentralized {L}earning in {M}arkov {G}ames},
  author={Vrancx, Peter and Verbeeck, Katja and Now{\'e}, Ann},
  journal={IEEE Transactions on Systems, Man, and Cybernetics, Part B (Cybernetics)},
  volume={38},
  Xnumber={4},
  Xpages={976--981},
  year={2008},
  Xpublisher={IEEE}
}

@inproceedings{foster2023hardness,
  title={Hardness of {I}ndependent {L}earning and {S}parse {E}quilibrium {C}omputation in {M}arkov {G}ames},
  author={Foster, Dylan J and Golowich, Noah and Kakade, Sham M},
  booktitle={{ICML}},
  Xpages={10188--10221},
  year={2023},
  Xorganization={PMLR}
}

@inproceedings{littman1994markov,
  author       = {Michael L. Littman},
  Xeditor       = {William W. Cohen and
                  Haym Hirsh},
  title        = {Markov {G}ames as a {F}ramework for {M}ulti-Agent {R}einforcement {L}earning},
  booktitle    = {ICML},
  Xpages        = {157--163},
  Xpublisher    = {Morgan Kaufmann},
  year         = {1994},
  Xurl          = {https://doi.org/10.1016/b978-1-55860-335-6.50027-1},
  Xdoi          = {10.1016/b978-1-55860-335-6.50027-1},
  Xtimestamp    = {Mon, 24 Jun 2019 15:47:45 +0200},
  Xbiburl       = {https://dblp.org/rec/conf/icml/Littman94.bib},
  Xbibsource    = {dblp computer science bibliography, https://dblp.org}
}

@article{nash1950equilibrium,
  title={Equilibrium {P}oints in n-person {G}ames},
  author={Nash Jr, John F},
  journal={Proceedings of the {N}ational {A}cademy of {S}ciences},
  Xvolume={36},
  Xnumber={1},
  Xpages={48--49},
  year={1950},
  Xpublisher={national academy of sciences}
}

@article{xiong2023gibbs,
  title={Gibbs {S}ampling from {H}uman {F}eedback: A {P}rovable {KL}-constrained {F}ramework for {RLHF}},
  author={Xiong, Wei and Dong, Hanze and Ye, Chenlu and Zhong, Han and Jiang, Nan and Zhang, Tong},
  journal={CoRR},
  year={2023}
}

@article{li2023reinforcement,
  title={Reinforcement {L}earning with {H}uman {F}eedback: Learning {D}ynamic {C}hoices via {P}essimism},
  author={Li, Zihao and Yang, Zhuoran and Wang, Mengdi},
  journal={arXiv preprint arXiv:2305.18438},
  year={2023}
}

@inproceedings{tian2021online,
  title={Online {L}earning in {U}nknown {M}arkov {G}ames},
  author={Tian, Yi and Wang, Yuanhao and Yu, Tiancheng and Sra, Suvrit},
  booktitle={ICML},
  Xpages={10279--10288},
  year={2021},
  Xorganization={PMLR}
}

@article{littman2001value,
  title={Value-function {R}einforcement {L}earning in {M}arkov {G}ames},
  author={Littman, Michael L},
  journal={Cognitive Systems Research},
  Xvolume={2},
  Xnumber={1},
  Xpages={55--66},
  year={2001},
  Xpublisher={Elsevier}
}

@article{shi2023badgpt,
  title={Badgpt: Exploring {S}ecurity {V}ulnerabilities of {C}hat{GPT} via {B}ackdoor {A}ttacks to {I}nstruct{GPT}},
  author={Shi, Jiawen and Liu, Yixin and Zhou, Pan and Sun, Lichao},
  journal={CoRR},
  volume={abs/2304.12298},
  year={2023}
}

@article{huang2017adversarial,
  title={{A}dversarial {A}ttacks on {N}eural {N}etwork {P}olicies},
  author={Huang, Sandy and Papernot, Nicolas and Goodfellow, Ian and Duan, Yan and Abbeel, Pieter},
  Xjournal={arXiv preprint arXiv:1702.02284},
  journal   = {CoRR},
  volume    = {abs/1702.02284},
  year={2017}
}

@inproceedings{lin2017tactics,
  title={Tactics of {A}dversarial {A}ttack on {D}eep {R}einforcement {L}earning {A}gents},
  author={Lin, Yen-Chen and Hong, Zhang-Wei and Liao, Yuan-Hong and Shih, Meng-Li and Liu, Ming-Yu and Sun, Min},
  booktitle={IJCAI},
  Xpages={3756--3762},
  year={2017}
}

@inproceedings{wu2023reward,
  author       = {Young Wu and
                  Jeremy McMahan and
                  Xiaojin Zhu and
                  Qiaomin Xie},
  Xeditor       = {Brian Williams and
                  Yiling Chen and
                  Jennifer Neville},
  title        = {Reward {P}oisoning {A}ttacks on {O}ffline {M}ulti-Agent {R}einforcement {L}earning},
  booktitle    = {{AAAI}},
  Xpages        = {10426--10434},
  Xpublisher    = {{AAAI} Press},
  year         = {2023},
  Xurl          = {https://doi.org/10.1609/aaai.v37i9.26240},
  Xdoi          = {10.1609/AAAI.V37I9.26240},
  Xtimestamp    = {Mon, 04 Sep 2023 16:50:28 +0200},
  Xbiburl       = {https://dblp.org/rec/conf/aaai/WuM0X23.bib},
  Xbibsource    = {dblp computer science bibliography, https://dblp.org}
}

@inproceedings{gleave2019adversarial,
  author       = {Adam Gleave and
                  Michael Dennis and
                  Cody Wild and
                  Neel Kant and
                  Sergey Levine and
                  Stuart Russell},
  title        = {{A}dversarial {P}olicies: {A}ttacking {D}eep {R}einforcement {L}earning},
  booktitle    = {ICLR},
  Xpublisher    = {OpenReview.net},
  year         = {2020},
  Xurl          = {https://openreview.net/forum?id=HJgEMpVFwB},
  Xtimestamp    = {Wed, 20 Apr 2022 13:29:52 +0200},
  Xbiburl       = {https://dblp.org/rec/conf/iclr/GleaveDWKLR20.bib},
  Xbibsource    = {dblp computer science bibliography, https://dblp.org}
}

@article{rakhsha2021policy,
  title={Policy {T}eaching in {R}einforcement {L}earning via {E}nvironment {P}oisoning {A}ttacks},
  author={Rakhsha, Amin and Radanovic, Goran and Devidze, Rati and Zhu, Xiaojin and Singla, Adish},
  journal={JMLR},
  Xvolume={22},
  Xnumber={210},
  Xpages={1--45},
  year={2021}
}

@article{rangi2022understanding,
  title={Understanding the {L}imits of {P}oisoning {A}ttacks in {E}pisodic {R}einforcement {L}earning},
  author={Rangi, Anshuka and Xu, Haifeng and Tran-Thanh, Long and Franceschetti, Massimo},
  journal={CoRR},
  volume={abs/2208.13663},
  year={2022}
}

@article{ma2023decentralized,
  title={Decentralized {R}obust {V}-{L}earning for {S}olving {M}arkov {G}ames with {M}odel {U}ncertainty},
  author={Ma, Shaocong and Chen, Ziyi and Zou, Shaofeng and Zhou, Yi},
  journal={{JMLR}},
  Xvolume={24},
  Xnumber={371},
  Xpages={1--40},
  year={2023}
}

@article{diakonikolas2019robust,
  title={Robust estimators in high-dimensions without the computational intractability},
  author={Diakonikolas, Ilias and Kamath, Gautam and Kane, Daniel and Li, Jerry and Moitra, Ankur and Stewart, Alistair},
  journal={SIAM Journal on Computing},
  volume={48},
  number={2},
  pages={742--864},
  year={2019},
  publisher={SIAM}
}

@inproceedings{nika2024reward,
  title={Reward {M}odel {L}earning vs. {D}irect {P}olicy {O}ptimization: A {C}omparative {A}nalysis of {L}earning from {H}uman {P}references},
  author={Nika, Andi and Mandal, Debmalya and Kamalaruban, Parameswaran and Tzannetos, Georgios and Radanovi{\'c}, Goran and Singla, Adish},
  booktitle={{ICML}},
  year={2024}
}

@inproceedings{mei2015using,
  title={Using {M}achine {T}eaching to {I}dentify {O}ptimal {T}raining-set {A}ttacks on {M}achine {L}earners},
  author={Mei, Shike and Zhu, Xiaojin},
  booktitle={AAAI},
  year={2015}
}

@inproceedings{xiao2015feature,
  title={Is {F}eature {S}election {S}ecure {A}gainst {T}raining {D}ata {P}oisoning?},
  author={Xiao, Huang and Biggio, Battista and Brown, Gavin and Fumera, Giorgio and Eckert, Claudia and Roli, Fabio},
  booktitle={ICML},
  Xpages={1689--1698},
  year={2015},
  Xorganization={PMLR}
}

@article{song2021can,
  title={When {C}an {W}e {L}earn {G}eneral-sum {M}arkov {G}ames with a {L}arge {N}umber of {P}layers {S}ample-{E}fficiently?},
  author={Song, Ziang and Mei, Song and Bai, Yu},
  journal={arXiv preprint arXiv:2110.04184},
  year={2021}
}

@inproceedings{rando2023universal,
  title={{U}niversal {J}ailbreak {B}ackdoors from {P}oisoned {H}uman {F}eedback},
  author={Rando, Javier and Tram{\`e}r, Florian},
  Xbooktitle={International Conference on Learning Representations},
  booktitle={ICLR},
  year={2023}
}

@article{baumgartner2024best,
  title={{B}est-of-{V}enom: {A}ttacking {RLHF} by {I}njecting {P}oisoned {P}reference {D}ata},
  author={Baumg{\"a}rtner, Tim and Gao, Yang and Alon, Dana and Metzler, Donald},
  journal={CoRR},
  volume={abs/2404.05530},
  year={2024}
}

@inproceedings{nika2025policy,
  title={Policy {T}eaching via {D}ata Poisoning in {L}earning from {H}uman {P}references},
  author={Nika, Andi and N{\"o}ther, Jonathan and Mandal, Debmalya and Kamalaruban, Parameswaran and Singla, Adish and Radanovi{\'c}, Goran},
  booktitle={{AISTATS}},
  year={2025}
}

@inproceedings{zhang2023offline,
  title={Offline {L}earning in {M}arkov {G}ames with {G}eneral {F}unction {A}pproximation},
  author={Zhang, Yuheng and Bai, Yu and Jiang, Nan},
  booktitle={{ICML}},
  Xpages={40804--40829},
  year={2023},
  Xorganization={PMLR}
}

@article{awasthi2022trimmed,
  title={Trimmed {M}aximum {L}ikelihood {E}stimation for {R}obust {G}eneralized {L}inear {M}odel},
  author={Awasthi, Pranjal and Das, Abhimanyu and Kong, Weihao and Sen, Rajat},
  journal={{NeurIPS}},
  Xvolume={35},
  Xpages={862--873},
  year={2022}
}

@inproceedings{ye2023corruption,
  title={Corruption-robust {O}ffline {R}einforcement {L}earning with {G}eneral {F}unction {A}pproximation},
  author={Ye, Chenlu and Yang, Rui and Gu, Quanquan and Zhang, Tong},
  booktitle={{NeurIPS}},
  Xvolume={36},
  Xpages={36208--36221},
  year={2023}
}

@inproceedings{daskalakis2021near,
  title={Near-optimal {N}o-regret {L}earning in {G}eneral {G}ames},
  author={Daskalakis, Constantinos and Fishelson, Maxwell and Golowich, Noah},
  booktitle={{NeurIPS}},
  Xvolume={34},
  Xpages={27604--27616},
  year={2021}
}

@inproceedings{wang2023breaking,
  title={Breaking the {C}urse of {M}ultiagency: Provably {E}fficient {D}ecentralized {M}ulti-agent {{RL}} with {F}unction {A}pproximation},
  author={Wang, Yuanhao and Liu, Qinghua and Bai, Yu and Jin, Chi},
  booktitle={{COLT}},
  Xpages={2793--2848},
  year={2023},
  Xorganization={PMLR}
}

@article{jin2021v,
  title={V-Learning: A {S}imple, {E}fficient, {D}ecentralized {A}lgorithm for {M}ultiagent {RL}},
  author={Jin, Chi and Liu, Qinghua and Wang, Yuanhao and Yu, Tiancheng},
  journal={arXiv preprint arXiv:2110.14555},
  year={2021}
}

@inproceedings{cui2023breaking,
  title={Breaking the {C}urse of {M}ultiagents in a {L}arge {S}tate {S}pace: {RL} in {M}arkov {G}ames with {I}ndependent {L}inear {F}unction {A}pproximation},
  author={Cui, Qiwen and Zhang, Kaiqing and Du, Simon},
  booktitle={{COLT}},
  Xpages={2651--2652},
  year={2023},
  Xorganization={PMLR}
}

@inproceedings{diakonikolas2020outlier,
  title={Outlier {R}obust {M}ean {E}stimation with {S}ubgaussian {R}ates via {S}tability},
  author={Diakonikolas, Ilias and Kane, Daniel M and Pensia, Ankit},
  booktitle={{NeurIPS}},
  Xvolume={33},
  Xpages={1830--1840},
  year={2020}
}

@inproceedings{nika2024corruption,
  title={Corruption-{R}obust {O}ffline {T}wo-player {Z}ero-sum {M}arkov {G}ames},
  author={Nika, Andi and Mandal, Debmalya and Singla, Adish and Radanovic, Goran},
  booktitle={{AISTATS}},
  Xpages={1243--1251},
  year={2024},
  Xorganization={PMLR}
}

@inproceedings{zanette2021cautiously,
  title={Cautiously {O}ptimistic {P}olicy {O}ptimization and {E}xploration with {L}inear {F}unction {A}pproximation},
  author={Zanette, Andrea and Cheng, Ching-An and Agarwal, Alekh},
  booktitle={{COLT}},
  Xpages={4473--4525},
  year={2021},
  Xorganization={PMLR}
}

@article{cui2022provably,
  title={Provably {E}fficient {O}ffline {M}ulti-agent {R}einforcement {L}earning via {S}trategy-wise {B}onus},
  author={Cui, Qiwen and Du, Simon S},
  journal={{NeurIPS}},
  Xvolume={35},
  Xpages={11739--11751},
  year={2022}
}

@article{ziegler2019fine,
  title={Fine-tuning {L}anguage {M}odels from {H}uman {P}references},
  author={Ziegler, Daniel M and Stiennon, Nisan and Wu, Jeffrey and Brown, Tom B and Radford, Alec and Amodei, Dario and Christiano, Paul and Irving, Geoffrey},
  journal={CoRR},
  volume={abs/1909.08593},
  year={2019}
}

@inproceedings{zhang2022corruption,
  author       = {Xuezhou Zhang and
                  Yiding Chen and
                  Xiaojin Zhu and
                  Wen Sun},
  Xeditor       = {Gustau Camps{-}Valls and
                  Francisco J. R. Ruiz and
                  Isabel Valera},
  title        = {Corruption-robust {O}ffline {R}einforcement {L}earning},
  booktitle    = {{AISTATS}},
  Xseries       = {Proceedings of Machine Learning Research},
  Xvolume       = {151},
  Xpages        = {5757--5773},
  Xpublisher    = {{PMLR}},
  year         = {2022},
  Xurl          = {https://proceedings.mlr.press/v151/zhang22c.html},
  Xtimestamp    = {Sat, 30 Sep 2023 09:34:08 +0200},
  Xbiburl       = {https://dblp.org/rec/conf/aistats/ZhangC0S22.bib},
  Xbibsource    = {dblp computer science bibliography, https://dblp.org}
}

@article{bradley1952rank,
  title={{R}ank {A}nalysis of {I}ncomplete {B}lock {D}esigns: I. {T}he {M}ethod of {P}aired {C}omparisons},
  author={Bradley, Ralph Allan and Terry, Milton E},
  journal={Biometrika},
  volume={39},
  number={3/4},
  Xpages={324--345},
  year={1952},
  Xpublisher={JSTOR}
}

@inproceedings{zhu2023principled,
  author       = {Banghua Zhu and
                  Michael I. Jordan and
                  Jiantao Jiao},
  Xeditor       = {Andreas Krause and
                  Emma Brunskill and
                  Kyunghyun Cho and
                  Barbara Engelhardt and
                  Sivan Sabato and
                  Jonathan Scarlett},
  title        = {Principled {R}einforcement {L}earning with {H}uman {F}eedback from {P}airwise
                  or {K}-wise {C}omparisons},
  booktitle    = {{ICML}},
  Xseries       = {Proceedings of Machine Learning Research},
  Xvolume       = {202},
  Xpages        = {43037--43067},
  Xpublisher    = {{PMLR}},
  year         = {2023},
  Xurl          = {https://proceedings.mlr.press/v202/zhu23f.html},
  Xtimestamp    = {Mon, 28 Aug 2023 17:23:09 +0200},
  Xbiburl       = {https://dblp.org/rec/conf/icml/ZhuJJ23.bib},
  Xbibsource    = {dblp computer science bibliography, https://dblp.org}
}

@inproceedings{stiennon2020learning,
  author       = {Stiennon, Nisan and Ouyang, Long and Wu, Jeffrey and Ziegler, Daniel and Lowe, Ryan and Voss, Chelsea and Radford, Alec and Amodei, Dario and Christiano, Paul F},
  Xeditor       = {Hugo Larochelle and
                  Marc'Aurelio Ranzato and
                  Raia Hadsell and
                  Maria{-}Florina Balcan and
                  Hsuan{-}Tien Lin},
  title        = {Learning to {S}ummarize with {H}uman {F}eedback},
  booktitle    = {{NeurIPS}},
  year         = {2020},
  Xurl          = {https://proceedings.neurips.cc/paper/2020/hash/1f89885d556929e98d3ef9b86448f951-Abstract.html},
  Xtimestamp    = {Tue, 19 Jan 2021 15:57:45 +0100},
  Xbibsource    = {dblp computer science bibliography, https://dblp.org}
}

@article{wu2021recursively,
  author       = {Wu, Jeff and Ouyang, Long and Ziegler, Daniel M and Stiennon, Nisan and Lowe, Ryan and Leike, Jan and Christiano, Paul},
  title        = {Recursively {S}ummarizing {B}ooks with {H}uman {F}eedback},
  journal      = {CoRR},
  volume       = {abs/2109.10862},
  year         = {2021},
  Xurl          = {https://arxiv.org/abs/2109.10862},
  Xeprinttype    = {arXiv},
  Xeprint       = {2109.10862},
  Xtimestamp    = {Mon, 04 Oct 2021 08:57:17 +0200},
  Xbiburl       = {https://dblp.org/rec/journals/corr/abs-2109-10862.bib},
  Xbibsource    = {dblp computer science bibliography, https://dblp.org}
}

@article{nakano2021webgpt,
  Xauthor       = {Nakano, Reiichiro and Hilton, Jacob and Balaji, Suchir and Wu, Jeff and Ouyang, Long and Kim, Christina and Hesse, Christopher and Jain, Shantanu and Kosaraju, Vineet and Saunders, William and others},
  author       = {Nakano, Reiichiro and others},
  title        = {WebGPT: Browser-assisted {Q}uestion-answering with {H}uman {F}eedback},
  journal      = {CoRR},
  volume       = {abs/2112.09332},
  year         = {2021},
  Xurl          = {https://arxiv.org/abs/2112.09332},
  Xeprinttype    = {arXiv},
  Xeprint       = {2112.09332},
  Xtimestamp    = {Mon, 03 Jan 2022 15:45:35 +0100},
  Xbiburl       = {https://dblp.org/rec/journals/corr/abs-2112-09332.bib},
  Xbibsource    = {dblp computer science bibliography, https://dblp.org}
}

@inproceedings{ouyang2022training,
  Xauthor       = {Ouyang, Long and Wu, Jeffrey and Jiang, Xu and Almeida, Diogo and Wainwright, Carroll and Mishkin, Pamela and Zhang, Chong and Agarwal, Sandhini and Slama, Katarina and Ray, Alex and others},
  author       = {Ouyang, Long and others},
  Xeditor       = {Sanmi Koyejo and
                  S. Mohamed and
                  A. Agarwal and
                  Danielle Belgrave and
                  K. Cho and
                  A. Oh},
  title        = {Training {L}anguage {M}odels to {F}ollow {I}nstructions with {H}uman {F}eedback},
  booktitle    = {{NeurIPS}},
  year         = {2022},
  Xurl          = {http://papers.nips.cc/paper\_files/paper/2022/hash/b1efde53be364a73914f58805a001731-Abstract-Conference.html},
  Xtimestamp    = {Mon, 08 Jan 2024 16:31:36 +0100},
  Xbiburl       = {https://dblp.org/rec/conf/nips/Ouyang0JAWMZASR22.bib},
  Xbibsource    = {dblp computer science bibliography, https://dblp.org}
}

@article{menick2022teaching,
  Xauthor       = {Menick, Jacob and Trebacz, Maja and Mikulik, Vladimir and Aslanides, John and Song, Francis and Chadwick, Martin and Glaese, Mia and Young, Susannah and Campbell-Gillingham, Lucy and Irving, Geoffrey and others},
  author       = {Menick, Jacob and others},
  title        = {Teaching {L}anguage {M}odels to {S}upport {A}nswers with {V}erified {Q}uotes},
  journal      = {CoRR},
  volume       = {abs/2203.11147},
  year         = {2022},
  Xurl          = {https://doi.org/10.48550/arXiv.2203.11147},
  Xdoi          = {10.48550/ARXIV.2203.11147},
  Xeprinttype    = {arXiv},
  Xeprint       = {2203.11147},
  Xtimestamp    = {Fri, 05 May 2023 15:54:56 +0200},
  Xbiburl       = {https://dblp.org/rec/journals/corr/abs-2203-11147.bib},
  Xbibsource    = {dblp computer science bibliography, https://dblp.org}
}

@article{glaese2022improving,
  Xauthor       = {Glaese, Amelia and McAleese, Nat and Tr{k{e}}bacz, Maja and Aslanides, John and Firoiu, Vlad and Ewalds, Timo and Rauh, Maribeth and Weidinger, Laura and Chadwick, Martin and Thacker, Phoebe and others},
  author       = {Glaese, Amelia and others},
  title        = {Improving {A}lignment of {D}ialogue {A}gents via {T}argeted {H}uman {J}udgements},
  journal      = {CoRR},
  volume       = {abs/2209.14375},
  year         = {2022},
  Xurl          = {https://doi.org/10.48550/arXiv.2209.14375},
  Xdoi          = {10.48550/ARXIV.2209.14375},
  Xeprinttype    = {arXiv},
  Xeprint       = {2209.14375},
  Xtimestamp    = {Fri, 05 May 2023 15:54:56 +0200},
  Xbiburl       = {https://dblp.org/rec/journals/corr/abs-2209-14375.bib},
  Xbibsource    = {dblp computer science bibliography, https://dblp.org}
}

@inproceedings{gao2023scaling,
  author       = {Leo Gao and
                  John Schulman and
                  Jacob Hilton},
  Xeditor       = {Andreas Krause and
                  Emma Brunskill and
                  Kyunghyun Cho and
                  Barbara Engelhardt and
                  Sivan Sabato and
                  Jonathan Scarlett},
  title        = {Scaling {L}aws for {R}eward {M}odel {O}veroptimization},
  booktitle    = {{ICML}},
  Xseries       = {Proceedings of Machine Learning Research},
  Xvolume       = {202},
  Xpages        = {10835--10866},
  Xpublisher    = {{PMLR}},
  year         = {2023},
  Xurl          = {https://proceedings.mlr.press/v202/gao23h.html},
  Xtimestamp    = {Mon, 28 Aug 2023 17:23:08 +0200},
  Xbiburl       = {https://dblp.org/rec/conf/icml/GaoSH23.bib},
  Xbibsource    = {dblp computer science bibliography, https://dblp.org}
}

@article{bai2022training,
  Xauthor       = {Yuntao Bai and
                  Andy Jones and
                  Kamal Ndousse and
                  Amanda Askell and
                  Anna Chen and
                  Nova DasSarma and
                  Dawn Drain and
                  Stanislav Fort and
                  Deep Ganguli and
                  Tom Henighan and
                  Nicholas Joseph and
                  Saurav Kadavath and
                  Jackson Kernion and
                  Tom Conerly and
                  Sheer El Showk and
                  Nelson Elhage and
                  Zac Hatfield{-}Dodds and
                  Danny Hernandez and
                  Tristan Hume and
                  Scott Johnston and
                  Shauna Kravec and
                  Liane Lovitt and
                  Neel Nanda and
                  Catherine Olsson and
                  Dario Amodei and
                  Tom B. Brown and
                  Jack Clark and
                  Sam McCandlish and
                  Chris Olah and
                  Benjamin Mann and
                  Jared Kaplan},
  author={Yuntao Bai and others},
  title        = {Training a {H}elpful and {H}armless {A}ssistant with {R}einforcement {L}earning from {H}uman {F}eedback},
  journal      = {CoRR},
  volume       = {abs/2204.05862},
  year         = {2022},
  Xurl          = {https://doi.org/10.48550/arXiv.2204.05862},
  Xdoi          = {10.48550/ARXIV.2204.05862},
  Xeprinttype    = {arXiv},
  Xeprint       = {2204.05862},
  Xtimestamp    = {Tue, 19 Apr 2022 17:11:58 +0200},
  Xbiburl       = {https://dblp.org/rec/journals/corr/abs-2204-05862.bib},
  Xbibsource    = {dblp computer science bibliography, https://dblp.org}
}

@article{ganguli2022red,
  Xauthor       = {Ganguli, D and Lovitt, L and Kernion, J and Askell, A and Bai, Y and Kadavath, S and Mann, B and Perez, E and Schiefer, N and Ndousse, K and others},
  author       = {Ganguli, D and others},
  title        = {Red {T}eaming {L}anguage {M}odels to {R}educe {H}arms: Methods, {S}caling {B}ehaviors,
                  and {L}essons {L}earned},
  journal      = {CoRR},
  volume       = {abs/2209.07858},
  year         = {2022},
 xurl          = {https://doi.org/10.48550/arXiv.2209.07858},
  Xdoi          = {10.48550/ARXIV.2209.07858},
  Xeprinttype    = {arXiv},
  Xeprint       = {2209.07858},
  Xtimestamp    = {Tue, 27 Sep 2022 16:29:43 +0200},
  Xbiburl       = {https://dblp.org/rec/journals/corr/abs-2209-07858.bib},
  Xbibsource    = {dblp computer science bibliography, https://dblp.org}
}

@inproceedings{ramamurthy2023is,
  author       = {Ramamurthy, Rajkumar and Ammanabrolu, Prithviraj and Brantley, Kiant{\'e} and Hessel, Jack and Sifa, Rafet and Bauckhage, Christian and Hajishirzi, Hannaneh and Choi, Yejin},
  title        = {Is {R}einforcement {L}earning (not) for {N}atural {L}anguage {P}rocessing: Benchmarks, {B}aselines, and {B}uilding {B}locks for {N}atural {L}anguage {P}olicy {O}ptimization},
  booktitle    = {{ICLR}},
  Xpublisher    = {OpenReview.net},
  year         = {2023}
}

@inproceedings{brown2019learning,
  author       = {Daniel S. Brown and
                  Wonjoon Goo and
                  Prabhat Nagarajan and
                  Scott Niekum},
  Xeditor       = {Kamalika Chaudhuri and
                  Ruslan Salakhutdinov},
  title        = {Extrapolating {B}eyond {S}uboptimal {D}emonstrations via {I}nverse {R}einforcement {L}earning from {O}bservations},
  booktitle    = {{ICML}},
  Xseries       = {Proceedings of Machine Learning Research},
  year         = {2019}
}

@article{shin2023benchmarks,
  author       = {Daniel Shin and
                  Anca D. Dragan and
                  Daniel S. Brown},
  title        = {Benchmarks and {A}lgorithms for {O}ffline {P}reference-{B}ased {R}eward {L}earning},
  Xjournal      = {Transactions of Machine Learning Research},
  journal      = {Transactions of Machine Learning Research},
  Xvolume       = {2023},
  year         = {2023},
  Xurl          = {https://openreview.net/forum?id=TGuXXlbKsn},
  Xtimestamp    = {Thu, 18 May 2023 14:15:45 +0200},
  Xbiburl       = {https://dblp.org/rec/journals/tmlr/ShinDB23.bib},
 Xbibsource    = {dblp computer science bibliography, https://dblp.org}
}

@inproceedings{knox2008tamer,
  title={Tamer: Training an {A}gent {M}anually via {E}valuative {R}einforcement},
  author={Knox, W Bradley and Stone, Peter},
  booktitle={{ICDL}},
  Xpages={292--297},
  year={2008},
  Xorganization={IEEE}
}

@inproceedings{macglashan2017interactive,
  author       = {MacGlashan, James and Ho, Mark K and Loftin, Robert and Peng, Bei and Wang, Guan and Roberts, David L and Taylor, Matthew E and Littman, Michael L},
  Xeditor       = {Doina Precup and
                  Yee Whye Teh},
  title        = {Interactive {L}earning from {P}olicy-Dependent {H}uman {F}eedback},
  booktitle    = {{ICML}},
  Xseries       = {Proceedings of Machine Learning Research},
  Xvolume       = {70},
  Xpages        = {2285--2294},
  Xpublisher    = {{PMLR}},
  year         = {2017},
  Xurl          = {http://proceedings.mlr.press/v70/macglashan17a.html},
  Xtimestamp    = {Sun, 18 Dec 2022 19:02:44 +0100},
  Xbiburl       = {https://dblp.org/rec/conf/icml/MacGlashanHLPWR17.bib},
  Xbibsource    = {dblp computer science bibliography, https://dblp.org}
}

@inproceedings{christiano2017deep,
  author       = {Christiano, Paul F and Leike, Jan and Brown, Tom and Martic, Miljan and Legg, Shane and Amodei, Dario},
  title        = {Deep {R}einforcement {L}earning from {H}uman {P}references},
  booktitle    = {{NeurIPS}},
  Xpages        = {4299--4307},
  year         = {2017}
}

@inproceedings{warnell2018deep,
  author       = {Warnell, Garrett and Waytowich, Nicholas and Lawhern, Vernon and Stone, Peter},
  title        = {Deep {TAMER:} {I}nteractive {A}gent {S}haping in {H}igh-Dimensional {S}tate {S}paces},
  booktitle    = {{AAAI}},
  Xpages        = {1545--1554},
  Xpublisher    = {{AAAI} Press},
  year         = {2018},
  Xurl          = {https://doi.org/10.1609/aaai.v32i1.11485},
  Xdoi          = {10.1609/AAAI.V32I1.11485},
  Xtimestamp    = {Mon, 04 Sep 2023 16:50:27 +0200},
  Xbiburl       = {https://dblp.org/rec/conf/aaai/WarnellWLS18.bib},
  Xbibsource    = {dblp computer science bibliography, https://dblp.org}
}

@article{zhan2023provable,
  title={Provable {O}ffline {R}einforcement {L}earning with {H}uman {F}eedback},
  author={Zhan, Wenhao and Uehara, Masatoshi and Kallus, Nathan and Lee, Jason D and Sun, Wen},
  journal={CoRR},
  volume={abs/:2305.14816},
  year={2023}
}

@inproceedings{diakonikolas2025sos,
  title={Sos {C}ertifiability of {S}ubgaussian {D}istributions and its {A}lgorithmic {A}pplications},
  author={Diakonikolas, Ilias and Hopkins, Samuel B and Pensia, Ankit and Tiegel, Stefan},
  booktitle={Proceedings of the 57th Annual ACM Symposium on Theory of Computing},
  pages={1689--1700},
  year={2025}
}

@article{dong2019quantum,
  title={Quantum {E}ntropy {S}coring for {F}ast {R}obust {M}ean {E}stimation and {I}mproved {O}utlier {D}etection},
  author={Dong, Yihe and Hopkins, Samuel and Li, Jerry},
  journal={Advances in Neural Information Processing Systems},
  volume={32},
  year={2019}
}


\newpage
\onecolumn
\appendix
\begin{center}
    \LARGE \textbf{Corruption-robust Offline Multi-agent Reinforcement Learning from Human Feedback}
\end{center}
\addcontentsline{toc}{section}{Appendix}
\part{Appendix}
\parttoc

\section{Proof of Theorem \ref{thm:uniform_coverage_bounds}}\label{appendix:1}

In this section, we provide the full proof for Theorem \ref{thm:uniform_coverage_bounds}.

\begin{lemma}\label{lem:bound_on_reward_for_full_sample}
    [Lemma A.1 of \cite{mandal2024corruption}] Let Assumption \ref{asmp:uniform_coverage} hold with $\xi_R \geq 5\epsilon$, for some positive constant $c$, and let 
    \begin{align*}
        m \geq \Omega\left( \frac{H^{3/2}}{\epsilon^2}\left(d+\log(n/\delta)\right) \right)~.
    \end{align*}
    Then, for every $i$, Algorithm \ref{alg:regularized_mle} returns an estimator $\widetilde{\theta}_i$ such that, with probability at least $1-\delta/2$, satisfies
    \begin{align*}
        \norm{\widetilde{\theta}_i-\theta^*_i}_2 \leq O\left(\frac{\epsilon}{\xi_R}\exp\left(H+ \sqrt{\log(n/2\delta\epsilon)}\right) \right)~,
    \end{align*}
    where $\widetilde{\theta}_i$ denotes the $Hd$-dimensional vector with sub-vectors $\widetilde{\theta}_{i,h}$ for every $h$.
\end{lemma}
\begin{proof}
    This result is an immediate application of Lemma A.1 of \citep{mandal2024corruption} to the multi-agent setting by applying the union bound over $n$ agents.
\end{proof}

This upper bound provides us with a provable threshold function for our confidence sets. Next, we will make use of the following result.
\begin{theorem}\label{thm:rls_guarantee} \citep{zhang2022corruption}
Given an $\epsilon$-corrupted dataset $D=\{ x_i,y_i\}_{i \in [m]}$, where the clean data is generated as $\widetilde{x}_i \sim \beta$, $\mathbb{P}(\norm{\widetilde{x}_i}\leq 1) = 1$, $\widetilde{y}_i = \widetilde{x}_i^\top \omega^*+\zeta_i$, where $\zeta_i$ is zero-mean $\sigma^2$-variance sub-Gaussian random noise, then a robust least square estimator returns an estimator $\omega$ such that:
\begin{itemize}
    \item If $\mathbb{E}_\beta[xx^\top]\succeq \xi I$, then with probability at least $1-\delta/2$, we have $$\norm{\omega^* - \omega}_2 \leq c_1(\delta) \cdot \left( \sqrt{\frac{\sigma^2\textnormal{poly}(d)}{\xi^2 m}} + \frac{\sigma}{\xi}\epsilon \right);$$
    \item With probability at least $1-\delta/2$, we have $$\mathbb{E}_\beta\left[ \norm{\widetilde{x}^\top(\omega^* - \omega)}^2_2\right] \leq c_2(\delta) \cdot \left( \frac{\sigma^2\textnormal{poly}(d)}{m} + \sigma^2\epsilon \right),$$
\end{itemize}
where $c_1$ and $c_2$ hide constants and $\textnormal{polylog} (1/\delta)$ terms.
\end{theorem}
Applying this to our setting means considering the corrupted Bellman operator samples from our data as signals generated from an unknown underlying distribution. We thus define, for every $i,h,s,\bs{a}\in [n]\times [H-1]\times S\times A$, the Bellman operator as 
\begin{align*}
    \mathbb{B}_{i,h}V_{i,h+1}(s,\bs{a}) = R_{i,h}(s,\bs{a}) + \sum_{s'\in S}P(s'|s,\bs{a})V_{i,h+1}(s')~.
\end{align*}
We also define the Bellman operator with respect to the estimated rewards as
\begin{align}\label{eq:lower_bellman}
     \underline{\mathbb{B}}_{i,h}\underline{V}^{\bs{\pi}}_{i,h+1}(s,\bs{a}) = \underline{R}^{\bs{\pi}}_{i,h}(s,\bs{a}) + \sum_{s'\in S}P(s'|s,\bs{a})\underline{V}^{\bs{\pi}}_{i,h+1}(s')~,
\end{align}
and 
\begin{align}\label{eq:upper_bellman}
    \overline{\mathbb{B}}_{i,h}\overline{V}^{\dagger,\bs{\pi}_{-i}}_{i,h+1}(s,\bs{a}) = \underline{R}^{\dagger,\bs{\pi}_
    {-i}}_{i,h}(s,\bs{a}) + \sum_{s'\in S}P(s'|s,\bs{a})\overline{V}^{\dagger,\bs{\pi}_{-i}}_{i,h+1}(s')~.
\end{align}
We then have the following result.
\begin{lemma}\label{lem:variance_of_noise}
    We have, for every tuple $(s_h,\bs{a}_h,s_{h+1})$ in $D$,
    \begin{align*}
        Var\left(\underline{R}_{i,h}(s_h,\bs{a}_h) + \underline{V}_{i,h+1}(s_{h+1})- \underline{\mathbb{B}}_{i,h}\underline{V}_{i,h+1}(s_h,\bs{a}_h)| s_h,\bs{a}_h\right) \leq (H\sqrt{d}+\gamma)^2~,
    \end{align*}
    and 
    \begin{align*}
        Var\left(\overline{R}_{i,h}(s_h,\bs{a}_h) + \overline{V}_{i,h+1}(s_{h+1})- \overline{\mathbb{B}}_{i,h}\overline{V}_{i,h+1}(s_h,\bs{a}_h)| s_h,\bs{a}_h\right) \leq (H\sqrt{d}+\gamma)^2~.
    \end{align*}
\end{lemma}
\begin{proof}
    Note that we have
    \begin{align*}
         Var & \left(\underline{R}_{i,h}(s_h,\bs{a}_h) + \underline{V}_{i,h+1}(s_{h+1})- \underline{\mathbb{B}}_{i,h}\underline{V}_{i,h+1}(s_h,\bs{a}_h)| s_h,\bs{a}_h\right) \\ & = \expct{\left(\underline{R}_{i,h}(s_h,\bs{a}_h) + \underline{V}_{i,h+1}(s_{h+1})-\expct{\underline{R}_{i,h}(s_h,\bs{a}_h) + \underline{V}_{i,h+1}(s_{h+1})}\right)^2} \\
            & \leq Var(\underline{R}_{i,h}(s_h,\bs{a}_h)) + Var(\underline{V}_{i,h+1}(s_{h+1})) \\
        & \leq (H\sqrt{d}+\gamma)^2~,
    \end{align*} 
    since both $H$ and $\gamma$ are nonnegative numbers. The proof of the second statement is similar. 
\end{proof}
Using the above, we will define the error stated above using short-hand notation for ease of presentation:
\begin{align}\label{eq:RLS_error}
    E_1(d,m,\delta,\epsilon) = c_1(\delta) \cdot \left( \sqrt{\frac{(H\sqrt{d}+\gamma)^2\textnormal{poly}(d)}{\xi^2_P m}} + \frac{H\sqrt{d}+\gamma}{\xi_P}\epsilon \right)~.
\end{align}

Next, we prove upper bounds on the maximum and minimum values of estimated reward functions in terms of ground-truth rewards.
\begin{lemma}\label{lem:reward_bounds}
    With probability at least $1-\delta/2$, we have 
    \begin{align*}
        R_{i,h}(s,\bs{a}) - C_1\frac{\epsilon\cdot\exp\left(H+ \sqrt{\log(n/2\delta\epsilon)}\right)}{\xi_R} & \leq \underline{R}_{i,h}(s,\bs{a}) \leq R_{i,h}(s,\bs{a})  \\ & \leq \overline{R}_{i,h}(s,\bs{a}) \leq R_{i,h}(s,\bs{a}) + C_1\frac{\epsilon\cdot\exp\left(H+ \sqrt{\log(n/2\delta\epsilon)}\right)}{\xi_R} ~.
    \end{align*}
\end{lemma}
\begin{proof}
    Let $\overline{\theta}_{i,h}$ be the parameter that corresponds to $\overline{R}_{i,h}$ and let $\underline{\theta}_{i,h}$ be defined similarly. Observe that, for every agent $i$ and time-step $h$, we have
    \begin{align*}
        \left|\overline{R}_{i,h}(s,\bs{a}) - R_{i,h}(s,\bs{a})\right| & =  \left|\left\langle\phi(s,\bs{a}),\overline{\theta}_{i,h}\right\rangle - \left\langle\phi(s,\bs{a}),\theta^*_{i,h}\right\rangle\right| \\
        & \leq\left|\left\langle\phi(s,\bs{a}),\overline{\theta}_{i,h}\right\rangle - \left\langle\phi(s,\bs{a}),\theta^*_{i,h}\right\rangle\right| \\
            & = \left|\left\langle\phi(s,\bs{a}),\overline{\theta}_{i,h}\right\rangle - \left\langle\phi(s,\bs{a}),\widetilde{\theta}_{i,h}\right\rangle +  \left\langle\phi(s,\bs{a}),\widetilde{\theta}_{i,h}\right\rangle -\left\langle\phi(s,\bs{a}),\theta^*_{i,h}\right\rangle\right|\\
            & = \left|\left\langle \phi(s,\bs{a}), \widetilde{\theta}_{i,h}-\overline{\theta}_{i,h} \right\rangle + \left\langle\phi(s,\bs{a}),\overline{\theta}_{i,h}- \theta^*_{i,h}\right\rangle\right|\\
        & \leq \norm{\phi(s,\bs{a})}_2\norm{\widetilde{\theta}_{i,h} - \overline{\theta}_{i,h}}_2 +  \norm{\phi(s,\bs{a})}_2\norm{\overline{\theta}_{i,h}-\theta^*_{i,h}}_2 \\
            & \leq \norm{\widetilde{\theta}_{i,h}-\overline{\theta}_{i,h}}_2 + \norm{\overline{\theta}_{i,h}-\theta^*_{i,h}}_2\\
        & \leq C_1\frac{\epsilon\cdot\exp\left(H+ \sqrt{\log(n/2\delta\epsilon)}\right)}{\xi_R}~,
    \end{align*}
    where the first equality follows by definition, and the fact that the true expected rewards are already in $[-\sqrt{d},\sqrt{d}]$; we have used Cauchy-Schwarz for the second inequality, the fact that $\norm{\phi(s,\bs{a})}_2\leq 1$ by assumption, and Lemma \ref{lem:bound_on_reward_for_full_sample} for the final inequality. Similarly, we have
    \begin{align*}
        \left|R_{i,h}(s,\bs{a}) - \underline{R}_{i,h}(s,\bs{a})\right| \leq C_1\frac{\epsilon\cdot\exp\left(H+ \sqrt{\log(n/2\delta\epsilon)}\right)}{\xi_R}~.
    \end{align*}
\end{proof}
Next, we have the following result that provides bounds on the Bellman errors.
\begin{lemma}\label{lem:bellman_operator_bounds}
    With probability at least $1-\delta/2$, we have, for every $i,h,s,\bs{a}$ and policy profile $\bs{\pi}$, 
    \begin{align*}
        -E_1(d,m,\delta,\epsilon) & \leq \underline{\mathbb{B}}_{i,h}\underline{V}^{\bs{\pi}}_{i,h+1}(s,\bs{a}) - \underline{Q}^{\bs{\pi}}_{i,h}(s,\bs{a}) \leq E_1(d,m,\delta,\epsilon)~,\\
        -E_1(d,m,\delta,\epsilon) & \leq \overline{\mathbb{B}}_{i,h}\overline{V}^{\dagger,\bs{\pi}_{-i}}_{i,h+1}(s,\bs{a}) - \overline{Q}^{\dagger,\bs{\pi}_{-i}}_{i,h}(s,\bs{a}) \leq E_1(d,m,\delta,\epsilon)~.
    \end{align*}
\end{lemma}
\begin{proof}
    First, as noted in \citep{zhong2022pessimistic}, in linear MDPs, the value functions are also linear in features. Thus, denoting by $\underline{\omega}^{\bs{\pi},*}_{i,h}$ the parameter of the Bellman transform of $\underline{V}^{\bs{\pi}}_{i,h}$ and defining $\overline{\omega}^{\dagger,\bs{\pi}_{-i},*}_{i,h}$ similarly, we have
    \begin{align*}
         \left|\phi(s,\bs{a})^\top\underline{\omega}^{\bs{\pi},*}_{i,h}-\underline{\mathbb{B}}_{i,h}\underline{V}^{\bs{\pi}}_{i,h+1}(s,\bs{a}) \right| & = \left\langle \phi(s,\bs{a}),\underline{\omega}^{\bs{\pi},*}_{i,h}- \underline{\omega}^{\bs{\pi}}_{i,h} \right\rangle \\ & \leq \norm{\phi(s,\bs{a})}_2 \norm{ \underline{\omega}^{\bs{\pi},*}_{i,h}-\underline{\omega}^{\bs{\pi}}_{i,h}}_2 \\ &  \leq E_1(d,m,\delta,\epsilon)~,
    \end{align*}
    where the penultimate step uses Cauchy-Schwarz and the final step uses the feature norm assumption and Theorem \ref{thm:rls_guarantee}.
\end{proof}

Next, we prove a similar result for the estimated value functions and best responses.
\begin{lemma}\label{lem:value_bounds}
   Under the event of Lemma \ref{lem:bellman_operator_bounds}, we have, for every agent $i$, state $s$, step $h$ and policy $\bs{\pi}$:
    \begin{align*}
        \underline{V}^{\bs{\pi}}_{i,h}(s) & \leq V_{i,h}^{\bs{\pi}}(s)+E_1(d,m,\delta,\epsilon) + C_1\frac{\epsilon\cdot\exp\left(H+ \sqrt{\log(n/2\delta\epsilon)}\right)}{\xi_R},\;\; \text{and}\\ \overline{V}^{\dagger,\bs{\pi}_{-i}}_{i,h}(s) & \geq V_{i,h}^{\dagger,\bs{\pi}_{-i}}(s)-E_1(d,m,\delta,\epsilon) - C_1\frac{\epsilon\cdot\exp\left(H+ \sqrt{\log(n/2\delta\epsilon)}\right)}{\xi_R}~.
    \end{align*}
\end{lemma}
\begin{proof}
    We prove the result by induction. Let $\underline{V}^s_{i,H}(\bs{\pi})=0$. The result holds for step $H$. Suppose the result holds for step $h+1$. Then, for step $h$ we have
    \begin{align}
        & \underline{V}^{\bs{\pi}}_{i,h}(s) = \expctu{\bs{a}\sim\bs{\pi}_h}{\underline{Q}^{\bs{\pi}}_{i,h}(s,\bs{a})}\label{eq:value_bounding_eq_01} \\
            & \leq \expctu{\bs{a}\sim\bs{\pi}_h}{\underline{\mathbb{B}}_{i,h}\underline{V}_{i,h+1}^{\bs{\pi}}(s,\bs{a})} + E_1(d,m,\delta,\epsilon)\label{eq:value_bounding_eq_02}\\
        & \leq \expctu{\bs{a}\sim\bs{\pi}_h}{\underline{\mathbb{B}}_{i,h}{V}_{i,h+1}^{\bs{\pi}}(s,\bs{a})} + E_1(d,m,\delta,\epsilon)\label{eq:value_bounding_eq_03}\\
            & = \expctu{\bs{a}\sim\bs{\pi}_h}{\underline{R}_{i,h}(s,\bs{a}) + \sum_{s'}P(s'|s,\bs{a})V^{\bs{\pi}}_{i,h+1}(s')} + E_1(d,m,\delta,\epsilon) \nonumber\\
        & \leq \expctu{\bs{a}\sim\bs{\pi}_h}{R_{i,h}(s,\bs{a}) + \sum_{s'}P(s'|s,\bs{a})V^{\bs{\pi}}_{i,h+1}(s')} + C_1\frac{\epsilon\cdot\exp\left(H+ \sqrt{\log(n/2\delta\epsilon)}\right)}{\xi_R} + E_1(d,m,\delta,\epsilon) \label{eq:value_bounding_eq_03.1}\\
            & =  \expctu{\bs{a}\sim\bs{\pi}_h}{\mathbb{B}_{i,h}V^{\bs{\pi}}_{i,h+1}(s,\bs{a})} + C_1\frac{\epsilon\cdot\exp\left(H+ \sqrt{\log(n/2\delta\epsilon)}\right)}{\xi_R} + E_1(d,m,\delta,\epsilon) \nonumber\\
        & = V^{\bs{\pi}}_{i,h}(s) + C_1\frac{\epsilon\cdot\exp\left(H+ \sqrt{\log(n/2\delta\epsilon)}\right)}{\xi_R} + E_1(d,m,\delta,\epsilon)\label{eq:value_bounding_eq_04}
    \end{align}
    where Equation \eqref{eq:value_bounding_eq_01} follows by definition; Equation \eqref{eq:value_bounding_eq_02} follows from Lemma \ref{lem:bellman_operator_bounds}; Equation \eqref{eq:value_bounding_eq_03} follows by the inductive assumption; Equation \eqref{eq:value_bounding_eq_03.1} follows from Lemma \ref{lem:reward_bounds};  Equation \eqref{eq:value_bounding_eq_04} follows by definition. Similarly,
    \begin{align}
        \overline{V}_{i,h}^{\dagger,\bs{\pi}_{-i}}(s) & = \max_{a_i\in A_i}\expctu{\bs{a}_{-i}\sim\bs{\pi}_{-i,h}(\cdot|s)}{\overline{Q}^{\dagger,\bs{\pi}_{-i}}_{i,h}(s,\bs{a})} \label{eq:value_bounding_eq_05}\\ 
            & \geq \max_{a_i\in A_i}\expctu{\bs{a}_{-i}\sim\bs{\pi}_{-i,h}(\cdot|s)}{\overline{\mathbb{B}}_{i,h}\overline{V}^{\dagger,\bs{\pi}_{-i}}_{i,h+1}(s,\bs{a})} - E_1(d,m,\delta,\epsilon)\label{eq:value_bounding_eq_06}\\
        & \geq \max_{a_i\in A_i}\expctu{\bs{a}_{-i}\sim\bs{\pi}_{-i,h}(\cdot|s)}{\overline{\mathbb{B}}_{i,h}V^{\dagger,\bs{\pi}_{-i}}_{i,h+1}(s,\bs{a})} - E_1(d,m,\delta,\epsilon) \label{eq:value_bounding_eq_06.1}\\
            & \geq  \max_{a_i\in A_i}\expctu{\bs{a}_{-i}\sim\bs{\pi}_{-i,h}(\cdot|s)}{\mathbb{B}_{i,h}{V}^{\dagger,\bs{\pi}_{-i}}_{i,h+1}(s,\bs{a})} - C_1\frac{\epsilon\cdot\exp\left(H+ \sqrt{\log(n/2\delta\epsilon)}\right)}{\xi_R} -E_1(d,m,\delta,\epsilon)\label{eq:value_bounding_eq_07}\\
        & = {V}^{\dagger,\bs{\pi}_{-i}}_{i,h}(s) - C_1\frac{\epsilon\cdot\exp\left(H+ \sqrt{\log(n/2\delta\epsilon)}\right)}{\xi_R} - E_1(d,m,\delta,\epsilon)\label{eq:value_bounding_eq_08}~,
    \end{align}
    where again Equation \eqref{eq:value_bounding_eq_05} follows by definition; Equation \eqref{eq:value_bounding_eq_06} follows from Lemma \ref{lem:bellman_operator_bounds}; Equation \eqref{eq:value_bounding_eq_06.1} follows from the linearity and monotonicity of $\overline{\mathbb{B}}_{i,h}$; Equation \eqref{eq:value_bounding_eq_07} follows by the inductive assumption and Lemma \ref{lem:reward_bounds}, and \eqref{eq:value_bounding_eq_08} follows by definition.
\end{proof}
Next, we provide bounds on the difference between the estimated and true expected returns.
\begin{lemma}\label{lem:initial_value_bounds}
    Under the event of Lemma \ref{lem:bellman_operator_bounds} we have, for any $\bs{\pi}\in\Pi^\textnormal{PP}$,
    \begin{align*}
        V^{\bs{\pi}}_{i,0}(s_0) - \underline{V}^{\bs{\pi}}_{i,0}(s_0) & \leq HC_1\frac{\epsilon\cdot\exp\left(H+ \sqrt{\log(n/2\delta\epsilon)}\right)}{\xi_R} + HE_1(d,m,\delta,\epsilon)~, \\
        V^{\bs{\pi}}_{i,0}(s_0) - \overline{V}^{\bs{\pi}}_{i,0}(s_0) & \leq HC_1\frac{\epsilon\cdot\exp\left(H+ \sqrt{\log(n/2\delta\epsilon)}\right)}{\xi_R} + HE_1(d,m,\delta,\epsilon)~.
    \end{align*}
\end{lemma}
\begin{proof}
    Note that we have
    \begin{align}
        & V^{\bs{\pi}}_{i,0}(s_0) - \underline{V}^{\bs{\pi}}_{i,0}(s_0) = \expctu{\bs{a}_0\sim\bs{\pi}_0(\cdot|s_0)}{Q^{\bs{\pi}}_{i,0}(s_0,\bs{a}_0)} - \expctu{\bs{a}_0\sim\bs{\pi}_0(\cdot|s_0)}{\underline{Q}^{\bs{\pi}}_{i,0}(s_0,\bs{a}_0)} \nonumber\\
            & \leq  \expctu{\bs{a}_0\sim\bs{\pi}_0(\cdot|s_0)}{Q^{\bs{\pi}}_{i,0}(s_0,\bs{a}_0) - \underline{\mathbb{B}}_{i,0}\underline{V}_{i,1}^{\bs{\pi}}(s_0,\bs{a}_0) + E_1(d,m,\delta,\epsilon)} \label{eq:value_gap_eq_01}\\
        & = \expctu{\bs{a}_0\sim\bs{\pi}_0(\cdot|s_0)}{\mathbb{B}_{i,0}V^{\bs{\pi}}_{i,1}(s_0,\bs{a}_0) - \underline{\mathbb{B}}_{i,0}\underline{V}_{i,1}^{\bs{\pi}}(s_0,\bs{a}_0) + E_1(d,m,\delta,\epsilon)} \label{eq:value_gap_eq_02}\\
            & = \expctu{\bs{a}_0\sim\bs{\pi}_0(\cdot|s_0)}{R_{i,0}(s_0,\bs{a}_0)-\underline{R}_{i,0}(s_0,\bs{a}_0) + \expctu{s_1\sim P(\cdot|s_0,\bs{a}_0)}{V_{i,1}(s_1)-\underline{V}_{i,1}(s_1)} + E_1(d,m,\delta,\epsilon)} \label{eq:value_gap_eq_03}\\
        & \leq \expctu{\bs{a}_0\sim\bs{\pi}_0(\cdot|s_0)}{ \expctu{s_1\sim P(\cdot|s_0,\bs{a}_0)}{V_{i,1}(s_1)-\underline{V}_{i,1}(s_1)} + C_1\frac{\epsilon\cdot\exp\left(H+ \sqrt{\log(n/2\delta\epsilon)}\right)}{\xi_R} + E_1(d,m,\delta,\epsilon)} \label{eq:value_gap_eq_04}\\ 
            & = C_1\frac{\epsilon\cdot\exp\left(H+ \sqrt{\log(n/2\delta\epsilon)}\right)}{\xi_R} + E_1(d,m,\delta,\epsilon) + \expctu{\bs{a}_0\sim\bs{\pi}_0(\cdot|s_0),s_1\sim P(\cdot|s_0,\bs{a}_0)}{V^{\bs{\pi}}_{i,1}(s_1) - \underline{V}^{\bs{\pi}}_{i,1}(s_1)} \nonumber\\
        & = \ldots \nonumber\\
            & \leq H\left(C_1\frac{\epsilon\cdot\exp\left(H+ \sqrt{\log(n/2\delta\epsilon)}\right)}{\xi_R} + E_1(d,m,\delta,\epsilon)\right)~,\label{eq:value_gap_eq_05}
    \end{align}
    where Equation \eqref{eq:value_gap_eq_01} follows from Lemma \ref{lem:bellman_operator_bounds}; Equation \eqref{eq:value_gap_eq_02} follows from the fact that the true action-value function has zero error with respect to the Bellman operator; Equation \eqref{eq:value_gap_eq_03} follows from expanding the Bellman operator for both value functions; Equation \eqref{eq:value_gap_eq_04} follows from Lemma \ref{lem:reward_bounds} and, finally, Equation \eqref{eq:value_gap_eq_05} follows from applying the same bounds $H$ steps.

    Following similar arguments, for the best response value gap, we have:
    \begin{align*}
        & V^{\dagger,\bs{\pi}_{-i}}_{i,0}(s_0) - \overline{V}^{\dagger,\bs{\pi}_{-i}}_{i,0}(s_0) = \max_{a_{i,0}\in A_i} \expctu{\bs{a}_{-i,0}\sim\bs{\pi}_{-i,0}(\cdot|s_0)}{Q^{\dagger,\bs{\pi}_{-i}}_{i,0}(s_0,\bs{a}_0) - \overline{Q}^{\dagger,\bs{\pi}_{-i}}_{i,0}(s_0,\bs{a}_0)} \\ 
            & \leq \max_{a_{i,0}\in A_i} \expctu{\bs{a}_{-i,0}\sim\bs{\pi}_{-i,0}(\cdot|s_0)}{\mathbb{B}_{i,0}V^{\dagger,\bs{\pi}_{-i}}_{i,1}(s_0,\bs{a}_0) - \overline{\mathbb{B}}_{i,0}\overline{V}^{\dagger,\bs{\pi}_{-i}}_{i,1}(s_0,\bs{a}_0) + E_1(d,m,\delta,\epsilon)}\\
        & \leq C_1\frac{\epsilon\cdot\exp\left(H+ \sqrt{\log(n/2\delta\epsilon)}\right)}{\xi_R} + E_1(d,m,\delta,\epsilon) + \max_{a_{i,0}\in A_i} \expctu{\bs{a}_{-i,0}\sim\bs{\pi}_{-i,0}(\cdot|s_0), s_1\sim P(\cdot|s_0,\bs{a}_0)}{V^{\dagger,\bs{\pi}_{-i}}_{i,1}(s_1) - \overline{V}^{\dagger,\bs{\pi}_{-i}}_{i,1}(s_1)}\\
        & \leq H\left(C_1\frac{\epsilon\cdot\exp\left(H+ \sqrt{\log(n/2\delta\epsilon)}\right)}{\xi_R} + E_1(d,m,\delta,\epsilon)\right)~.
    \end{align*}
\end{proof}
Next, we state a result that provides an upper bound on the Nash gap in terms of the estimated value functions. 
\begin{lemma}\label{lem:nash_gap_first_bound}
    Under the event of Lemma \ref{lem:value_bounds}, we have, for some $C_1>0$,
    \begin{align*}
        \textnormal{Gap}(\widetilde{\bs{\pi}}) \leq \min_{\bs{\pi}} \sum_{i\in[n]} \overline{V}^{\dagger,{\bs{\pi}}_{-i}}_{i,0}(s_0) - \underline{V}^{{\bs{\pi}}}_{i,0}(s_0) + 2nE_1(d,m,\delta,\epsilon)+ 2nC_1\frac{\epsilon\cdot\exp\left(H+ \sqrt{\log(n/2\delta\epsilon)}\right)}{\xi_R}~.
    \end{align*}
\end{lemma}
\begin{proof}
    Note that, by definition of the Nash gap and Lemma \ref{lem:value_bounds}, we have
    \begin{align*}
        \textnormal{Gap}(\widetilde{\bs{\pi}}) & = \sum_{i\in[n]} V^{\dagger,\widetilde{\bs{\pi}}_{-i}}_{i,0}(s_0) - V^{\widetilde{\bs{\pi}}}_{i,0}(s_0) \\
            & \leq \sum_{i\in[n]} \overline{V}^{\dagger,\widetilde{\bs{\pi}}_{-i}}_{i,0}(s_0) - \underline{V}^{\widetilde{\bs{\pi}}}_{i,0}(s_0) + 2nE_1(d,m,\delta,\epsilon) + 2nC_1\frac{\epsilon\cdot\exp\left(H+ \sqrt{\log(n/2\delta\epsilon)}\right)}{\xi_R}\\
        & = \min_{\bs{\pi}} \sum_{i\in[n]} \overline{V}^{\dagger,{\bs{\pi}}_{-i}}_{i,0}(s_0) - \underline{V}^{{\bs{\pi}}}_{i,0}(s_0) + 2nE_1(d,m,\delta,\epsilon) + 2nC_1\frac{\epsilon\cdot\exp\left(H+ \sqrt{\log(n/2\delta\epsilon)}\right)}{\xi_R}~,
    \end{align*}
    where the last step uses the fact that $\widetilde{\bs{\pi}}$ minimizes the quantity within the summation, as defined in Algorithm \ref{alg:pmael}.
\end{proof}

Now we are ready to finalize the proof of the main theorem of Section \ref{sec:uniform_setting}. We restate it for convenience.

\begin{theorem}\label{thm:uniform_appendix}
    Let $\epsilon\in [0,1/2)$,  $\delta >0$ and $\Gamma(\cdot,\cdot)=0$. Furthermore, assume that $ m \geq \Omega( (H^{3/2}/\epsilon^2)(d+\log(n/\delta)))$. Then, under Assumption \ref{asmp:uniform_coverage} with $\xi_R \geq 5\epsilon$, for some positive constant $c$, there exist robust algorithms \texttt{TrimmedMLE} and \texttt{RobEst} such that, with probability at least $1-\delta$, the output $\widetilde{\bs{\pi}}$ of Objective \eqref{eq:uniform_estimated_gap} satisfies
    \begin{align*}
        \textnormal{Gap}(\widetilde{\bs{\pi}}) \leq O\left(Hn\left(\frac{\exp\left(H+\sqrt{\log(n/2\delta\epsilon)}\right)}{\xi_R}+\frac{H\sqrt{d}+\gamma}{\xi_P}\right)\cdot \epsilon + Hn\sqrt{\frac{(H\sqrt{d}+\gamma)^2\textnormal{poly}(d)}{\xi^2_P m}}\right)~.
    \end{align*}
\end{theorem}
\begin{proof}
    Let $\bs{\pi}^*$ be a Nash equilibrium. We have
    \begin{align}
         & \textnormal{Gap}(\widetilde{\bs{\pi}}) \leq \min_{\bs{\pi}}\sum_{i\in[n]} \overline{V}^{\dagger,{\bs{\pi}}_{-i}}_{i,0}(s_0) - \underline{V}^{{\bs{\pi}}}_{i,0}(s_0) + 2nE_1(d,m,\delta,\epsilon) + 2n\epsilon C_1\frac{\exp\left(H+\sqrt{\log(n/2\delta\epsilon)}\right)}{\xi_R} \label{eq:uniform_nash_bound_eq_01}\\
            & \leq \sum_{i\in[n]} \overline{V}^{\dagger,{\bs{\pi}^*}_{-i}}_{i,0}(s_0) - V^{\dagger,{\bs{\pi}^*}_{-i}}_{i,0}(s_0) + \sum_{i\in[n]} V^{{\bs{\pi}^*}}_{i,0}(s_0) - \underline{V}^{{\bs{\pi}^*}}_{i,0}(s_0) + \sum_{i\in[n]} V^{\dagger,{\bs{\pi}^*}_{-i}}_{i,0}(s_0) - V^{{\bs{\pi}^*}}_{i,0}(s_0) \nonumber\\ & \quad\quad\quad + 2nE_1(d,m,\delta,\epsilon)  + 2n\epsilon C_1\frac{\exp\left(H+\sqrt{\log(n/2\delta\epsilon)}\right)}{\xi_R}\label{eq:uniform_nash_bound_eq_02}\\ 
        & \leq 2Hn\cdot\left( 2\epsilon C_1\frac{\exp\left(H+\sqrt{\log(n/2\delta\epsilon)}\right)}{\xi_R} + 2E_1(d,m,\delta,\epsilon)\right) +  \sum_{i\in[n]} V^{\dagger,{\bs{\pi}^*}_{-i}}_{i,0}(s_0) - V^{{\bs{\pi}^*}}_{i,0}(s_0) \label{eq:uniform_nash_bound_eq_03}\\
            & \leq  2Hn\cdot\left( 2\epsilon C_1\frac{\exp\left(H+\sqrt{\log(n/2\delta\epsilon)}\right)}{\xi_R} + 2E_1(d,m,\delta,\epsilon)\right)  \label{eq:uniform_nash_bound_eq_04}~,
    \end{align}
    where Equation \eqref{eq:uniform_nash_bound_eq_01} follows from Lemma \ref{lem:nash_gap_first_bound}; for Equation \eqref{eq:uniform_nash_bound_eq_02}, we pick a Nash equilibrium $\bs{\pi}^*$ and use the fact that $\widetilde{\bs{\pi}}$ is the minimizer of the estimated gap; Equation \eqref{eq:uniform_nash_bound_eq_03} follows from Lemma \ref{lem:initial_value_bounds} and, finally, Equation \eqref{eq:uniform_nash_bound_eq_04} follows from the fact that $\bs{\pi}^*$ is a Nash equilibrium and thus any unilateral deviation yields a smaller value for agent $i$.
\end{proof}

\section{Proof of Theorem \ref{thm:unilateral_coverage_bounds}}\label{appendix:2}

In this section, we provide the full proof for Theorem \ref{thm:unilateral_coverage_bounds}. First, let us define state-action occupancy measures. Given policy $\bs{\pi}$, we define $d^{\bs{\pi}}(s) = (1/H)\sum^{H-1}_{h=0}\mathbb{P}\left(s_h=s|s_0,\bs{\pi}\right)$ and  $d^{\bs{\pi}}(s,a) = (1/H)\sum^{H-1}_{h=0}\mathbb{P}\left(s_h=s,a_h=a|s_0,\bs{\pi}\right)$. We will use the notation $\tau\sim d^{\bs{\pi}}$ to imply that trajectory $\tau$ has been sampled according to $\bs{\pi}$ and the transition kernel of the underlying game. We begin by stating upper bound results that are used to specify the structure of our confidence set for this setting. 

\begin{lemma}\label{lem:log_prob_bound_unilateral}
    Let $\mathbb{P}(o|\tau,\tau',\theta) = 1/(1+\exp(-o\cdot\theta^\top (\phi(\tau)-\phi(\tau'))$. Then, given $\delta >0$, with probability at least $1-\delta$, we have for any agent $i$:
    \begin{align*}
        \frac{2}{m}\sum_{(\tau,\tau',o)\in D}\log\left(\frac{\mathbb{P}(o|\tau,\tau',\widetilde{\theta}_i)}{\mathbb{P}(o|\tau,\tau',\theta^*_i)}\right) \leq 6H\sqrt{d}\epsilon +c\frac{d}{m}\log\left(\frac{Hmn}{\delta}\right)~,
    \end{align*}
    where $\widetilde{\theta}_i$ is the output of Algorithm \ref{alg:regularized_mle}.
\end{lemma}
\begin{proof}
    This is an immediate application of Lemma 4.2 of \citep{mandal2024corruption} by applying the union bound for all agents. 
\end{proof}


Note that the above bounds characterize the confidence set that we used throughout Section \ref{sec:unilateral_setting}. As a robust estimation technique, we again make use of \texttt{RobEst} based on the second part of Theorem \ref{thm:rls_guarantee}, which does not require uniform coverage. Lemma \ref{lem:variance_of_noise} and Theorem \ref{thm:rls_guarantee} give us the following guarantee for the output of the robust estimate $\widetilde{\omega}$, where, without loss of generality, we use the behavior policy $\bs{\mu}$:
\begin{align}\label{eq:non_uniform_rls_guarantee}
    \expctu{\bs{\mu}}{\norm{\phi(s,\bs{a})^\top(\omega^*-\widetilde{\omega})}_2} \leq  c_2(\delta)\cdot \sqrt{\frac{(H\sqrt{d}+\gamma)^2\textnormal{poly}(d)}{m}+(H\sqrt{d}+\gamma)^2\epsilon}~.
\end{align}
Note that, using the above, we can equivalently write 
    \begin{align*}
         \norm{\omega^*-\widetilde{\omega}}^2_{\Sigma_{\bs{\mu}}(h)} \leq  c_2(\delta)\left(\frac{(H\sqrt{d}+\gamma)^2\textnormal{poly}(d)}{m}+(H\sqrt{d}+\gamma)^2\epsilon\right)~,
    \end{align*}
    which implies that 
    \begin{align*}
        \norm{\omega^*-\widetilde{\omega}}^2_{\Sigma_{\bs{\mu}}(h)+(2\epsilon+\lambda) I} \leq c_2(\delta)\left(\frac{(H\sqrt{d}+\gamma)^2\textnormal{poly}(d)}{m}+(H\sqrt{d}+\gamma)^2\epsilon + (2\epsilon+\lambda) H\sqrt{d} \right)~,
    \end{align*}
    since $\norm{\omega^*}\leq H\sqrt{d}$ (Lemma A.1 of \citep{zhang2022corruption}). Let us define
    \begin{align}\label{eq:e3_error_term}
        E(d,m,\delta,\epsilon) := \sqrt{c_2(\delta)\left(\frac{(H\sqrt{d}+\gamma)^2\textnormal{poly}(d)}{m}+(H\sqrt{d}+\gamma)^2\epsilon + (2\epsilon+\lambda) H\sqrt{d} \right)}~.
    \end{align}
This term will be useful in defining our bonus for this section.
Recall that, for each step $h$, we have defined the scaled sample covariance matrix with respect to the corrupted data as follows:
\begin{align}\label{eq:sample_covariance}
    \Lambda_h & = \frac{3}{5}\left(\frac{1}{m}\sum^m_{i=1}\left(\phi({s}_h,{\bs{a}}_h)\phi({s}_h,{\bs{a}}_h)^\top \right) + (\epsilon+\lambda) I\right)~,
\end{align}
while the bonus has been defined as
\begin{align}\label{eq:bonus_term}
    \Gamma(s,\bs{a}) = E(d,m,\delta,\epsilon)\cdot\norm{\phi(s,\bs{a})}_{\Lambda_h^{-1}}~.
\end{align}

In the absence of bounds on the norm of the parameter, we cannot bound the difference in rewards directly, as we did in Lemma \ref{lem:reward_bounds}. Thus, we will need to follow another approach. First, similar to the previous section, we have the following result.

\begin{lemma}\label{lem:bellman_error_unilateral}
    Let $\lambda \geq \Omega(dH\log(m/\delta))$ and $\Gamma$ be defined as in Equation \eqref{eq:bonus_term}. Then, with probability at least $1-\delta/2$ we have, for every $i,h,s,\bs{a}$, and policy $\bs{\pi}$, 
    \begin{align*}
        0 & \leq \underline{\mathbb{B}}_{i,h}\underline{V}^{\bs{\pi}}_{i,h+1}(s,\bs{a}) - \underline{Q}^{\bs{\pi}}_{i,h}(s,\bs{a}) \leq 2\Gamma(s,\bs{a})~, \\
        0 & \geq \overline{\mathbb{B}}_{i,h}\overline{V}^{\dagger,\bs{\pi}_{-i}}_{i,h+1}(s,\bs{a}) - \overline{Q}^{\dagger,\bs{\pi}_{-i}}_{i,h}(s,\bs{a}) \geq -2\Gamma(s,\bs{a})~.
    \end{align*}
\end{lemma}
\begin{proof}
    Following a similar approach as the proof of Lemma \ref{lem:bellman_operator_bounds}, we have
    \begin{align*}
        \left|\phi(s,\bs{a})\top\underline{\omega}^*_{i,h} - \underline{\mathbb{B}}_{i,h}\underline{V}^{\bs{\pi}}_{i,h+1}(s,\bs{a})\right| & = \left|\left\langle \phi(s,\bs{a}),\underline{\omega}^*_{i,h} - \underline{\omega}^{\bs{\pi}}_{i,h}\right\rangle\right| \\ 
            & \leq \norm{\underline{\omega}^*_{i,h} - \underline{\omega}^{\bs{\pi}}_{i,h}}_{\Sigma_{\bs{\mu}}(h)+(2\epsilon +\lambda) I}\norm{\phi(s,\bs{a})}_{(\Sigma_{\bs{\mu}}(h)+(2\epsilon+\lambda) I)^{-1}} \\
        & \leq E(d,m,\delta,\epsilon)\cdot \norm{\phi(s,\bs{a})}_{(\Sigma_{\bs{\mu}}(h)+(2\epsilon+\lambda) I)^{-1}}\\
            & \leq E(d,m,\delta,\epsilon)\cdot\norm{\phi(s,\bs{a})}_{(\Sigma_{\bs{\mu}}(h)+(2\epsilon+\lambda) I)^{-1}}\\
        & \leq E(d,m,\delta,\epsilon)\cdot\norm{\phi(s,\bs{a})}_{\Lambda_h^{-1}}\\
            & = \Gamma(s,\boldsymbol{a})~,
    \end{align*} 
    where the penultimate inequality uses the fact that $\norm{\underline{\omega}^*_{i,h}}_2\leq H\sqrt{d}$  and the final inequality follows from the following observation:
    \begin{align*}
        \Lambda_h & = \frac{3}{5}\left(\frac{1}{m}\sum^m_{i=1}\phi({s}_h,{\bs{a}}_h)\phi({s}_h,{\bs{a}}_h)^\top +(\epsilon+\lambda) I\right) \\           
        & \preceq \frac{3}{5}\left(\frac{1}{m}\sum^m_{i=1}\phi(\widetilde{s}_h,\widetilde{\bs{a}}_h)\phi(\widetilde{s}_h,\widetilde{\bs{a}}_h)^\top +(2\epsilon+\lambda) I\right) \\
            & \preceq\Sigma_{\bs{\mu}}(h) +(2\epsilon+\lambda)I~,
    \end{align*}
    where the second step uses the fact that $\norm{\phi(s,\bs{a})}_2\leq 1$ and that only $\epsilon\cdot m$ samples are corrupted, while the last step uses Lemma \ref{lem:concentration_of_covariances} and the fact that $m(2\epsilon+\lambda) \geq \Omega(d\log(m/\delta))$, due to our choice of $\lambda$ and the fact that $\epsilon \geq 0$.

    Thus, we obtained that
    \begin{align*}
        \underline{\mathbb{B}}_{i,h}\underline{V}^{\bs{\pi}}_{i,h+1}(s,\bs{a})-\Gamma(s,\boldsymbol{a}) \leq \phi(s,\bs{a})\top\underline{\omega}^*_{i,h}  \leq  \underline{\mathbb{B}}_{i,h}\underline{V}^{\bs{\pi}}_{i,h+1}(s,\bs{a})+\Gamma(s,\boldsymbol{a})  ~,
    \end{align*}
    which, by subtracting $\Gamma(s,\boldsymbol{a})$ from all sides, further implies that
    \begin{align*}
        \underline{\mathbb{B}}_{i,h}\underline{V}^{\bs{\pi}}_{i,h+1}(s,\bs{a})-2\Gamma(s,\boldsymbol{a}) \leq \phi(s,\bs{a})\top\underline{\omega}^*_{i,h} -\Gamma(s,\boldsymbol{a})   \leq  \underline{\mathbb{B}}_{i,h}\underline{V}^{\bs{\pi}}_{i,h+1}(s,\bs{a})~.
    \end{align*}
    Now, since $\underline{\mathbb{B}}_{i,h}\underline{V}^{\bs{\pi}}_{i,h+1}(s,\bs{a})\in [-(H-h)\sqrt{d},(H-h)\sqrt{d}]$, and since the clipping operator is monotone, we have
    \begin{align*}
        \underline{\mathbb{B}}_{i,h}\underline{V}^{\bs{\pi}}_{i,h+1}(s,\bs{a})-2\Gamma(s,\boldsymbol{a}) & \leq \textnormal{Clip}_{[-(H-h)\sqrt{d},(H-h)\sqrt{d}]}\left( \underline{\mathbb{B}}_{i,h}\underline{V}^{\bs{\pi}}_{i,h+1}(s,\bs{a})-2\Gamma(s,\boldsymbol{a})\right) \\
            & \leq \textnormal{Clip}_{[-(H-h)\sqrt{d},(H-h)\sqrt{d}]}\left(\phi(s,\bs{a})\top\underline{\omega}^*_{i,h} -\Gamma(s,\boldsymbol{a})\right) \\
        & = \underline{Q}^{\bs{\pi}}_{i,h}(s,\bs{a}) \\
            & \leq \underline{\mathbb{B}}_{i,h}\underline{V}^{\bs{\pi}}_{i,h+1}(s,\bs{a})~.
    \end{align*}
    This finally implies that 
    \begin{align*}
        0 \leq \underline{\mathbb{B}}_{i,h}\underline{V}^{\bs{\pi}}_{i,h+1}(s,\bs{a}) - \underline{Q}^{\bs{\pi}}_{i,h}(s,\bs{a}) \leq 2\Gamma(s,\bs{a})~.
    \end{align*}

    For the optimistic estimates, we argue in a symmetrical fashion, thus, we omit the proof. 
\end{proof}

Next, we will state a result which is the analogue of Lemma \ref{lem:value_bounds} for this section. We define $\underline{V}^{\bs{\pi}}_{i,h}(s,\widehat{\theta}_i)$ and $\overline{V}^{\dagger,\bs{\pi}_{-i}}_{i,h}(s,\widehat{\theta}_i)$ to be the lower and upper estimates of the value functions of given policy $\bs{\pi}$ with respect to parameter $\widehat{\theta}_i$. We use similar notation for the $Q$-function estimates.

\begin{lemma}\label{lem:value_bounds_unilateral}
    Let $\widehat{\theta}_i\in\Theta_\textnormal{Unil}(\widetilde{\theta}_i)$ be a parameter used by the robust subroutine. Then, under the event of Lemma \ref{lem:bellman_error_unilateral}, we have, for every agent $i$, state $s$, step $h$ and policy $\bs{\pi}$:
    \begin{align*}
        \underline{V}^{\bs{\pi}}_{i,h}(s,\widehat{\theta}_i) & \leq V^{\bs{\pi}}_{i,h}(s,\widehat{\theta}_i)~, \;\; \text{and} \;\; 
        \overline{V}^{\dagger,\bs{\pi}_{-i}}_{i,h}(s,\widehat{\theta}_i) \geq V^{\dagger,\bs{\pi}_{-i}}_{i,h}(s,\widehat{\theta}_i)~.
    \end{align*}
\end{lemma}
\begin{proof}

    Similar to Lemma \ref{lem:value_bounds}, we again apply induction. Note that the result holds for step $H$ where all value estimates are $0$, since the bound term is non-negative. Suppose the statement holds for step $h+1$. Then, for step $h$, we have
    \begin{align*}
        \underline{V}^{\bs{\pi}}_{i,h}(s,\widehat{\theta}_i) & = \expctu{\bs{a}\sim\bs{\pi}_h}{\underline{Q}^{\bs{\pi}}_{i,h}(s,\bs{a},\widehat{\theta}_i)}
            \\ & \leq \expctu{\bs{a}\sim\bs{\pi}_h}{\underline{\mathbb{B}}_{i,h}\underline{V}_{i,h+1}^{\bs{\pi}}(s,\bs{a},\widehat{\theta}_i)} \\ & \leq \expctu{\bs{a}\sim\bs{\pi}_h}{\underline{\mathbb{B}}_{i,h}{V}_{i,h+1}^{\bs{\pi}}(s,\bs{a},\widehat{\theta}_i)} \\ &={V}_{i,h}^{\bs{\pi}}(s,\widehat{\theta}_i)~.
    \end{align*}
    For $\overline{V}^{\dagger,\bs{\pi}_{-i}}_{i,h}(s,\widehat{\theta}_i)$ we have
    \begin{align*}
        \overline{V}_{i,h}^{\dagger,\bs{\pi}_{-i}}(s,\widehat{\theta}_i) & = \max_{a_i\in A_i}\expctu{\bs{a}_{-i}\sim\bs{\pi}_{-i,h}(\cdot|s)}{\overline{Q}^{\dagger,\bs{\pi}_{-i}}_{i,h}(s,\bs{a},\widehat{\theta}_i)}\\ 
            & \geq \max_{a_i\in A_i}\expctu{\bs{a}_{-i}\sim\bs{\pi}_{-i,h}(\cdot|s)}{\overline{\mathbb{B}}_{i,h}\overline{V}^{\dagger,\bs{\pi}_{-i}}_{i,h+1}(s,\bs{a},\widehat{\theta}_i)} \\
        & \geq \max_{a_i\in A_i}\expctu{\bs{a}_{-i}\sim\bs{\pi}_{-i,h}(\cdot|s)}{\overline{\mathbb{B}}_{i,h}V^{\dagger,\bs{\pi}_{-i}}_{i,h+1}(s,\bs{a},\widehat{\theta}_i)}\\
            & = V^{\dagger,\bs{\pi}_{-i}}_{i,h}(s,\widehat{\theta}_i)~.
    \end{align*}
\end{proof}

Next, we prove an upper bound on the expected sum of bonuses. 

\begin{lemma}\label{lem:bouns_term_bound}
    Let $\bs{\pi}^*$ be a Nash equilibrium which is covered by $D$. Then, for every agent $i$, we have
    \begin{align*}
        \expctu{\bs{\pi}^*}{\sum^{H-1}_{h=0}\Gamma(s_h,\bs{a}_h)} \leq H\cdot E(d,m,\delta,\epsilon)\cdot\sqrt{\frac{5 d}{C_P}}~.
    \end{align*}
\end{lemma}
\begin{proof}
    Using the definition of the bonus in Equation \eqref{eq:bonus_term}, we have
    \begin{align*}
        \expctu{\bs{\pi}^*}{\sum^{H-1}_{h=0}\Gamma(s_h,\bs{a}_h)} & = E(d,m,\delta,\epsilon)\cdot\expctu{\bs{\pi}^*}{\sum^{H-1}_{h=0}\norm{\phi(s_h,\bs{a}_h)}_{\Lambda_h^{-1}}}~.
    \end{align*}
    We bound the last factor on the right-hand side of the equation above. Using the definition of $\Lambda_h$ in Equation \eqref{eq:sample_covariance}, we have, for every $h\in[H-1]$:
    \begin{align}
        \expctu{\bs{\pi}^*}{\norm{\phi(s_h,\bs{a}_h)}_{\Lambda_h^{-1}}} & \leq \expctu{\bs{\pi}^*}{\norm{\phi(s_h,\bs{a}_h)}_{\left(\left(\Sigma_{\bs{{\mu}}}(h) +\lambda I\right)\right)^{-1}}} \label{eq:bouns_term_bound_eq_01}\\ 
            & \leq \expctu{\bs{\pi}^*}{\sqrt{\phi(s_h,\bs{a}_h)^\top\left(\left(\Sigma_{\bs{{\mu}}}(h) +\lambda I\right)\right)^{-1}\phi(s_h,\bs{a}_h)}}\nonumber \\
        & \leq \sqrt{ \expctu{\bs{\pi}^*}{\phi(s_h,\bs{a}_h)^\top\left(\left(\Sigma_{\bs{{\mu}}}(h) +\lambda I\right)\right)^{-1}\phi(s_h,\bs{a}_h)}} \label{eq:bouns_term_bound_eq_02}\\
            & = \sqrt{Tr\left(\expctu{\bs{\pi}^*}{\phi(s_h,\bs{a}_h)\phi(s_h,\bs{a}_h)^\top} \left(\left(\Sigma_{\bs{{\mu}}}(h) +\lambda I\right)\right)^{-1} \right)}\label{eq:bouns_term_bound_eq_03}\\
        & \leq \sqrt{\frac{1}{C_P}} \sqrt{Tr\left(\expctu{\bs{\mu}_h}{\phi(s_h,\bs{a}_h)\phi(s_h,\bs{a}_h)^\top} \left(\left(\Sigma_{\bs{{\mu}}}(h) +\lambda I\right)\right)^{-1} \right)}\label{eq:bouns_term_bound_eq_04}\\
            & = \sqrt{\frac{1}{C_P}} \sqrt{Tr\left(\Sigma_{\bs{\mu}}(h) \left(\left(\Sigma_{\bs{{\mu}}}(h) +\lambda I\right)\right)^{-1} \right)}\nonumber\\
        & \leq \sqrt{\frac{1}{C_P}}\sqrt{\sum^d_{j=1}\frac{\sigma_j}{\sigma_j +\lambda}}\label{eq:bouns_term_bound_eq_05}\\
            & \leq \sqrt{\frac{d}{C_P}}~,
    \end{align}
    where $Tr(M)$ denotes the trace of matrix $M$ and $\sigma_j$ denote the eigenvalues of covariance matrix $\Sigma_{\bs{\mu}}(h)$. Above, Equation \eqref{eq:bouns_term_bound_eq_01} uses the observation
    \begin{align*}
        \Lambda^{-1}_h & = \left(\frac{3}{5}\left(\frac{1}{m}\sum^m_{i=1}\phi({s}_h,{\bs{a}}_h)\phi({s}_h,{\bs{a}}_h)^\top +(\epsilon+\lambda) I\right)\right)^{-1} \\           
        & \preceq \left(\frac{3}{5}\left(\frac{1}{m}\sum^m_{i=1}\phi(\widetilde{s}_h,\widetilde{\bs{a}}_h)\phi(\widetilde{s}_h,\widetilde{\bs{a}}_h)^\top + \lambda I\right)\right)^{-1} \\
            & \preceq \left(\left(\Sigma_{\bs{\mu}}(h) +\lambda I\right)\right)^{-1}~,
    \end{align*}
    which follows from Lemma \ref{lem:concentration_of_covariances} and similar arguments as in the proof of Lemma \ref{lem:bellman_error_unilateral}; Equation \eqref{eq:bouns_term_bound_eq_02} uses Jensen's inequality; Equation \eqref{eq:bouns_term_bound_eq_03} uses the commutativity of the trace operator: $Tr(x^\top Mx) = Tr(Mxx^\top)$; Equation \eqref{eq:bouns_term_bound_eq_04} uses the transition coverage part of Assumption \ref{asmp:low_relative_uncertainty}; finally, Equation \eqref{eq:bouns_term_bound_eq_05} uses the fact that the eigenvalues of $\Sigma_{\bs{\mu}}(h)$ and $\lambda$ are nonnegative real numbers. 
\end{proof}

Next, we will provide an upper bound on the difference in preference functions between the ground-truth and estimated parameters. 

\begin{lemma}\label{lem:difference_in_probs_bound}
    For any agent $i$ and $\theta_i\in\Theta_\textnormal{Unil}(\widetilde{\theta}_i)$, with probability at least $1-\delta/2$, we have
    \begin{align*}
        \expctu{\substack{\tau\sim d^{\bs{\mu}}\\ \tau'\sim d^{\bs{\mu}_\textnormal{ref}}}}{\norm{\mathbb{P}\left(\cdot|\tau,\tau',\theta^*_i\right) - \mathbb{P}\left(\cdot|\tau,\tau',\theta_i\right)}^2_1} & \leq 8H\sqrt{d}\epsilon + c\cdot \frac{d}{m}\log\left(\frac{2Hnm\sqrt{d}}{\delta}\right) ~,
    \end{align*}
    where $c$ is an absolute constant.
\end{lemma}
\begin{proof}
    Lemma \ref{lem:bound_on_log_prob_difference} gives us the following bound:
    \begin{align*}
        \expctu{\substack{\tau\sim d^{\bs{\mu}}\\ \tau'\sim d^{\bs{\mu}_\textnormal{ref}}}}{\norm{\mathbb{P}\left(\cdot|\tau,\tau',\theta^*_i\right) - \mathbb{P}\left(\cdot|\tau,\tau',\theta\right)}^2_1} \leq \frac{c_1}{m}\sum^m_{j=1}\log\left(\frac{\mathbb{P}\left(\widetilde{o}_j|\widetilde{\tau}_j,\widetilde{\tau}'_j,\theta\right)}{\mathbb{P}\left(\widetilde{o}_j|\widetilde{\tau}_j,\widetilde{\tau}'_j,\theta^*_i\right)}\right) +\log\left( d\log \left(2n/\delta\right)\right)~.
    \end{align*}
    Let us deal with the first term of the bound below. Note that the bound depends on clean samples. Let $\widehat{D}$ be the given dataset and $S$ denote the set of corrupted trajectories in $\widehat{D}$. We can write
    \begin{align*}
        \sum^m_{j=1}\log\left(\frac{\mathbb{P}\left(\widetilde{o}_j|\widetilde{\tau}_j,\widetilde{\tau}'_j,\theta\right)}{\mathbb{P}\left(\widetilde{o}_j|\widetilde{\tau}_j,\widetilde{\tau}'_j,\theta^*_i\right)}\right) & = \sum_{(\tau,\tau',o)\in S}\log\left(\frac{\mathbb{P}\left(o_j|\tau_j,\tau'_j,\theta\right)}{\mathbb{P}\left(o_j|\tau_j,\tau'_j,\theta^*_i\right)}\right)  +\sum_{(\tau,\tau',o)\not\in S}\log\left(\frac{\mathbb{P}\left(o_j|\tau_j,\tau'_j,\theta\right)}{\mathbb{P}\left(o_j|\tau_j,\tau'_j,\theta^*_i\right)}\right) \\
            & \leq \sum_{(\tau,\tau',o)\in \widehat{D}} \log\left(\frac{\mathbb{P}\left(o_j|\tau_j,\tau'_j,\theta\right)}{\mathbb{P}\left(o_j|\tau_j,\tau'_j,\theta^*_i\right)}\right) + \epsilon\cdot\log\left(\frac{1+\exp\left(H\sqrt{d}\right)}{1+\exp\left(-H\sqrt{d}\right)}\right)\\
        & \leq \sum_{(\tau,\tau',o)\in \widehat{D}} \log\left(\frac{\mathbb{P}\left(o_j|\tau_j,\tau'_j,\widetilde{\theta}_i\right)}{\mathbb{P}\left(o_j|\tau_j,\tau'_j,\theta^*_i\right)}\right) + 2\epsilon\cdot H\sqrt{d} \\
            & \leq 8H\sqrt{d}\epsilon + c_2\cdot\frac{d}{m}\log\left(\frac{2Hmn}{\delta}\right)~,
    \end{align*}
    where the first inequality uses the fact that the corrupted subset comprises an $\epsilon$-fraction of the whole dataset and the fact that $\phi^\top\theta\leq H\sqrt{d}$, by assumption of linear MDPs; the second inequality uses the fact that $\widetilde{\theta}_i$ maximizes the log-likelihood with respect to the corrupted data, and that $H$ and $d$ are natural numbers so we can bound the $\log$ expression directly in terms of the bounds on the rewards; finally, for the last inequality we have used Lemma \ref{lem:log_prob_bound_unilateral} with some constant $c_2$. Putting things together, we obtain the stated bound. 
\end{proof}

Next, we prove bounds on the gaps with respect to any chosen reward parameters from the confidence set.

\begin{lemma}\label{lem:main_gap_bounds_unilaterl}
    Let $\widehat{\bs{\theta}}=(\widehat{\theta}_1,\ldots,\widehat{\theta}_n)$, where $\widehat{\theta}_i\in\Theta_\textnormal{Unil}(\widetilde{\theta}_i)$, for every $i\in[n]$, and let $\bs{\pi}^*$ be a Nash equilibrium covered by $D$. Then, if 
    \begin{align*}
        \widetilde{\bs{\pi}} \in \arg\min_{\bs{\pi}}\widetilde{\textnormal{Gap}}\left(\bs{\pi},\widehat{\bs{\theta}}\right)~,
    \end{align*}
    we have, with probability at least $1-\delta$:
    \begin{align*}
        0 & \leq \textnormal{Gap}\left(\widetilde{\bs{\pi}},\widehat{\bs{\theta}}\right) \leq \widetilde{O}\left( n\cdot\left(\frac{1}{\sqrt{C_R}}+\frac{1}{\sqrt{C_P}}\right)\cdot\left(H^{5/2}d^{3/4}\sqrt{\epsilon} + H^2\textnormal{poly}(d)\frac{1}{\sqrt{m}}\right) \right)~,\\
        0 & \leq \textnormal{Gap}\left(\bs{\pi}^*,\widehat{\bs{\theta}}\right) \leq \widetilde{O}\left( n\cdot\left(\frac{1}{\sqrt{C_R}}+\frac{1}{\sqrt{C_P}}\right)\cdot\left(H^{5/2}d^{3/4}\sqrt{\epsilon} + H^2\textnormal{poly}(d)\frac{1}{\sqrt{m}}\right) \right)~,
    \end{align*}
\end{lemma}
\begin{proof}
    First, note that both the true gap and estimated gap are, by definition, non-negative. Now, given $\widehat{\bs{\theta}}$ and $\widetilde{\bs{\pi}}$, as specified in the statement, we have
    \begin{align}
        & \textnormal{Gap}\left(\widetilde{\bs{\pi}},\widehat{\bs{\theta}}\right) = \sum_{i\in[n]}V^{\dagger,\widetilde{\bs{\pi}}_{-i}}_{i,0}\left(s_0,\widehat{\theta}_i\right) - V^{\widetilde{\bs{\pi}}}_{i,0}\left(s_0,\widehat{\theta}_i\right)\label{eq:every_theta_works_01}\\
            & \leq \sum_{i\in[n]}\overline{V}^{\dagger,\widetilde{\bs{\pi}}_{-i}}_{i,0}\left(s_0,\widehat{\theta}_i\right)   - \underline{V}^{\widetilde{\bs{\pi}}}_{i,0}\left(s_0,\widehat{\theta}_i\right)\label{eq:every_theta_works_02}\\
            & \leq \min_{\bs{\pi}} \sum_{i\in[n]}\overline{V}^{\dagger,{\bs{\pi}}_{-i}}_{i,0}\left(s_0,\widehat{\theta}_i\right) - \underline{V}^{{\bs{\pi}}}_{i,0}\left(s_0,\widehat{\theta}_i\right) \label{eq:every_theta_works_04}\\
        & \leq \sum_{i\in[n]}\overline{V}^{\dagger,\bs{\pi}^*_{-i}}_{i,0}\left(s_0,\widehat{\theta}_i\right) - V^{\bs{\mu}_\textnormal{ref}}_{i,0}\left(s_0,\widehat{\theta}_i\right)  - \left(\underline{V}^{\bs{\pi}^*}_{i,0}\left(s_0,\widehat{\theta}_i\right) - V^{\bs{\mu}_\textnormal{ref}}_{i,0}\left(s_0,\widehat{\theta}_i\right)\right)\label{eq:every_theta_works_05}\\
            & = \sum_{i\in[n]}\overline{V}^{\dagger,\bs{\pi}^*_{-i}}_{i,0}\left(s_0,\widehat{\theta}_i\right) + V^{\bs{\mu}_\textnormal{ref}}_{i,0}\left(s_0,\theta^*\right) - V^{\bs{\mu}_\textnormal{ref}}_{i,0}\left(s_0,\widehat{\theta}_i\right)  \nonumber \\ & \hspace{2cm}- \left(\underline{V}^{\bs{\pi}^*}_{i,0}\left(s_0,\widehat{\theta}_i\right) + V^{\bs{\mu}_\textnormal{ref}}_{i,0}\left(s_0,\theta^*\right) - V^{\bs{\mu}_\textnormal{ref}}_{i,0}\left(s_0,\widehat{\theta}_i\right)\right) \label{eq:every_theta_works_06}\\
        & = \sum_{i\in[n]} \overline{V}^{\dagger,\bs{\pi}^*_{-i}}_{i,0}\left(s_0,\widehat{\theta}_i\right) + V^{\bs{\mu}_\textnormal{ref}}_{i,0}\left(s_0,\theta^*\right) - V^{\bs{\mu}_\textnormal{ref}}_{i,0}\left(s_0,\widehat{\theta}_i\right)  - V^{\dagger,\bs{\pi}^*_{-i}}_{i,0}\left(s_0,\theta^*\right)  \nonumber\\
            &  \hspace{2cm} + V^{\bs{\pi}^*}_{i,0}\left(s_0,\theta^*\right)  - \left(\underline{V}^{\bs{\pi}^*}_{i,0}\left(s_0,\widehat{\theta}_i\right) + V^{\bs{\mu}_\textnormal{ref}}_{i,0}\left(s_0,\theta^*\right) - V^{\bs{\mu}_\textnormal{ref}}_{i,0}\left(s_0,\widehat{\theta}_i\right)\right)\nonumber \\ 
        & \hspace{2cm}+\underbrace{\left( V^{\dagger,\bs{\pi}^*_{-i}}_{i,0}\left(s_0,\theta^*\right) - V^{\bs{\pi}^*}_{i,0}\left(s_0,\theta^*\right)\right)}_{=0}\label{eq:every_theta_works_07}\\
        & = \sum_{i\in[n]} \underbrace{\overline{V}^{\dagger,\bs{\pi}^*_{-i}}_{i,0}\left(s_0,\widehat{\theta}_i\right) - V^{\bs{\mu}_\textnormal{ref}}_{i,0}\left(s_0,\widehat{\theta}_i\right) + V^{\bs{\mu}_\textnormal{ref}}_{i,0}\left(s_0,\theta^*\right)  - V^{\dagger,\bs{\pi}^*_{-i}}_{i,0}\left(s_0,\theta^*\right)}_{:=Z_1} \nonumber\\
            &  \hspace{2cm} + \underbrace{V^{\bs{\pi}^*}_{i,0}\left(s_0,\theta^*\right) -V^{\bs{\mu}_\textnormal{ref}}_{i,0}\left(s_0,\theta^*\right)   - \left(\underline{V}^{\bs{\pi}^*}_{i,0}\left(s_0,\widehat{\theta}_i\right) - V^{\bs{\mu}_\textnormal{ref}}_{i,0}\left(s_0,\widehat{\theta}_i\right)\right)}_{:=Z_2}\label{eq:every_theta_works_08}
    \end{align}
    where Equation \eqref{eq:every_theta_works_01} follows by definition; Equation \eqref{eq:every_theta_works_02} follows from  Lemma \ref{lem:value_bounds_unilateral}; Equation \eqref{eq:every_theta_works_04} follows by design of Algorithm \ref{alg:pmael_unilateral}; in Equation \eqref{eq:every_theta_works_05} we just substitute $\bs{\pi}^*$ and add and subtract the same term; in Equation \eqref{eq:every_theta_works_06} and Equation \eqref{eq:every_theta_works_07} we again add and subtract identical terms; Equation \eqref{eq:every_theta_works_08} follows from the fact that $\bs{\pi}^*$ is a NE under $\theta^*$.

    Now, we will deal with the two terms above, $Z_1$ and $Z_2$, separately. 
    First, let us consider the term $Z_2$. For every $i$, we have
    \begin{align}
       & V^{{\bs{\pi}}^*}_{i,0}\left(s_0,\theta^*_i\right) - {V}^{\bs{\mu}_\textnormal{ref}}_{i,0}\left(s_0,\theta^*_i\right) - \left(\underline{V}^{{\bs{\pi}}^*}_{i,0}\left(s_0,\widehat{\theta}_i\right) - {V}^{\bs{\mu}_\textnormal{ref}}_{i,0}\left(s_0,\widehat{\theta}_i\right)\right) \nonumber\\ & = \expctu{\bs{a}_0\sim\bs{\pi}^*_0(\cdot|s_0)}{Q^{\bs{\pi}^*}_{i,0}(s_0,\bs{a}_0) - \underline{Q}^{\bs{\pi}^*}_{i,0}(s_0,\bs{a}_0,\widehat{\theta}_i)} - {V}^{\bs{\mu}_\textnormal{ref}}_{i,0}\left(s_0,\theta^*_i\right) +  {V}^{\bs{\mu}_\textnormal{ref}}_{i,0}\left(s_0,\widehat{\theta}_i\right)\nonumber\\
       & \leq \expctu{\bs{a}_0\sim\bs{\pi}^*_0(\cdot|s_0)}{\mathbb{B}_{i,0}V^{\bs{\pi}^*}_{i,1}(s_0,\bs{a}_0) - \underline{\mathbb{B}}_{i,0}\underline{V}^{\bs{\pi}^*}_{i,0}(s_0,\bs{a}_0,\widehat{\theta}_i) + 2\Gamma(s_0,\bs{a}_0)} - {V}^{\bs{\mu}_\textnormal{ref}}_{i,0}\left(s_0,\theta^*_i\right) +  {V}^{\bs{\mu}_\textnormal{ref}}_{i,0}\left(s_0,\widehat{\theta}_i\right)\label{eq:T_2_bound_eq_001}\\
            & = \expctu{\bs{a}_0\sim\bs{\pi}^*_0(\cdot|s_0)}{R_{i,0}(s_0,\bs{a}_0) - \underline{R}_{i,0}(s_0,\bs{a}_0) + \expctu{s_1\sim P_1(\cdot|s_0,\bs{a}_0)}{V^{\bs{\pi}^*}_{i,1}(s_1) - \underline{V}^{\bs{\pi}^*}_{i,1}(s_1,\widehat{\theta}_i)} + 2\Gamma(s_0,\bs{a}_0)}\nonumber \\ & \quad\quad\quad - {V}^{\bs{\mu}_\textnormal{ref}}_{i,0}\left(s_0,\theta^*_i\right) +  {V}^{\bs{\mu}_\textnormal{ref}}_{i,0}\left(s_0,\widehat{\theta}_i\right)\label{eq:T_2_bound_eq_02} \\
        & \leq \expctu{\bs{a}_0\sim\bs{\pi}^*_0(\cdot|s_0),s_1\sim P_1(\cdot|s_0,\bs{a}_0),\bs{a}_1\sim\bs{\pi}^*_1(\cdot|s_1)}{\sum^1_{h=0}\left(R_{i,h}(s_h,\bs{a}_h)-\underline{R}_{i,h}(s_h,\bs{a}_h) + 2\Gamma(s_h,\bs{a}_h)\right)} \nonumber \\ & \quad\quad\quad  +\expctu{\bs{a}_0\sim\bs{\pi}^*_0(\cdot|s_0),s_1\sim P_1(\cdot|s_0,\bs{a}_0),\bs{a}_1\sim\bs{\pi}^*_1(\cdot|s_1),s_2\sim P_2(\cdot|s_1,\bs{a}_1)}{V^{\bs{\pi}^*}_{i,2}(s_2) - \underline{V}^{\bs{\pi}^*}_{i,2}(s_2,\widehat{\theta}_i)}\nonumber\\
        & \quad\quad\quad - {V}^{\bs{\mu}_\textnormal{ref}}_{i,0}\left(s_0,\theta^*_i\right) +  {V}^{\bs{\mu}_\textnormal{ref}}_{i,0}\left(s_0,\widehat{\theta}_i\right)\nonumber \\
            & \leq \ldots \ldots \nonumber \\
        & \leq \expctu{\bs{\pi}^*}{\sum^{H-1}_{h=0}\left(R_{i,h}(s_h,\bs{a}_h)-\underline{R}_{i,h}(s_h,\bs{a}_h)+2\Gamma(s_h,\bs{a}_h)\right)} - {V}^{\bs{\mu}_\textnormal{ref}}_{i,0}\left(s_0,\theta^*_i\right) +  {V}^{\bs{\mu}_\textnormal{ref}}_{i,0}\left(s_0,\widehat{\theta}_i\right)\nonumber\\
            & = \expctu{\tau\sim d^{\bs{\pi}^*}}{\phi(\tau)^\top\theta^*_i - \phi(\tau)^\top\widehat{\theta}_i} - \expctu{\tau\sim d^{\bs{\mu}_\textnormal{ref}}}{\phi(\tau)^\top\theta^*_i - \phi(\tau)^\top\widehat{\theta}_i} + 2 \expctu{\bs{\pi}*}{\sum^{H-1}_{h=0}\Gamma(s_h,\bs{a}_h)} \label{eq:T_2_bound_eq_03}\\
        & = \expctu{\tau\sim d^{\bs{\pi}^*}, \tau'\sim d^{\bs{\mu}_\textnormal{ref}}}{\left(\phi(\tau)-\phi(\tau')\right)^\top\left(\theta^*_i-\widehat{\theta}_i\right)} + 2 \expctu{\bs{\pi}*}{\sum^{H-1}_{h=0}\Gamma(s_h,\bs{a}_h)}\nonumber\\
            & \leq \sqrt{ \expctu{\tau\sim d^{\bs{\pi}^*}, \tau'\sim d^{\bs{\mu}_\textnormal{ref}}}{\left(\left(\phi(\tau)-\phi(\tau')\right)^\top\left(\theta^*_i-\widehat{\theta}_i\right)\right)^2}}+ 2 \expctu{\bs{\pi}*}{\sum^{H-1}_{h=0}\Gamma(s_h,\bs{a}_h)}\label{eq:T_2_bound_eq_06}\\
        & = \sqrt{\left(\theta^*_i-\widehat{\theta}_i\right)^\top \expctu{\tau\sim d^{\bs{\pi}^*}, \tau'\sim d^{\bs{\mu}_\textnormal{ref}}}{\left(\phi(\tau)-\phi(\tau')\right)\left(\phi(\tau)-\phi(\tau')\right)^\top}\left(\theta^*_i-\widehat{\theta}_i\right)} \nonumber\\ & \quad\quad\quad+ 2 \expctu{\bs{\pi}*}{\sum^{H-1}_{h=0}\Gamma(s_h,\bs{a}_h)}\nonumber\\
            & = \sqrt{\left(\theta^*_i-\widehat{\theta}_i\right)^\top \Sigma^-_{\bs{\pi}^*,\bs{\mu}_\textnormal{ref}}\left(\theta^*_i-\widehat{\theta}_i\right)} + 2 \expctu{\bs{\pi}*}{\sum^{H-1}_{h=0}\Gamma(s_h,\bs{a}_h)}\label{eq:T_2_bound_eq_07}\\
        & \leq \sqrt{\frac{1}{C_R}}\sqrt{\left(\theta^*_i-\widehat{\theta}_i\right)^\top \Sigma^-_{\bs{\mu},\bs{\mu}_\textnormal{ref}}\left(\theta^*_i-\widehat{\theta}_i\right)}+ 2 \expctu{\bs{\pi}*}{\sum^{H-1}_{h=0}\Gamma(s_h,\bs{a}_h)}\label{eq:T_2_bound_eq_08}\\
            & = \sqrt{\frac{1}{C_R}} \sqrt{\expctu{\tau\sim d^{\bs{\mu}}, \tau'\sim d^{\bs{\mu}_\textnormal{ref}}}{\left(\left(\phi(\tau)-\phi(\tau')\right)^\top\theta^*_i-\left(\phi(\tau)-\phi(\tau')\right)^\top\widehat{\theta}_i\right)^2}} + 2 \expctu{\bs{\pi}*}{\sum^{H-1}_{h=0}\Gamma(s_h,\bs{a}_h)}\nonumber\\
        & = \sqrt{\frac{1}{C_R}}\sqrt{\expctu{\tau\sim d^{\bs{\mu}}, \tau'\sim d^{\bs{\mu}_\textnormal{ref}}}{\left|\sigma^{-1}\left(\mathbb{P}\left(o=1|\tau,\tau',\theta^*_i\right)\right) - \sigma^{-1}\left(\mathbb{P}\left(o=1|\tau,\tau',\widehat{\theta}_i\right)\right)\right|^2}} \nonumber \\ & \quad\quad\quad + 2 \expctu{\bs{\pi}*}{\sum^{H-1}_{h=0}\Gamma(s_h,\bs{a}_h)}\label{eq:T_2_bound_eq_09}\\
            & \leq \sqrt{\frac{\iota^2}{C_R}}\sqrt{\expctu{\tau\sim d^{\bs{\mu}}, \tau'\sim d^{\bs{\mu}_\textnormal{ref}}}{\left|\mathbb{P}\left(o=1|\tau,\tau',\theta^*_i\right) - \mathbb{P}\left(o=1|\tau,\tau',\widehat{\theta}_i\right)\right|^2}} + 2 \expctu{\bs{\pi}*}{\sum^{H-1}_{h=0}\Gamma(s_h,\bs{a}_h)}\label{eq:T_2_bound_eq_10}\\
        & \leq \sqrt{\frac{\iota^2}{2C_R}}\sqrt{\expctu{\tau\sim d^{\bs{\mu}}, \tau'\sim d^{\bs{\mu}_\textnormal{ref}}}{\norm{\mathbb{P}\left(\cdot|\tau,\tau',\theta^*_i\right) - \mathbb{P}\left(\cdot|\tau,\tau',\widehat{\theta}_i\right)}^2_1}} + 2H\cdot E(d,m,\delta,\epsilon)\cdot \sqrt{\frac{5 d}{C_P}}\label{eq:T_2_bound_eq_11}\\
            & \leq \sqrt{\frac{\iota^2}{2C_R}}\sqrt{8\epsilon + c\cdot \frac{d}{m}\log\left(\frac{nm}{\delta}\right)}+ 2H\cdot E(d,m,\delta,\epsilon)\cdot \sqrt{\frac{5 d}{C_P}}\label{eq:T_2_bound_eq_12}\\
    \end{align}
    where Equation \eqref{eq:T_2_bound_eq_001} follows by definition of Q-functions and the Bellman operator, and Lemma \ref{lem:bellman_error_unilateral}; Equation \eqref{eq:T_2_bound_eq_02} follows from the definition of the Bellman operator; Equation \eqref{eq:T_2_bound_eq_03} uses the trajectory-based definition of the return; Equation \eqref{eq:T_2_bound_eq_06} uses Jensen's inequality; Equation \eqref{eq:T_2_bound_eq_07} uses the definition of the difference covariance matrix with respect to $\bs{\pi}^*$ and $\bs{\mu}_\textnormal{ref}$; Equation \eqref{eq:T_2_bound_eq_08} uses the first part of Assumption \ref{asmp:low_relative_uncertainty}; Equation \eqref{eq:T_2_bound_eq_09} uses the definition of the link function and the preference data generation assumption; Equation \eqref{eq:T_2_bound_eq_10} uses uses Lemma \ref{lem:lipschitz_sigmoid_inverse}; finally, for Equation \eqref{eq:T_2_bound_eq_11} we have used Lemma \ref{lem:bouns_term_bound}.

    Now, denote $\pi^\dagger_i\in\arg\max_{\pi'}V^{\pi',\bs{\pi}^*_{-i}}_{i,0}(s_0,\widehat{\theta}_i)$. For the final term, $Z_1$, we similarly have
    \begin{align*}
        \overline{V}^{\dagger,{\bs{\pi}}^*_{-i}}_{i,0} & \left(s_0,\widehat{\theta}_i\right) - {V}^{\bs{\mu}_\textnormal{ref}}_{i,0}\left(s_0,\widehat{\theta}_i\right) -  \left(V^{\dagger,{\bs{\pi}}^*_{-i}}_{i,0}\left(s_0,\theta^*_i\right) - {V}^{\bs{\mu}_\textnormal{ref}}_{i,0}\left(s_0,\theta^*_i\right)\right) \nonumber\\ 
        & \leq  \expctu{\tau\sim d^{\pi^\dagger_i},\tau'\sim d^{\bs{\pi}^*_{-i}}}{\left(\phi(\tau)-\phi(\tau')\right)^\top\left(\widehat{\theta}_i-\theta^*_i\right) } +  2\expctu{\pi^\dagger_i,\bs{\pi}^*_{-i}}{\sum^{H-1}_{h=0}\Gamma(s_h,\bs{a}_h)}\\
            & \leq \sqrt{\frac{1}{C_R}}\sqrt{\left(\widehat{\theta}_i-\theta^*_i\right)^\top\Sigma^-_{\bs{\mu},\bs{\mu}_\textnormal{ref}}\left(\widehat{\theta}_i-\theta^*_i\right)} + 2H\cdot E(d,m,\delta,\epsilon)\cdot \sqrt{\frac{5 d}{C_P}}\\
        & \leq \sqrt{\frac{\iota^2}{2C_R}}\sqrt{\expctu{\tau\sim d^{\bs{\mu}}, \tau'\sim d^{\bs{\mu}_\textnormal{ref}}}{\norm{\mathbb{P}\left(\cdot|\tau,\tau',\theta^*_i\right) - \mathbb{P}\left(\cdot|\tau,\tau',\underline{\theta}_i\right)}^2_1}} +2H\cdot E(d,m,\delta,\epsilon)\cdot \sqrt{\frac{5 d}{C_P}}\\
            & \leq \sqrt{\frac{\iota^2}{2C_R}}\sqrt{8\epsilon  + c\cdot \frac{d}{m}\log\left(\frac{nm}{\delta}\right)}+2H\cdot E(d,m,\delta,\epsilon)\cdot \sqrt{\frac{5 d}{C_P}}~,
    \end{align*}
    where we have used similar arguments as above (note that this is the part where low relative uncertainty is needed, as opposed to coverage of only a Nash equilibrium).

    Putting everything together, and using the definition of $E(d,m,\delta,\epsilon)$ from Equation \eqref{eq:non_uniform_rls_guarantee} we obtain
    \begin{align*}
         \textnormal{Gap}(\widetilde{\bs{\pi}},\widehat{\theta}_1,\ldots,\widehat{\theta}_n) &  \leq 2n\cdot \sqrt{\frac{\iota^2}{2C_R}}\sqrt{8\epsilon + c\cdot \frac{d}{m}\log\left(\frac{nm}{\delta}\right)} \\ &  + 4nH\cdot \sqrt{c_2(\delta)\left(\frac{(H\sqrt{d}+\gamma)^2\textnormal{poly}(d)}{m}+(H\sqrt{d}+\gamma)^2\epsilon + (2\epsilon+\lambda) H\sqrt{d} \right)}\cdot \sqrt{\frac{5 d}{C_P}}\\
            & \leq \widetilde{O}\left( n\cdot\left(\frac{1}{\sqrt{C_R}}+\frac{1}{\sqrt{C_P}}\right)\cdot\left(H^{5/2}d^{3/4}\sqrt{\epsilon} + H^2\textnormal{poly}(d)\frac{1}{\sqrt{m}}\right)\right)~,
    \end{align*}
    where we have also used our choice of $\lambda$. Finally, for the third statement, note that
    \begin{align*}  \textnormal{Gap}\left(\bs{\pi}^*,\widehat{\theta}_1,\ldots,\widehat{\theta}_n\right) & = \sum_{i\in[n]} V^{\dagger,\bs{\pi}^*_{-i}}_{i,0}\left(s_0,\widehat{\theta}_i\right) - V^{\bs{\pi}^*}_{i,0}\left(s_0,\widehat{\theta}_i\right) \\
        & \leq \sum_{i\in[n]} \overline{V}^{\dagger,\bs{\pi}^*_{-i}}_{i,0}\left(s_0,\widehat{\theta}_i\right) - V^{\bs{\mu}_\textnormal{ref}}_{i,0}\left(s_0,\widehat{\theta}_i\right) - \left(\underline{V}^{\bs{\pi}^*}_{i,0}\left(s_0,\widehat{\theta}_i\right) -V^{\bs{\mu}_\textnormal{ref}}_{i,0}\left(s_0,\widehat{\theta}_i\right) \right)\\
    & \leq \widetilde{O}\left( n\cdot\left(\frac{1}{\sqrt{C_R}}+\frac{1}{\sqrt{C_P}}\right)\cdot\left(H^{5/2}d^{3/4}\sqrt{\epsilon} + H^2\textnormal{poly}(d)\frac{1}{\sqrt{m}}\right) \right)~,
    \end{align*}
    where the first inequality uses Lemma \ref{lem:value_bounds_unilateral} as if the robust subroutine were applied on $\bs{\pi}^*$, while the second inequality follows from noting that we already have a bound on the previous quantity from Equation \eqref{eq:every_theta_works_05}.
\end{proof}

Before we proceed, we need to provide guarantees on the output of the PGA methods used in Algorithm \ref{alg:pmael_unilateral}. 

\begin{proposition}\label{pro:pgd_guarantees}
Let $\eta_1 = 1/\sqrt{T_1}$ and, for every agent $i$, let
\[
\widehat{\theta}_i^\ast \in \arg\max_{\theta_i \in \Theta_\textnormal{Unil}(\widetilde{\theta}_i)}
\textnormal{Gap}_i(\bs{\pi}^\ast,\theta_i),
\]
where
\[
\widehat{\theta}_i := \frac{1}{T_1}\sum_{t=1}^{T_1}\widehat{\theta}_i^{(t)},
\]
and the iterates $\widehat{\theta}_i^{(t)}$ are generated by
\[
\widehat{\theta}_i^{(t+1)}
=
\mathcal P_{\Theta_\textnormal{Unil}(\widetilde{\theta}_i)}
\Bigl(
\widehat{\theta}_i^{(t)} + \eta_1 \,\tilde{\nabla}_{\theta}\textnormal{Gap}_i(\bs{\pi}^\ast,\widehat{\theta}_i^{(t)})
\Bigr)
\]
for $T_1$ steps. Then,
\[
\textnormal{Gap}_i(\bs{\pi}^\ast,\widehat{\theta}_i^\ast)
-
\textnormal{Gap}_i(\bs{\pi}^\ast,\widehat{\theta}_i)
\le
\widetilde{\mathcal O}
\left(
\left(
\frac{1}{\sqrt{C_R}}+\frac{1}{\sqrt{C_P}}
\right)
\left(
H^{5/2} d^{3/4}\sqrt{\epsilon}
+
H^2 \frac{\textnormal{poly}(d)}{\sqrt{m}}
\right)
+
\frac{H^2\sqrt{\textnormal{poly}(d)}}{\sqrt{T_1}}
\left(
\sqrt{\epsilon}+\frac{1}{\sqrt{m}}
\right)
\right).
\]
\end{proposition}

\begin{proof}
    First, note that, for any agent $i$ and parameter $\theta_i$, we have
    \begin{align*}
        \textnormal{Gap}_i\left(\bs{\pi}^*,\theta_i\right)
        &= V^{\dagger,\bs{\pi}^*_{-i}}_{i,0}(s_0,\theta_i) - V^{\bs{\pi}^*}_{i,0}(s_0,\theta_i) \\
        &= \max_{\pi'_i} V^{\pi'_i,\bs{\pi}^*_{-i}}_{i,0}(s_0,\theta_i) - V^{\bs{\pi}^*}_{i,0}(s_0,\theta_i)~.
    \end{align*}
    Now, let us define
    \begin{align*}
        \Pi^{\textnormal{PP},\dagger}_i(\theta_i) = \left\{ \pi^\dagger_i\in\Pi^\textnormal{PP}_i : V^{\pi^\dagger_i,\bs{\pi}^*_{-i}}_{i,0}(s_0,\theta_i) = \max_{\pi'_i}V^{\pi'_i,\bs{\pi}^*_{-i}}_{i,0}(s_0,\theta_i)\right\}~,
    \end{align*}
    as the set of unilateral maximizer policies for player $i$ at $\theta_i$. Let $\pi_i^\dagger \in \Pi^{\textnormal{PP},\dagger}_i(\widehat{\theta}^*_i)$ be any unilateral maximizer at $\widehat{\theta}^*_i$, and define
    \begin{align*}
        g_i^* := \nabla_{\theta_i}\left(V^{\pi^\dagger_i,\bs{\pi}^*_{-i}}_{i,0}(s_0,\theta_i) - V^{\bs{\pi}^*}_{i,0}(s_0,\theta_i)\right)~.
    \end{align*}
    Since the value function is linear in $\theta_i$, $g_i^*$ does not depend on $\theta_i$. Moreover, Danskin's subdifferential formula implies that
    \begin{align*}
        g_i^* \in \partial_{\theta_i}\textnormal{Gap}_i\left(\bs{\pi}^*,\widehat{\theta}^*_i\right)~.
    \end{align*}
    By linearity in $\theta_i$, for any $\theta'_i$ we have
    \begin{align}
        \left(\theta'_i\right)^\top g_i^*
        = V^{\pi^\dagger_i,\bs{\pi}^*_{-i}}_{i,0}(s_0,\theta'_i) - V^{\bs{\pi}^*}_{i,0}(s_0,\theta'_i)~. \label{eq:linear_subgradient}
    \end{align}
    In particular, since $\pi_i^\dagger$ is active at $\widehat{\theta}^*_i$,
    \begin{align*}
        \textnormal{Gap}_i\left(\bs{\pi}^*,\widehat{\theta}^*_i\right)
        = \left(\widehat{\theta}^*_i\right)^\top g_i^*~,
    \end{align*}
    while for any $\theta_i$,
    \begin{align*}
        \textnormal{Gap}_i\left(\bs{\pi}^*,\theta_i\right)
        \geq \theta_i^\top g_i^*~.
    \end{align*}
    Thus, we can write:
    \begin{align}\label{eq:active_branch_argument}
        \textnormal{Gap}_i\left(\bs{\pi}^*,\widehat{\theta}^*_i\right) - \textnormal{Gap}_i\left(\bs{\pi}^*,\theta_i\right) \leq \left\langle \widehat{\theta}^*_i - \theta_i, g_i^*\right\rangle~.
    \end{align}

    Now, since $\bs{\mu}$ and $\bs{\mu}_\textnormal{ref}$ cover $\bs{\pi}^*$ and its unilateral deviations, we can estimate the gradient of the value function at $\bs{\pi}^*$, or its unilateral deviations, from the gradient of the value function at $\bs{\mu}$ or $\bs{\mu}_\textnormal{ref}$. Let $\widehat{\theta}_i = (1/T_1)\sum^{T_1}_{t=1}\widehat{\theta}^{(t)}_i$. Note that, since $\widehat{\theta}_i \in \Theta_\textnormal{Unil}(\widetilde{\theta}_i)$ and $\widehat{\theta}^*_i\in\Theta_\textnormal{Unil}(\widetilde{\theta}_i)$, we have
    \begin{align}
        & \left|\left\langle \widehat{\theta}^*_i - \widehat{\theta}_i, g_i^* - \nabla_\theta V^{\bs{\mu}}_{i,0}\left(s_0,\widehat{\theta}_i\right) + \nabla_\theta V^{\bs{\pi}^*}_{i,0}\left(s_0,\widehat{\theta}_i\right)\right\rangle\right| \nonumber\\
        & = \Big| \left(V^{\pi^\dagger_i,\bs{\pi}^*_{-i}}_{i,0}\left(s_0, \widehat{\theta}^*_i\right) - V^{\bs{\mu}}_{i,0}\left(s_0,\widehat{\theta}^*_i\right)\right) - \left(V^{\pi^\dagger_i,\bs{\pi}^*_{-i}}_{i,0}\left(s_0,\widehat{\theta}_i\right) - V^{\bs{\mu}}_{i,0}\left(s_0,\widehat{\theta}_i\right)\right) \Big| \nonumber\\
        & = \left|\expctu{\tau\sim d^{\pi^\dagger_i,\bs{\pi}^*_{-i}},\tau'\sim d^{\bs{\mu}}}{\left(\phi(\tau)-\phi(\tau')\right)^\top\left(\widehat{\theta}_i-\widehat{\theta}^*_i\right) }\right| \nonumber\\
        & \leq \left|\expctu{\tau\sim d^{\pi^\dagger_i,\bs{\pi}^*_{-i}},\tau'\sim d^{\bs{\mu}}}{\left(\phi(\tau)-\phi(\tau')\right)^\top\left(\widehat{\theta}_i-{\theta}^*_i\right) }\right|
        + \left|\expctu{\tau\sim d^{\pi^\dagger_i,\bs{\pi}^*_{-i}},\tau'\sim d^{\bs{\mu}}}{\left(\phi(\tau)-\phi(\tau')\right)^\top\left(\theta^*_i-\widehat{\theta}^*_i\right) }\right| \nonumber\\
        & \leq 2\sqrt{\frac{\iota^2}{2C_R}}\sqrt{8\epsilon  + c\cdot \frac{d}{m}\log\left(\frac{nm}{\delta}\right)}~,\label{eq:pgd_additonal_01}
    \end{align}
    where the first equality follows from Equation \eqref{eq:linear_subgradient} and linearity of the value function in $\theta$; the rest follows the same argument as the proof of Lemma \ref{lem:difference_in_probs_bound}. Similarly, we have
    \begin{align*}
        \left|\left\langle \widehat{\theta}^*_i - \widehat{\theta}_i, \nabla_{\theta_i} V^{\bs{\mu}_\textnormal{ref}}_{i,0}\left(s_0,\widehat{\theta}_i\right) - \nabla_{\theta_i} V^{\bs{\pi}^*}_{i,0}\left(s_0,\widehat{\theta}_i\right) \right\rangle\right|
        \leq 2\sqrt{\frac{\iota^2}{2C_R}}\sqrt{8\epsilon  + c\cdot \frac{d}{m}\log\left(\frac{nm}{\delta}\right)}~.
    \end{align*}
    Combining, we get
    \begin{align}\label{eq:gap_gradient_estimate_bound}
        & \left|\left\langle \widehat{\theta}^*_i - \widehat{\theta}_i, g_i^* - \nabla_{\theta_i} \mathcal{R}(\widehat{\theta}_i)  \right\rangle\right|
        \leq 4\sqrt{\frac{\iota^2}{2C_R}}\sqrt{8\epsilon  + c\cdot \frac{d}{m}\log\left(\frac{nm}{\delta}\right)}~,
    \end{align}
    where
    \begin{align*}
        \mathcal{R}(\widehat{\theta}_i) = V^{\bs{\mu}}_{i,0}\left(s_0,\widehat{\theta}_i\right)  -V^{\bs{\mu}_\textnormal{ref}}_{i,0}\left(s_0,\widehat{\theta}_i\right)~.
    \end{align*}
    So, along the direction of $ \widehat{\theta}^*_i - \widehat{\theta}_i$, the gradient of the difference of values at $\bs{\mu}$ and $\bs{\mu}_\textnormal{ref}$ is a good approximation of the active linear branch of the gap at $\widehat{\theta}^*_i$. Now, in order to approximate the gradient at $\bs{\mu}$ and $\bs{\mu}_\textnormal{ref}$, we use the fact that
    \begin{align*}
        \nabla_{\theta_i}V^{\bs{\mu}}_{i,0}(s_0,\widehat{\theta}_i) = \sum^{H-1}_{h=0}(d^{\bs{\mu}}_h)^\top\Phi~,
    \end{align*}
    together with the fact that we already have access to $\epsilon$-corrupted features from $d^{\bs{\mu}}$. Thus, we can define a robust estimate of the above as
    \begin{align*}
        \widetilde{\nabla}_{\theta_i} V^{\bs{\mu}}_{i,0}\left(s_0,\widehat{\theta}_i\right)= \sum^{H-1}_{h=0} \texttt{RobMean}\left(D^{\bs{\mu}}_{h,\phi}\right)~,
    \end{align*}
    and 
    \begin{align*}
        \widetilde{\nabla}_{\theta_i} V^{\bs{\mu}_\textnormal{ref}}_{i,0}\left(s_0,\widehat{\theta}_i\right)= \sum^{H-1}_{h=0} \texttt{RobMean}\left(D^{\bs{\mu}_\textnormal{ref}}_{h,\phi}\right)~.
    \end{align*}
    Corollary \ref{cor:robust_mean_bound} gives us bounds $f(\epsilon)$ on the L2-error:
    \begin{align*}
        \norm{\widetilde{\nabla}_\theta V^{\bs{\mu}}_{i,0}\left(s_0,\theta\right) -  {\nabla}_\theta V^{\bs{\mu}}_{i,0}(s_0,\theta)} \leq O(Hf(\epsilon))~,
    \end{align*}
    and 
    \begin{align*}
        \norm{\widetilde{\nabla}_\theta V^{\bs{\mu}_\textnormal{ref}}_{i,0}\left(s_0,\theta\right) -  {\nabla}_\theta V^{\bs{\mu}_\textnormal{ref}}_{i,0}(s_0,\theta)} \leq O(Hf(\epsilon))~,
    \end{align*}
    where 
    \begin{align*}
        f(\epsilon) = \sqrt{\frac{d\log(\textnormal{poly}(d))}{m}} + \sqrt{d\epsilon} + \sqrt{\frac{d\log(1/\delta)}{m}}~.
    \end{align*}
    
    Let us now define 
    \begin{align*}
        \widetilde{\nabla}_{\theta_i}\textnormal{Gap}_i\left(\bs{\pi}^*,\widehat{\theta}_i\right) =  \widetilde{\nabla}_{\theta_i} V^{\bs{\mu}}_{i,0}\left(s_0,\widehat{\theta}_i\right) - \widetilde{\nabla}_{\theta_i} V^{\bs{\mu}_\textnormal{ref}}_{i,0}\left(s_0,\widehat{\theta}_i\right)~.
    \end{align*}
    Note that we have 
    \begin{align*}
        \expct{\norm{ \widetilde{\nabla}_\theta \textnormal{Gap}_i\left(\bs{\pi}^*,\widehat{\theta}_i\right)}} \leq 4H + O(Hf(\epsilon))~,
    \end{align*}
    due to the guarantees of \texttt{RobMean} and the feature norm bounds. Therefore, we are in the conditions of Lemma \ref{lem:stochastic_gradient_descent}. Recall that $\widehat{\theta}^*_i$ is the true maximizer of $\textnormal{Gap}_i(\bs{\pi}^*,\theta_i)$ over $\Theta_{\textnormal{Unil}}(\widetilde{\theta}_i)$. Then, using the active branch $g_i^*$ selected above, we have
    \begin{align}
        & \textnormal{Gap}_i\left(\bs{\pi}^*,\widehat{\theta}^*_i\right) - \frac{1}{T_1} \sum^{T_1}_{t=1}\textnormal{Gap}_i\left(\bs{\pi}^*,\widehat{\theta}^{(t)}_i\right)
        \leq \frac{1}{T_1} \sum^{T_1}_{t=1}\left\langle \widehat{\theta}^*_i-\widehat{\theta}^{(t)}_i,g_i^* \right\rangle \label{eq:pgd_convergence_01}\\
        & = \frac{1}{T_1} \sum^{T_1}_{t=1}\left\langle \widehat{\theta}^*_i-\widehat{\theta}^{(t)}_i,\widetilde{\nabla}_\theta\textnormal{Gap}_i\left(\bs{\pi}^*,\widehat{\theta}^{(t)}_i\right) \right\rangle  + \frac{1}{T_1} \sum^{T_1}_{t=1}\left\langle \widehat{\theta}^*_i-\widehat{\theta}^{(t)}_i, g_i^* - \widetilde{\nabla}_\theta\textnormal{Gap}_i\left(\bs{\pi}^*,\widehat{\theta}^{(t)}_i\right)\right\rangle \label{eq:pgd_convergence_02}\\
        & \leq \frac{1}{T_1}\left(\frac{\norm{\widehat{\theta}^*_i-\widehat{\theta}^{(1)}_i}^2}{2\eta} + \frac{\eta T_1 (4H+O(Hf(\epsilon)))^2}{2}\right) \nonumber\\
        & \quad\quad + \frac{1}{T_1} \sum^{T_1}_{t=1}\left\langle \widehat{\theta}^*_i-\widehat{\theta}^{(t)}_i, g_i^* - \nabla_\theta V^{\bs{\mu}}_{i,0}\left(s_0,\widehat{\theta}^{(t)}_i\right) + \nabla_\theta V^{\bs{\mu}_\textnormal{ref}}_{i,0}\left(s_0,\widehat{\theta}^{(t)}_i\right)\right\rangle \nonumber\\
        & \quad\quad + \frac{1}{T_1} \sum^{T_1}_{t=1}\left\langle \widehat{\theta}^*_i-\widehat{\theta}^{(t)}_i,\nabla_\theta V^{\bs{\mu}}_{i,0}\left(s_0,\widehat{\theta}^{(t)}_i\right) - \nabla_\theta V^{\bs{\mu}_\textnormal{ref}}_{i,0}\left(s_0,\widehat{\theta}^{(t)}_i\right) - \widetilde{\nabla}_\theta\textnormal{Gap}_i\left(\bs{\pi}^*,\widehat{\theta}^{(t)}_i\right)\right\rangle \label{eq:pgd_convergence_03}\\
        & \leq O\left(\frac{H^2 + f(\epsilon)}{\sqrt{T_1}}\right)
        + 4\sqrt{\frac{\iota^2}{2C_R}}\sqrt{8\epsilon  + c\cdot \frac{d}{m}\log\left(\frac{nm}{\delta}\right)} \nonumber\\
        & \quad\quad + \frac{1}{T_1}\sum_{t=1}^{T_1}\norm{\widehat{\theta}^*_i-\widehat{\theta}^{(t)}_i}\norm{\nabla_\theta V^{\bs{\mu}}_{i,0}\left(s_0,\widehat{\theta}^{(t)}_i\right) - \nabla_\theta V^{\bs{\mu}_\textnormal{ref}}_{i,0}\left(s_0,\widehat{\theta}^{(t)}_i\right) - \widetilde{\nabla}_\theta\textnormal{Gap}_i\left(\bs{\pi}^*,\widehat{\theta}^{(t)}_i\right)} \label{eq:pgd_convergence_03.5}\\
        & \leq O\left(\frac{H^2 + f(\epsilon)}{\sqrt{T_1}}\right) + 4\sqrt{\frac{\iota^2}{2C_R}}\sqrt{8\epsilon  + c\cdot \frac{d}{m}\log\left(\frac{nm}{\delta}\right)} + 2H^2\sqrt{d}f(\epsilon)\nonumber\\
        & \leq \widetilde{O}\left(\left(\frac{1}{\sqrt{C_R}}+\frac{1}{\sqrt{C_P}}\right)\cdot\left(H^{5/2}d^{3/4}\sqrt{\epsilon} + H^2\textnormal{poly}(d)\frac{1}{\sqrt{m}}\right) + \frac{H^2\sqrt{\textnormal{poly}(d)}}{\sqrt{T_1}}\left(\sqrt{\epsilon}+ \frac{1}{\sqrt{m}}\right) \right)\nonumber~,
    \end{align}
    where Equation \eqref{eq:pgd_convergence_01} follows from Equation \eqref{eq:active_branch_argument}; Equation \eqref{eq:pgd_convergence_02} adds and subtracts the same term; Equation \eqref{eq:pgd_convergence_03} follows from Lemma \ref{lem:stochastic_gradient_descent}; Equation \eqref{eq:pgd_convergence_03.5} follows from Cauchy-Schwarz and Equation \eqref{eq:gap_gradient_estimate_bound}; and the last inequality follows from Corollary \ref{cor:robust_mean_bound}.
\end{proof}

\begin{theorem}\label{thm:gap_bounds_for_unilateral}
    Let $\epsilon\in[0,1/2)$, $\lambda\geq \Omega(dH\log(m/\delta)/m)$, and $\delta >0$. Set $\Theta_\textnormal{Unil}(\cdot)$ as in Equation \eqref{eq:confidence_set_unilateral} and $\Gamma(s,\bs{a})=E(d,m,\delta,\epsilon)\cdot\norm{\phi(s,\bs{a})}_{\Lambda_h^{-1}}$. Suppose Assumption \ref{asmp:low_relative_uncertainty} is satisfied and PGA is run for $T_1$ steps with learning rate $\eta= O(1/\sqrt{T_1})$. Then, there exist robust subroutines \texttt{RobEst}, \texttt{TrimmedMLE}, and \texttt{RobMean} such that, with probability at least $1-\delta$, the output $\widetilde{\bs{\pi}}$ of Algorithm \ref{alg:pmael_unilateral} with subroutines \texttt{RobEst}, \texttt{TrimmedMLE}, \texttt{RobMean} and \texttt{RewardEst}, satisfies
    \begin{align*}
        \textnormal{Gap}(\widetilde{\bs{\pi}})
        \leq
        \widetilde{O}\left(
        \left(\frac{1}{\sqrt{C_R}}+\frac{1}{\sqrt{C_P}}+\frac{1}{\sqrt{T_1}}\right)\cdot
        \left(H^{5/2}nd^{3/4}\sqrt{\epsilon} + H^2n\sqrt{\frac{\textnormal{poly}(d)}{m}}\right)
        \right)~.
    \end{align*}
\end{theorem}

\begin{proof}
    Define $\bs{\theta}^*=(\theta^*_1,\ldots,\theta^*_n)$ and let
    \begin{align*}
        f(\epsilon) = \widetilde{O}\left(\left(\frac{1}{\sqrt{C_R}}+\frac{1}{\sqrt{C_P}}\right)\cdot\left(H^{5/2}d^{3/4}\sqrt{\epsilon} + H^2\textnormal{poly}(d)\frac{1}{\sqrt{m}}\right) + \frac{H^2\sqrt{\textnormal{poly}(d)}}{\sqrt{T_1}}\left(\sqrt{\epsilon}+ \frac{1}{\sqrt{m}}\right) \right)~.
    \end{align*}
    Also define
    \begin{align*}
        \Delta
        :=
        \widetilde{O}\left(
        n\cdot\left(\frac{1}{\sqrt{C_R}}+\frac{1}{\sqrt{C_P}}\right)\cdot
        \left(H^{5/2}d^{3/4}\sqrt{\epsilon} + H^2\textnormal{poly}(d)\frac{1}{\sqrt{m}}\right)
        \right)~.
    \end{align*}
    Moreover, let $\widehat{\bs{\theta}}^*$ denote the maximizer of the gap with respect to $\bs{\pi}^*$ over the unilateral confidence set, i.e.,
    \begin{align*}
        \widehat{\bs{\theta}}^* \in \arg\max_{\bs{\theta}\in\Theta_\textnormal{Unil}(\widetilde{\theta}_1)\times\ldots\times\Theta_\textnormal{Unil}(\widetilde{\theta}_n)} \textnormal{Gap}\left(\bs{\pi}^*,\bs{\theta}\right)~.
    \end{align*}
    We have
    \begin{align}
       & \textnormal{Gap} \left(\widetilde{\bs{\pi}},\bs{\theta}^*\right)
       \leq \textnormal{Gap}\left(\bs{\pi}^*,\bs{\theta}^*\right) + \Delta \label{eq:unilateral_gap_01}\\
       & \leq \textnormal{Gap}\left(\bs{\pi}^*,\widehat{\bs{\theta}}^*\right) + \Delta \label{eq:unilateral_gap_02}\\
       & \leq \textnormal{Gap}\left(\bs{\pi}^*,\widehat{\bs{\theta}}\right)+n\cdot f(\epsilon)+\Delta \label{eq:unilateral_gap_03}\\
       & \leq \textnormal{Gap}\left(\widetilde{\bs{\pi}},\widehat{\bs{\theta}}\right)+n\cdot f(\epsilon)+2\Delta \label{eq:unilateral_gap_04}\\
       & \leq \widetilde{\textnormal{Gap}}\left(\widetilde{\bs{\pi}},\widehat{\bs{\theta}}\right)+n\cdot f(\epsilon)+2\Delta \label{eq:unilateral_gap_05}\\
       & \leq \widetilde{\textnormal{Gap}}\left({\bs{\pi}}^*,\widehat{\bs{\theta}}\right)+n\cdot f(\epsilon)+2\Delta \label{eq:unilateral_gap_06}\\
       & = \sum_{i\in[n]}\overline{V}^{\dagger,\bs{\pi}^*_{-i}}_{i,0}\left(s_0,\widehat{\theta}_i\right) - \underline{V}^{\bs{\pi}^*}_{i,0}\left(s_0,\widehat{\theta}_i\right) + n\cdot f(\epsilon)+2\Delta \label{eq:unilateral_gap_07}\\
       & \leq \sum_{i\in[n]}\overline{V}^{\dagger,\bs{\pi}^*_{-i}}_{i,0}\left(s_0,\widehat{\theta}_i\right) - V^{\bs{\mu}_\textnormal{ref}}_{i,0}\left(s_0,\widehat{\theta}_i\right) \nonumber\\
       & \quad\quad - \left( \underline{V}^{\bs{\pi}^*}_{i,0}\left(s_0,\widehat{\theta}_i\right) - V^{\bs{\mu}_\textnormal{ref}}_{i,0}\left(s_0,\widehat{\theta}_i\right)\right) + n\cdot f(\epsilon)+2\Delta \label{eq:unilateral_gap_08}\\
       & \leq \widetilde{O}\left(
       n\cdot\left(\frac{1}{\sqrt{C_R}}+\frac{1}{\sqrt{C_P}}\right)\cdot
       \left(H^{5/2}d^{3/4}\sqrt{\epsilon} + H^2\textnormal{poly}(d)\frac{1}{\sqrt{m}}\right)
       \right) + n\cdot f(\epsilon)+2\Delta \label{eq:unilateral_gap_09}\\
       & \leq \widetilde{O}\left(
       n\cdot\left(\frac{1}{\sqrt{C_R}}+\frac{1}{\sqrt{C_P}}\right)\cdot
       \left(H^{5/2}d^{3/4}\sqrt{\epsilon} + H^2\textnormal{poly}(d)\frac{1}{\sqrt{m}}\right)
       + \frac{H^2n\sqrt{\textnormal{poly}(d)}}{\sqrt{T_1}}\left(\sqrt{\epsilon}+ \frac{1}{\sqrt{m}}\right)
       \right)\nonumber~.
    \end{align}
    Equation \eqref{eq:unilateral_gap_01} follows from Lemma \ref{lem:main_gap_bounds_unilaterl}. Equation \eqref{eq:unilateral_gap_02} follows from the definition of $\widehat{\bs{\theta}}^*$. Equation \eqref{eq:unilateral_gap_03} follows from Proposition \ref{pro:pgd_guarantees}, applied coordinate-wise at $\bs{\pi}^*$. Equation \eqref{eq:unilateral_gap_04} follows again from Lemma \ref{lem:main_gap_bounds_unilaterl}. Equation \eqref{eq:unilateral_gap_05} follows from Lemma \ref{lem:value_bounds_unilateral}. Equation \eqref{eq:unilateral_gap_06} follows by design of Algorithm \ref{alg:pmael_unilateral}, since $\widetilde{\bs{\pi}}$ minimizes the estimated gap at $\widehat{\bs{\theta}}$. Equation \eqref{eq:unilateral_gap_07} follows by definition of estimated gap. In Equation \eqref{eq:unilateral_gap_08} we add and subtract the same term. Equation \eqref{eq:unilateral_gap_09} follows from Lemma \ref{lem:main_gap_bounds_unilaterl}. The final inequality follows from the definitions of $f(\epsilon)$ and $\Delta$.
\end{proof}

\section{Proof of Theorem \ref{thm:robust_cce_learning}}

This section includes full proof of Theorem \ref{thm:robust_cce_learning}. We begin by establishing regret guarantees on the Optimistic Hedge algorithm applied to our setting. 

\begin{lemma}\label{lem:regret_bounds}
    Denote by $\widetilde{\bs{\pi}}$ the joint policy returned by Optimistic Hedge run for $T_2$ rounds. For each $i\in[n]$, let 
    \begin{align*}
        & \overline{\textnormal{Reg}}_{i,h,T_2} = \max_{\pi^\dagger_{i,h}} \expctu{a^\dagger_i\sim\pi^\dagger_{i,h}(\cdot|s),a'_i\sim\widetilde{\pi}_{i,h}(\cdot|s)}{  \expctu{\bs{a}_{-i}\sim\widetilde{\bs{\pi}}_{-i,h}(\cdot|s)}{\overline{Q}^{\pi^\dagger_i,\widetilde{\bs{\pi}}_{-i,h}}_{i,h}(s,a^\dagger_i,\bs{a}_{-i}) - \underline{Q}^{\widetilde{\bs{\pi}}}(s,a'_i,\bs{a}_{-i})} } \\ & - \min_{\pi'_{i,h}}\max_{\pi^\dagger_{i,h}} \expctu{a^\dagger_i\sim\pi^\dagger_{i,h}(\cdot|s),a'_i\sim\pi'_{i,h}(\cdot|s)}{  \expctu{\bs{a}_{-i}\sim\bs{\pi}_{-i,h}(\cdot|s)}{\overline{Q}^{\pi^\dagger_i,\widetilde{\bs{\pi}}_{-i}}_{i,h}(s,a^\dagger_i,\bs{a}_{-i}) - \underline{Q}^{\pi'_i,\widetilde{\bs{\pi}}_{-i}}(s,a'_i,\bs{a}_{-i})} }~.
    \end{align*}
    For every $i\in[n]$ and $h\in[H-1]$, we have
    \begin{align*}
        \sum_{i\in[n]} \overline{\textnormal{Reg}}_{i,h,T_2} \leq O\left( n^2H\cdot\log|A|\cdot\log^4T_2 \right)
    \end{align*}
\end{lemma}
\begin{proof}
    Recall that the loss we use for Optimistic Hedge is defined as
    \begin{align*}
        \mathcal{L}^s_i(a^\dagger,a') = \expctu{\bs{a}_{-i}\sim\widetilde{\bs{\pi}}_{-i,h}(\cdot|s)}{\overline{Q}^{\dagger,\widetilde{\bs{\pi}}_{-i}}_{i,h}(s,a^\dagger_i,\bs{a}_{-i}) - \underline{Q}^{\widetilde{\bs{\pi}}}_{i,h}(s,a',\bs{a}_{-i})}~.
    \end{align*}
    Now, let $\pi^{(\dagger,(t)}_i$ denote the iterations of player $i$ with respect to the $\max$ problem. Note that, for every agent $i$ and state $s$, after $T_2$ steps, we have
    \begin{align*}
        \max_{\pi^\dagger_i} & \sum^{T_2}_{t=1} \expctu{a^\dagger\sim\pi^\dagger_i,a'\sim\widetilde{\pi}^{(t)}_i}{\mathcal{L}^s_i(a^\dagger,a')} - \min_{\pi_i}\max_{\pi^\dagger_i}\sum^{T_2}_{t=1} \expctu{a^\dagger\sim\pi^\dagger_i,a'\sim\widetilde{\pi}^{(t)}_i}{\mathcal{L}^s_i(a^\dagger,a')}   \\
            & = \underbrace{\max_{\pi^\dagger_i}\sum^{T_2}_{t=1} \expctu{a^\dagger\sim\pi^\dagger_i,a'\sim\widetilde{\pi}^{(t)}_i}{\mathcal{L}^s_i(a^\dagger,a')} - \sum^{T_2}_{t=1}\expctu{a^\dagger\sim\pi^{\dagger,(t)}_i,a'\sim\widetilde{\pi}^{(t)}_i}{\mathcal{L}^s_i(a^\dagger,a')}}_{\textnormal{Reg}_{i,h,T_2}\; \text{for}\; \pi^{\dagger,(t)}_i} \\ & \quad + \underbrace{\sum^{T_2}_{t=1}\expctu{a^\dagger\sim\pi^{\dagger,(t)}_i,a'\sim\widetilde{\pi}^{(t)}_i}{\mathcal{L}^s_i(a^\dagger,a')} - \min_{\pi_i}\sum^{T_2}_{t=1} \expctu{a^\dagger\sim\pi^{\dagger,(t)}_i,a'\sim\widetilde{\pi}^{(t)}_i}{\mathcal{L}^s_i(a^\dagger,a')}}_{\textnormal{Reg}_{i,h,T_2}\; \text{for}\; \pi^{(t)}_i} \\
            & \quad + \underbrace{\min_{\pi_i}\sum^{T_2}_{t=1} \expctu{a^\dagger\sim\pi^{\dagger,(t)}_i,a'\sim\widetilde{\pi}^{(t)}_i}{\mathcal{L}^s_i(a^\dagger,a')} - \min_{\pi_i}\max_{\pi^\dagger_i}\sum^{T_2}_{t=1} \expctu{a^\dagger\sim\pi^\dagger_i,a'\sim\widetilde{\pi}^{(t)}_i}{\mathcal{L}^s_i(a^\dagger,a')}}_{\leq 0} \\
        & \leq O\left( nH\log|A_i|\log^4T_2 \right) ~,
    \end{align*}
    where for the inequality we have used Theorem \ref{thm:optimistic_hedge_regret} applied on the regrets with respect to the $\max$ and $\min$ players, and the fact that $\min_x f(x,y) \leq \min_x \max_y f(x,y)$, due to the monotonicity of the $\min$ operator. This implies that, the empirical distribution of the sequence of policies up to time step $T_2$, which equals the returned policy $\widetilde{\bs{\pi}}$, satisfies, for any player $i$ and state $s$,
    \begin{align*}
        \max_{\pi^\dagger_i} \expctu{a^\dagger\sim\pi^\dagger_i,a'\sim\widetilde{\pi}_i}{\mathcal{L}^s_i(a^\dagger,a')} - \min_{\pi_i}\max_{\pi^\dagger_i} \expctu{a^\dagger\sim\pi^\dagger_i,a'\sim\pi_i}{\mathcal{L}^s_i(a^\dagger,a')} \leq O\left(\frac{n\cdot\log|A|\cdot\log^4T_2}{T_2}\right)~.
    \end{align*}
    In particular, we have 
    \begin{align*}
        & \sum_{i\in[n]}\max_{\pi^\dagger_{i,h}} \expctu{a^\dagger_i\sim\pi^\dagger_{i,h}(\cdot|s),a'_i\sim\widetilde{\pi}_{i,h}(\cdot|s)}{  \expctu{\bs{a}_{-i}\sim\widetilde{\bs{\pi}}_{-i,h}(\cdot|s)}{\overline{Q}^{\dagger,\widetilde{\bs{\pi}}_{-i,h}}_{i,h}(s,a^\dagger_i,\bs{a}_{-i}) - \underline{Q}^{\widetilde{\bs{\pi}}}_{i,h}(s,a'_i,\bs{a}_{-i})} } \\ & - \min_{\bs{\pi}'_h}\sum_{i\in[n]}\max_{\pi^\dagger_{i,h}} \expctu{a^\dagger_i\sim\pi^\dagger_{i,h}(\cdot|s),a'_i\sim\pi'_{i,h}(\cdot|s)}{  \expctu{\bs{a}_{-i}\sim\bs{\pi}_{-i,h}(\cdot|s)}{\overline{Q}^{\dagger,\widetilde{\bs{\pi}}_{-i}}_{i,h}(s,a^\dagger_i,\bs{a}_{-i}) - \underline{Q}^{\widetilde{\bs{\pi}}}_{i,h}(s,a'_i,\bs{a}_{-i})} } \\
        & \leq O\left(\frac{n^2H\cdot\log|A|\cdot\log^4T_2}{T_2}\right)~,
    \end{align*}
    where we have used the fact that the $\min$ of the sum is larger than the sum of individual $\min$s.
\end{proof}

Finally, we are ready to state and prove Theorem \ref{thm:robust_cce_learning}. We restate it here for convenience. 

\begin{theorem}
    Let $\epsilon\in[0,1/2), \lambda\geq \Omega(dH\log(m/\delta)/m)$, and $\delta >0$. Set $\Theta_\textnormal{Unil}(\cdot)$ as in Equation \eqref{eq:confidence_set_unilateral} and $\Gamma(s,\bs{a})=E(d,m,\delta,\epsilon)\cdot\norm{\phi(s,\bs{a})}_{\Lambda_h^{-1}}$. Suppose Assumption \ref{asmp:low_relative_uncertainty} is satisfied, PGA is run for $T_1$ steps with learning rate $\eta_1= O(1/\sqrt{T_1})$, and \texttt{OptimisticHedge} is run for $T_2$ steps with learning rate $\eta_2=O(1/(n\log^4T_2))$. Then, there exist robust subroutines \texttt{RobEst}, \texttt{TrimmedMLE}, and \texttt{RobMean}  such that, with probability at least $1-\delta$, the output $\widetilde{\bs{\pi}}$ of Algorithm \ref{alg:pmael_unilateral_cce} with subroutines \texttt{RobEst}, \texttt{TrimmedMLE}, \texttt{RobMean} and \texttt{OptimisticHedge}, satisfies
    \begin{align*}
        \textnormal{Gap}(\widetilde{\bs{\pi}}) \leq \widetilde{O}\left(\left(\frac{1}{\sqrt{C_R}}+\frac{1}{\sqrt{C_P}}+\frac{1}{\sqrt{T_1}}\right)\cdot\left(H^{5/2}nd^{3/4}\sqrt{\epsilon} + H^2n\frac{\sqrt{\textnormal{poly}(d)}}{\sqrt{m}}\right) + \frac{Hn^2}{T_2}\right)~.
    \end{align*}
\end{theorem}
\begin{proof}
    Similar to the proof of Theorem \ref{thm:unilateral_coverage_bounds}, we can write
    \begin{align}
        \textnormal{Gap} & \left(\widetilde{\bs{\pi}},\bs{\theta}^*\right) \leq {\textnormal{Gap}}\left(\widetilde{\bs{\pi}},\widehat{\bs{\theta}}^*\right)\nonumber\\
            & \leq \textnormal{Gap}\left( \widetilde{\bs{\pi}},\widehat{\bs{\theta}} \right)+n\cdot f(\epsilon)\nonumber\\
        & \leq \widetilde{\textnormal{Gap}}\left( \widetilde{\bs{\pi}},\widehat{\bs{\theta}} \right)+n\cdot f(\epsilon)\nonumber\\
        & \leq \min_{\bs{\pi}}\widetilde{\textnormal{Gap}}\left( {\bs{\pi}},\widehat{\bs{\theta}} \right) + O\left(\frac{n^2H\cdot\log|A|\cdot\log^4T_2}{T_2}\right)+n\cdot f(\epsilon) \label{eq:final_cce_bounds_01}\\
            & \leq \widetilde{\textnormal{Gap}}\left( {\bs{\pi}}^*,\widehat{\bs{\theta}} \right) + O\left(\frac{n^2H\cdot\log|A|\cdot\log^4T_2}{T_2}\right)+n\cdot f(\epsilon)\nonumber\\
        & = \sum_{i\in[n]}\overline{V}^{\dagger,\bs{\pi}^*_{-i}}_{i,0}\left(s_0,\widehat{\theta}_i\right) - \underline{V}^{\bs{\pi}^*}_{i,0}\left(s_0,\widehat{\theta}_i\right) + O\left(\frac{n^2H\cdot\log|A|\cdot\log^4T_2}{T_2}\right) + n\cdot f(\epsilon)\nonumber\\
            & \leq \widetilde{O}\left( n\cdot\left(\frac{1}{\sqrt{C_R}}+\frac{1}{\sqrt{C_P}}+\frac{1}{\sqrt{T_1}}\right)\cdot\left(H^{5/2}d^{3/4}\sqrt{\epsilon} + H^2\textnormal{poly}(d)\frac{1}{\sqrt{m}}\right) +\frac{n^2H}{T_2}\right)\label{eq:final_cce_bounds_02}~,
    \end{align}
    where Equation \eqref{eq:final_cce_bounds_01} follows from Lemma \ref{lem:regret_bounds} and Equation \eqref{eq:final_cce_bounds_02} follows from Theorem \ref{thm:unilateral_coverage_bounds}.
\end{proof}

\section{Technical Results}

This section includes various miscellaneous technical results that are used throughout the proofs in the paper. 

\begin{lemma}[\cite{zanette2021cautiously}]\label{lem:concentration_of_covariances} Let $\{\phi_i\}^m_{i=1}\subset \mathbb{R}^d$ be i.i.d. samples from an underlying bounded distribution $\bs{\mu}_\textnormal{ref}$, with $\norm{\phi_i}_2\leq 1$ and covariance $\Sigma_{\bs{\mu}_\textnormal{ref}}$. Define
\begin{align*}
    \Lambda = \sum^m_{i=1}\phi_i\phi_i^\top +\lambda I~,
\end{align*}
for some $\lambda \geq \Omega(d\log(m/\delta))$. Then, with probability at least $1-\delta$, we have
\begin{align*}
    \frac{1}{3}\left(m\Sigma_{\bs{\mu}_\textnormal{ref}}+\lambda I\right) \preceq \Lambda \preceq \frac{5}{3}\left(m\Sigma_{\bs{\mu}_\textnormal{ref}}+\lambda I\right)~.
\end{align*}
    
\end{lemma}

Next, we state a result that bounds the difference in log probabilities of parameters in a given space.
\begin{lemma}[Lemma 2 of \citep{zhan2023provable} for the linear setting]\label{lem:bound_on_log_prob_difference}
    With probability at least $1-\delta$, we have, for every agent $i$ and $\theta\in\Theta_\textnormal{Unil}(\widetilde{\theta}_i)$:
    \begin{align*}
        \expctu{\substack{\tau\sim d^{\bs{\mu}}\\ \tau'\sim d^{\bs{\mu}_\textnormal{ref}}}}{\norm{\mathbb{P}\left(\cdot|\tau,\tau',\theta^*_i\right) - \mathbb{P}\left(\cdot|\tau,\tau',\theta\right)}^2_1} \leq \frac{c}{m}\sum^m_{j=1}\log\left(\frac{\mathbb{P}\left(\widetilde{o}_j|\widetilde{\tau}_j,\widetilde{\tau}'_j,\theta^*_i\right)}{\mathbb{P}\left(\widetilde{o}_j|\widetilde{\tau}_j,\widetilde{\tau}'_j,\theta\right)}\right) +\log\left( d\log \left(n/\delta\right)\right)~,
    \end{align*}
    where $c>0$ is an absolute constant.
\end{lemma}
\begin{proof}
    This is just an application of Proposition 1 of \citep{zhan2023provable} to the linear setting, which then induces the result above by applying the union bound over all agents. 
\end{proof}

Next, we show that the inverse of the sigmoid link function is Lipschitz on a bounded domain.
\begin{lemma}\label{lem:lipschitz_sigmoid_inverse}
    Let $-\infty <a,b<\infty$ be two real numbers and let $\sigma(x)=1/(1+\exp(-x))$ be defined on the domain $x\in[a,b]$. Then, there exists a positive number $\iota <\infty$, such that the inverse $\sigma^{-1}$ of $\sigma$ is Lipschitz with constant $\iota$, that is,
    \begin{align*}
        \sup_{p(x)\in(0,1):x\in[a,b]}\left|\frac{\partial\sigma^{-1}(p)}{\partial p}\right| \leq \iota~.
    \end{align*}
    As a consequence, if the rewards are bounded, then the inverse of the preference link function is Lipschitz continuous for some $\iota >0$.
\end{lemma}
\begin{proof}
    First, note that the derivative of the inverse of the sigmoid can be written as
    \begin{align*}
        \frac{\partial\sigma^{-1}(p)}{\partial p} = \frac{\partial}{\partial p} \log\frac{p}{1-p} = \frac{1}{p(1-p)}~.
    \end{align*}
    Now, since $x\in[a,b]$ and $\sigma$ is continuous in $\mathbb{R}$, then, there exist $-\infty < a',b'<\infty$, such that $p=\sigma(x)\in [a',b']$, for every $x\in[a,b]$. Moreover, since the function $p(1-p)$ is also continuous in $\mathbb{{R}}$, then, there exist $-\infty < a'',b''<\infty$, such that $p(1-p)\in [a'',b'']$, for all $p\in [a',b']$. Thus, there exists a positive constant $\iota$, such that 
    \begin{align*}
        \frac{1}{p(1-p)} \leq \iota~,
    \end{align*}
    for all $x\in[a,b]$. This implies that the function $\sigma^{-1}$ is Lipschitz on the domain of $\sigma$. 

    For the final statement of the result, note that the preference function uses differences in expected rewards that are individually bounded in $[-\sqrt{d},\sqrt{d}]$. Thus, we have
    \begin{align*}
        \sum^{H-1}_{h=0} R(s_h,a_h) - R'(s_h,a_h) \in [-2H\sqrt{d},2H\sqrt{d}]~.
    \end{align*}
    Thus, the domain of the sigmoid link function, used for our preference model, has a bounded domain, which implies that its inverse is Lipschitz continuous. 
\end{proof}



\begin{corollary}\label{cor:smooth_mu_ref_function}
    The function $f(\theta) = \expctu{\tau\sim d^{\bs{\mu}_\textnormal{ref}}}{\phi(\tau)^\top\theta}$ is $H$-Lipschitz and convex. Moreover, the set $\Theta_\textnormal{Unil}(\theta')$ is a convex set, for any $\theta'$ and $\lambda >0$.
\end{corollary}
\begin{proof}
    Note that we have
    \begin{align*}
        \norm{\nabla_\theta f(\theta)} = \norm{(d^{\bs{\mu}_\textnormal{ref}})^\top\Phi} \leq \norm{d^{\bs{\mu}_\textnormal{ref}}}_1\norm{\Phi}_\infty \leq \max_{(s,a)}\norm{\phi(s,a)}_1 \leq H\max_{(s,a)}\norm{\phi(s,a)}_2\leq H~.
    \end{align*}
    Convexity follows from the direct observation that $\nabla_\theta (d^{\bs{\mu}_\textnormal{ref}})^\top\Phi = 0$. The convexity of $\Theta_\textnormal{Unil}(\theta')$ is observed in \citep{mandal2024corruption}.
\end{proof}

Next, we provide an upper bound on the error of the estimate returned by the \texttt{RobMean} algorithm.

\begin{theorem}[Proposition 1.5 of \citep{diakonikolas2020outlier}]\label{thm:robust_mean_bound}
    Let $T$ be an $\epsilon$-corrupted set of $m$ samples from a distribution in $\mathbb{R}^d$ with mean $\rho$ and covariance $\Sigma$. Let $\epsilon'$ be in the order of $(\log(1/\delta)/m+ \epsilon) \leq c$, for a constant $c > 0$. Then any
    stability-based algorithm on input $T$ and $\epsilon'$, efficiently computes $\widetilde{\rho}$ such that with probability at least
    $1-\delta$, we have 
    \begin{align*}
        \norm{\rho - \widetilde{\rho}} = O\left( \sqrt{\frac{Tr(\Sigma)\log (Tr(\Sigma)/\norm{\Sigma)}}{m}} + \sqrt{\norm{\Sigma}\epsilon} + \sqrt{\frac{\norm{\Sigma}\log(1/\delta)}{m}} \right)~.
    \end{align*}
\end{theorem}
\begin{corollary}\label{cor:robust_mean_bound}
    The \texttt{RobMean} algorithm returns a gradient estimate that satisfies 
    \begin{align*}
        \norm{\widetilde{\nabla}_\theta V^{\bs{\mu}_\textnormal{ref}}_{i,0}(s_0) - \nabla_\theta V^{\bs{\mu}_\textnormal{ref}}_{i,0}(s_0)} \leq O\left( \sqrt{\frac{d\log(\textnormal{poly}(d))}{m}} + \sqrt{d\epsilon} + \sqrt{\frac{d\log(1/\delta)}{m}}\right)~.
    \end{align*}
\end{corollary}
\begin{proof}
    Note that, since $\norm{\phi(s,\bs{a})}\leq 1$, for all state action tuples, then $\norm{\Phi}\leq d$ and $Tr(\Phi) \leq d$.
\end{proof}
Next, we state an upper bound on the individual regret of each agent when playing Optimistic Hedge.

\begin{theorem}[Theorem 1.1 of \citep{daskalakis2021near}]\label{thm:optimistic_hedge_regret}
    There are constants $C,C'>1$ so that the following holds. Suppose a time horizon $T \in\mathbb{N}$ and a game $G$ with $n$ players and $|A_i|$ actions for each player $i\in[n]$ is given. Suppose all players play according to Optimistic Hedge with any positive step size $\eta = 1/(C n \log^4T)$. Then, for any player $i\in[n]$, the regret of player $i$ satisfies 
    \begin{align*}
        \textnormal{Reg}_{i,T} \leq O\left( n\cdot \log |A|\cdot\log^4T\right)~.
    \end{align*}
\end{theorem}

\begin{lemma}[Lemma E.6 of \citep{mandal2024corruption}]\label{lem:stochastic_gradient_descent} Let $y_1\in W$, and $\eta >0$. Define the sequence $y_2,\ldots,y_{n+1}$
and $h_1,\ldots,h_n$ such that, for $k=1,\ldots,n$
$$y_{k+1} = \mathcal{P}_W\left(y_k -\eta \widehat{h}_k\right)~,$$
and $\widehat{h}_k$ satisfies $$\expct{\widehat{h}_k|\mathcal{F}_{k-1}}=h_k,\;\; \textnormal{and}\;\; \expct{\norm{\widehat{h}_k}^2|\mathcal{F}_{k-1}}\leq G^2~,$$
where $\mathcal{F}_k$ are the $\sigma$-algebras on which the variables up to $k$ are defined. Then, for any $y^*\in W$, we have 
\begin{align*}
    \expct{\sum^n_{k=1}\left\langle y^*-y_k,h_k\right\rangle} \leq \frac{\norm{y_1-y^*}^2}{2\eta} + \frac{\eta nG^2}{2}~.
\end{align*}
\end{lemma}

        

\section{Additional Algorithm Pseudocodes}\label{sec:additional_algorithms}

\begin{algorithm}[!ht]
    \caption{Alternating Minimization (\texttt{Trimmed MLE}) for full sample corruption.}\label{alg:regularized_mle}
    \begin{algorithmic}[1]
        \REQUIRE Corrupted data $D$; corruption parameter $\epsilon$; slackness parameter $\nu$.
        \STATE Split $D$ into equally-sized $D_1$ and $D_2$, uniformly at random. 
        \STATE Use $D_1$ to build a robust estimate $\widehat{\Sigma}$ of the $\Sigma^-_{\boldsymbol{\mu},\boldsymbol{\mu}_\textnormal{ref}}$ \citep{diakonikolas2025sos}.
        \STATE Whiten covariates using $\widehat{\Sigma}$, i.e. form $\widetilde{D}=\{ \widehat{\Sigma}^{-1/2}(\phi(\tau)-\phi(\tau'))| (\tau,\tau')\in D_2\}$.
        \STATE Let $\widehat{D} \leftarrow \texttt{Filtering}(\widetilde{D},\epsilon)$ (Algorithm 4 of \citep{dong2019quantum}).
        \STATE Define $L_\theta(\tau,\tau',o) = \log\sigma\left(o\cdot\theta^\top\widehat{\Sigma}^{1/2}\left(\phi(\tau)-\phi(\tau')\right)\right)~,$ for $\tau \in \widehat{D}$.
        \STATE Set $\widetilde{\theta}_0=0$.
        \FOR{$t=1,2,\ldots$}
        \STATE $\widetilde{S}_{t} = \arg\max_{\substack{S\subset \widehat{D}: \\ |S|=(1-\epsilon)m}}\sum_{(\tau,\tau',o)\in S} L_{\widetilde{\theta}_t}(\tau,\tau',o)~.$
        \STATE $\widetilde{\theta}_{t+1} = \arg\max_{\theta: \norm{\theta}\leq \sqrt{Hd}} \sum_{(\tau,\tau',o)\in \widetilde{S}_t}L_{\theta}(\tau,\tau',o)~.$
        \IF{$\sum_{(\tau,\tau',o)\in \widetilde{S}_t}L_{\widetilde{\theta}_{t+1}}(\tau,\tau',o) \leq \sum_{(\tau,\tau',o)\in \widetilde{S}_t}L_{\widetilde{\theta}_t}(\tau,\tau',o)+\nu$}
        \STATE Return $\widetilde{\theta}_{t+1}$.
        \ENDIF
        \ENDFOR
    \end{algorithmic}
\end{algorithm}

\begin{algorithm}
    \caption{Robust Estimation of Value Functions}\label{alg:robust_value_estimation}
    \begin{algorithmic}[1]
        \REQUIRE Dataset $D$, policy $\bs{\pi}$, agent $i$, reward functions $\overline{R}_i$ and $\underline{R}_i$, bonus function $\Gamma(\cdot,\cdot)$.
        \STATE Initialize $\underline{V}^{\bs{\pi}}_{i,H}(\cdot)=\overline{V}^{\dagger,\bs{\pi}_{-i}}_{i,H}(\cdot)=0$, for all agents $i\in[n]$.
        \FOR{$h=H-1, \ldots,0$:}
        \STATE $\underline{\omega}^{\bs{\pi}}_{i,h} = \texttt{RobEst}\left(\phi(s_h,\bs{a}_h),\underline{R}_{i,h}(s_h,\bs{a}_h) + \underline{V}^{\bs{\pi}}_{i,h+1}(s)\right)$.
        \STATE $\overline{\omega}^{\dagger,\bs{\pi}_{-i}}_{i,h} = \texttt{RobEst}\left(\phi(s_h,\bs{a}_h),\overline{R}_{i,h}(s_h,\bs{a}_h) + \overline{V}^{\dagger,\bs{\pi}_{-i}}_{i,h+1}(s)\right)$.
        \STATE $\underline{Q}_{i,h}^{\bs{\pi}}(\cdot,\cdot) = \textnormal{Clip}_{\left[-(H-h)\sqrt{d},(H-h)\sqrt{d}\right]}\left( \phi(\cdot,\cdot)^\top\underline{\omega}^{\bs{\pi}}_{i,h}- \Gamma(\cdot,\cdot)\right)$.
        \STATE $\overline{Q}_{i,h}^{\dagger,\bs{\pi}_{-i}}(\cdot,\cdot) = \textnormal{Clip}_{\left[-(H-h)\sqrt{d},(H-h)\sqrt{d}\right]}\left(\phi(\cdot,\cdot)^\top\overline{\omega}^{\dagger,\bs{\pi}_{-i}}_{i,h}+ \Gamma(\cdot,\cdot)\right)$.
        \STATE $\underline{V}^{\bs{\pi}}_{i,h}(s) = \expctu{\bs{a}\sim\bs{\pi}_h}{\underline{Q}^{\bs{\pi}}_{i,h}(s,\bs{a})}$.
        \STATE $\overline{V}^{\dagger,\bs{\pi}_{-i}}_{i,h}(s) = \max_{a_i\in A_i}\expctu{\bs{a}_{-i}\sim\bs{\pi}_{-i}}{\overline{Q}^{\dagger,\bs{\pi}_{-i}}_{i,h}(s,\bs{a})}$.
        \ENDFOR
        \RETURN Value functions $\overline{V}^{\dagger,\bs{\pi}_{-i}}_{i,h}(\cdot)$ and $\underline{V}^{\bs{\pi}}_{i,h}(\cdot)$, for all $h\in[H-1]$.
    \end{algorithmic}
\end{algorithm}

\begin{algorithm}[!ht]
    \caption{Optimistic Hedge for $n$ $\min-\max$ Games (\texttt{OptimisticHedge})}\label{alg:optimistic_hedge}
    \begin{algorithmic}[1]
        \REQUIRE Loss functions $\mathcal{L}_1(\cdot),\ldots,\mathcal{L}_n(\cdot)$; step size $\nu$; steps $T$.
        \STATE Initialize policies $\pi^{(0)}_i,\pi^{\dagger,(0)}_i$ uniformly at random, for every $i\in \{1,\ldots,n\}$
        \FOR{$t=0,\ldots,T-1$}
        \FOR{$i\in  \{1,\ldots,n\}$}
        \STATE Let $u^{(t)}_i(a^\dagger) = \expctu{a'\sim\pi^{(t)}_i}{\mathcal{L}_i(a^\dagger,a')}$, for every individual action $a$ of player $i$.
        \STATE Let $\ell^{(t)}_i(a') = \expctu{a^\dagger\sim\pi^{\dagger,(t)}_i}{\mathcal{L}_i(a^\dagger,a')}$, for every individual action $a$ of player $i$.
        \FOR{$a_i\in A_i$}
        \STATE Update for the $\max$ player:
        \begin{align*}
            \pi^{\dagger,(t+1)}_i(a^\dagger) = \frac{\pi^{\dagger,(t)}_i(a^\dagger)\cdot \exp\left(\nu\cdot u^{(t)}_i(a^\dagger)\right)}{\sum_{a^\ddagger} \pi^{\dagger,(t)}_i(a^\ddagger)\cdot \exp\left(\nu\cdot u^{(t)}_i(a^\ddagger)\right)}~.
        \end{align*}
        \STATE Update for the $\min$ player:
        \begin{align*}
            \pi^{(t+1)}_i(a') = \frac{\pi^{(t)}_i(a')\cdot\exp\left(-\nu\cdot\ell^{(t)}_i(a')\right)}{\sum_{a''}\pi^{(t)}_i(a'')\cdot\exp\left(-\nu\cdot\ell^{(t)}_i(a'')\right)}~.
        \end{align*}
        \ENDFOR
        \ENDFOR
        \ENDFOR
        \RETURN Average policy profile over $T$ rounds $\frac{1}{T}\sum^T_{t=1}\bs{\pi}^{(t)}$.
    \end{algorithmic}
\end{algorithm}

\end{document}